\documentclass{article}

\usepackage[final,nonatbib]{neurips_2022}

\usepackage[utf8]{inputenc} % allow utf-8 input
\usepackage[T1]{fontenc}    % use 8-bit T1 fonts
\usepackage{hyperref}       % hyperlinks
\usepackage{url}            % simple URL typesetting
\usepackage{booktabs}       % professional-quality tables
\usepackage{amsfonts}       % blackboard math symbols
\usepackage{nicefrac}       % compact symbols for 1/2, etc.
\usepackage{microtype}      % microtypography
\usepackage{color}

\usepackage{tikz, pgfplots}
\usetikzlibrary{plotmarks,spy}
\pgfplotsset{compat=newest}
\usepackage{graphicx, graphics, epsfig, color}
\usepackage{amsmath,amssymb,amsthm}
\usepackage{enumitem}
\usepackage{algorithm}
\usepackage{algorithmic}

\usepgfplotslibrary{groupplots}
\pgfplotsset{compat=newest}

\DeclareMathOperator{\tr}{tr}

\newcommand{\T}{{\sf T}}

\DeclareMathOperator*{\erf}{erf}

\DeclareMathOperator{\CK}{CK}
\DeclareMathOperator{\NTK}{NTK}
\DeclareMathOperator{\ReLU}{ReLU}
\DeclareMathOperator{\poly}{poly}

\newcommand{\RR}{{\mathbb{R}}}

\newcommand{\EE}{{\mathbb{E}}}
\newcommand{\NN}{{\mathcal{N}}}

\newcommand{\A}{\mathbf{A}}
\newcommand{\B}{\mathbf{B}}
\newcommand{\C}{\mathbf{C}}

\newcommand{\K}{\mathbf{K}}
\newcommand{\J}{\mathbf{J}}
\newcommand{\M}{\mathbf{M}}
\newcommand{\V}{\mathbf{V}}

\newcommand{\W}{\mathbf{W}}
\newcommand{\X}{\mathbf{X}}
\newcommand{\Z}{\mathbf{Z}}

\renewcommand{\P}{\mathbf{P}}

\newcommand{\x}{\mathbf{x}}

\newcommand{\z}{\mathbf{z}}

\newcommand{\s}{\mathbf{s}}
\newcommand{\w}{\mathbf{w}}

\newcommand{\zo}{\mathbf{0}}
\newcommand{\one}{\mathbf{1}}
\newcommand{\I}{\mathbf{I}}

\newcommand{\bt}{\mathbf{t}}
\newcommand{\bT}{\mathbf{T}}

\newcommand{\bmu}{\boldsymbol{\mu}}

\newcommand{\bSigma}{\boldsymbol{\Sigma}}

\newcommand{\bpsi}{\boldsymbol{\psi}}

\definecolor{RED}{rgb}{0.7,0,0}
\definecolor{BLUE}{rgb}{0,0,0.69}
\definecolor{GREEN}{rgb}{0,0.6,0}
\definecolor{PURPLE}{rgb}{0.69,0,0.8}
\definecolor{ORANGE}{RGB}{255,103,0}
\definecolor{BROWN}{RGB}{100,20,45}

\newcommand{\RED}{\color[rgb]{0.70,0,0}}
\newcommand{\BLUE}{\color[rgb]{0,0,0.69}}
\newcommand{\GREEN}{\color[rgb]{0,0.6,0}}
\newcommand{\PURPLE}{\color[rgb]{0.69,0,0.8}}
\newcommand{\ORANGE}{\color[RGB]{255,103,0}}
\newcommand{\BROWN}{\color[RGB]{100,20,45}}
\newtheorem{Assumption}{Assumption}
\newtheorem{Theorem}{Theorem}
\newtheorem{Corollary}{Corollary}

\newtheorem{Lemma}{Lemma}
\newtheorem{Remark}{Remark}

\title{``Lossless'' Compression of Deep Neural Networks: \\A High-dimensional Neural Tangent Kernel Approach}

\author{
  Lingyu Gu$^{*1}$
  \And
  Yongqi Du$^{*1}$
  \And
  Yuan Zhang$^2$
  \And 
  Di Xie$^2$
  \And
  Shiliang Pu$^2$
  \And 
  Robert C.~Qiu$^1$
  \And
  Zhenyu Liao$^{\dagger1}$ \\
  $^1$EIC, Huazhong University of Science and Technology, Wuhan, China\\
  $^2$Hikvision Research Institute, Hangzhou, China
}

\begin{document}

\maketitle

\begin{abstract}
Modern deep neural networks (DNNs) are extremely powerful; however, this comes at the price of increased depth and having more parameters per layer, making their training and inference more computationally challenging. 
In an attempt to address this key limitation, efforts have been devoted to the compression (e.g., sparsification and/or quantization) of these large-scale machine learning models, so that they can be deployed on low-power IoT devices.
In this paper, building upon recent advances in neural tangent kernel (NTK) and random matrix theory (RMT), we provide a novel compression approach to wide and fully-connected \emph{deep} neural nets. 
Specifically, we demonstrate that in the high-dimensional regime where the number of data points $n$ and their dimension $p$ are both large, and under a Gaussian mixture model for the data, there exists \emph{asymptotic spectral equivalence} between the NTK matrices for a large family of DNN models.
This theoretical result enables ``lossless'' compression of a given DNN to be performed, in the sense that the compressed network yields asymptotically the same NTK as the original (dense and unquantized) network, with its weights and activations taking values \emph{only} in $\{ 0, \pm 1 \}$ up to a scaling. 
Experiments on both synthetic and real-world data are conducted to support the advantages of the proposed compression scheme, with code available at \url{https://github.com/Model-Compression/Lossless_Compression}.
\end{abstract}

\section{Introduction}
\label{sec:intro}

\def\thefootnote{$*$}\footnotetext{Equal contribution, listed in a random order by rolling a dice on WeChat.}
\def\thefootnote{$\dagger$}\footnotetext{Author to whom any correspondence should be addressed. Email: \href{mailto:zhenyu_liao@hust.edu.cn}{zhenyu\_liao@hust.edu.cn}}
\def\thefootnote{\arabic{footnote}}

Modern deep neural networks (DNNs) are becoming increasingly over-parameterized, having more parameters than required to fit the also increasingly large, complex, and high-dimensional data.
While the list of successful applications of these large-scale machine learning (ML) models is rapidly growing, the energy consumption of these models is also increasing, making them more challenging to deploy on close-to-user and low-power devices.
To address this issue, compression techniques have been proposed that prune, sparsify, and/or quantize DNN models \cite{deng2020model,hoefler2021sparsity}, thereby yielding DNNs of a much smaller size that can still achieve satisfactory performance on a given ML task.
As an illustrative example, it has been recently shown that at least $90\%$ of the weights in popular DNN models such as VGG19 and ResNet32 can be removed with virtually no performance loss \cite{su2020sanity}.
 %VGG19\cite{simonyan2014very} and ResNet32\cite{he2016deep}

Despite the remarkable progress achieved by various DNN model compression techniques, due to the nonlinear and highly non-convex nature of DNNs, our theoretical understanding of these large-scale ML models, as well as of their compression schemes, is progressing at a more modest pace. 
For example, it remains unclear how much a given DNN model can be compressed \emph{without} severe performance degradation; perhaps more importantly, on the degree to which such a \emph{performance and complexity trade-off} depends on the ML task and the data also remains unknown.

In this respect, neural tangent kernels (NTKs) \cite{jacot2018neural}, provide a powerful tool for use in assessing the convergence and generalization properties of very wide (sometimes unrealistically so) DNNs by studying their corresponding NTK eigenspectra, which are \emph{solely} dependent on the input data, the network architecture and activation function, as well as the random weights distribution.\footnote{In the remainder of this article, what we refer to as the ``NTK matrix'' is essentially the limiting (and nonrandom) NTK matrix to which the random NTK converges under the infinite-wide limit. }% (in each layer).

In this paper, building upon recent advances in random matrix theory (RMT) and high-dimensional statistics, we demonstrate that for data $\x_1, \ldots, \x_n \in \RR^p$ drawn from a $K$-class Gaussian mixture model (GMM), in a high-dimensional and non-trivial classification regime where the input data dimension $p$ and their size $n$ are  both large and comparable, the \emph{eigenspectra} of both the NTK and the closely related conjugate kernel (CK) matrices at \emph{any} layer $\ell \in \{1, \ldots, L \}$ are \emph{independent} of the distribution of the i.i.d.\@ entries of the (random) weight matrix $\W_\ell$, provided that they are ``normalized'' to have zero mean and unit variance, and \emph{only} depend on the activation function $\sigma_\ell(\cdot)$ via \emph{four} scalar parameters.
In a sense, we establish, at least for GMM data, the asymptotic \emph{spectral equivalence} between the CK and NTK matrices of the corresponding network layers, and consequently of the whole network, for a large family of DNN models with possibly very different weights and activations, given that they have normalized weights and share the same few activation-related parameters.
%We further exploit this theoretical result to propose a novel NTK-based DNN compression scheme that allows for ``lossless'' compression (in the sense that the compressed net have the same NTK spectrum as the original one) of a given \emph{fully-connected} DNN model of depth $L$ to obtain a sparse and quantized DNN with both weights and activations taking values in the set $\{-1, 0, +1\}$ up to a scaling, and can thus be stored and computed much more efficiently.
%using solely $1$ bit per non-zero entry. 

Since the convergence and generalization properties of wide DNNs depend only on the eigenspectra (i.e., eigenvalue-eigenvector pairs; see also Remark~\ref{rem:spectral} below) of the corresponding NTK matrices \cite{jacot2018neural,fan2020spectra}, in the sense that, e.g., the time evolutions (when trained with gradient descent using a sufficiently small step size) of both the residual error and the in-sample prediction of the network can be expressed as \emph{explicit} functions of NTK eigenvalues and eigenvectors \cite{fan2020spectra,adlam2020neural,huang2020dynamics}, we further exploit the above theoretical results to propose a novel ``lossless'' compression approach for \emph{fully-connected} DNN models, by designing a sparse and quantized DNN that (i) has asymptotically the \emph{same} NTK eigenspectra as the original ``dense and full precision'' network, and (ii) has both weights and activations taking values in the set $\{-1, 0, +1\}$ before scaling, and can thus be stored and computed much more efficiently.
%, with weights and activation being sparse and binarized for non-zero entries.

Despite being derived here for Gaussian mixture data, an unexpected close match is observed between our theory and the empirical results on real-world datasets, suggesting possibly wider applicability for the proposed ``lossless'' compression approach.
Looking forward, we expect that our analysis will open the door to improved analysis of involved ML methods based on RMT and high-dimensional statistics, which will demystify the seemingly striking empirical observations in, say, modern DNNs.

\subsection{Our contributions}

Our main results can be summarized as follows: 
\begin{enumerate}%[leftmargin=\parindent,align=left,labelwidth=\parindent,labelsep=2pt]
  \item We provide, in Theorems~\ref{theo:CK}~and~\ref{theo:NTK} respectively, for GMM data and in the high-dimensional regime 
  of Assumption~\ref{ass:high-dimen}, \emph{precise} eigenspectral characterizations of CK and NTK matrices of fully-connected DNNs; we particularly show that the CK and NTK eigenspectra do \emph{not} depend on the distribution of i.i.d.\@ network weights, and \emph{depend} solely on the activation function of each layer via a few scalar parameters.
  \item In Corollary~\ref{coro:sparse_quantized} and Algorithm~\ref{alg:sparse_quantized}, we exploit these results to propose a novel DNN compression scheme, with \emph{sparsified and ternarized} weights and activations, without affecting the NTK spectral behavior, and thus the convergence and generalization properties of the network.
  \item %In Section~\ref{sec:experiments}, 
  We provide empirical evidence on (not so) wide DNNs trained on both synthetic Gaussian and real-world datasets such as MNIST~\cite{lecun1998gradient} and CIFAR10~\cite{krizhevsky2009learning}, and show a factor of $10^3$ less memory is needed with the proposed DNN compression approach, with virtually no performance loss.
\end{enumerate}

\subsection{Related work}

\textbf{Neural network model compression.} The study of NN compression dates back to early 
1990 \cite{lecun1990optimal}, at which point, in the absence of the (possibly more than) sufficient computational power that we have today, compression techniques allowed neural networks to be empirically evaluated on computers with limited computational and/or storage resources \cite{schmidhuber2015deep}.
Alongside the rapid growth of increasingly powerful computing devices, the development of more efficient NN architectures and training/inference protocols, and the need to implement NNs on mobile and low-power devices, (D)NN model compression has become an active research topic and many elegant and efficient compression approaches have been proposed over the years \cite{han2015deep,howard2017mobilenets,hubara2016binarized,hoefler2021sparsity}.
However, due to the nonlinear and highly non-convex nature of DNNs, our theoretical understanding of these large-scale ML models, as well as of (e.g., the fundamental ``performance and complexity'' trade-off of) compressed DNNs, is somewhat limited \cite{hoefler2021sparsity}.

\textbf{Neural tangent kernel.}
Neural tangent kernel (NTK) theory proposed in \cite{jacot2018neural}, by considering the limit of infinitely wide DNNs, characterizes the convergence and generalization properties of very wide DNNs when trained using gradient descent with small steps. Initially proposed for fully-connected nets, the NTK framework has been subsequently extended to convolutional \cite{arora2019exact}, graph \cite{du2019graph},  and recurrent \cite{alemohammad2020recurrent} settings. 
The NTK theory, while having the advantage of being mathematically more tractable (via, e.g., the characterization of the associated reproducing kernel Hilbert space \cite{bietti2019inductive}), seems to diverge from the regime on which modern (and not so wide) DNNs operate, see~\cite{chizat2019lazy,liu2020linearity,fan2020spectra}.

\textbf{Random matrix theory and neural networks.}
Random matrix theory (RMT), a powerful and flexible tool for assessing the behavior of large-scale systems with a large ``degree of freedom'', is gaining popularity in the field of NN analysis \cite{martin2019traditional,martin2020heavy}, in both shallow \cite{pennington2017nonlinear,liao2018spectrum,liao2018dynamics} and deep \cite{benigni2019eigenvalue,fan2020spectra,pastur2020gauss} settings, 
and for both homogeneous (e.g., standard normal) \cite{pennington2017nonlinear,mei2019generalization} and mixture-type data \cite{liao2018spectrum,ali2022random}.
From a technical perspective, the most relevant paper is \cite{ali2022random} , in which the authors proposed a RMT-inspired NN compression scheme, albeit only in the single-hidden-layer setting. 
This paper extends the analysis in \cite{ali2022random} to multi-layer fully-connected DNNs, by focusing on both the CK and NTK matrices, and proposes a novel sparsification and quantization scheme for fully-connected DNNs (which is in spirit similar to, although formally different from, that in \cite{ali2022random}).

\subsection{Notations and organization of the paper}

We denote scalars by lowercase letters, vectors by bold lowercase, and matrices by bold uppercase. 
We denote the transpose operator by $(\cdot)^{\sf T}$, and use $\|\cdot\|$ to denote the Euclidean norm for vectors and the spectral/operator norm for matrices. 
For a random variable $z$, $\mathbb{E}[z]$ denotes the expectation of $z$. 
%We denote $\delta_{x \in A}$ the Kronecker delta that takes value $1$ when $x \in A$ and $0$ otherwise. 
We use $\one_p$ and $\I_{p}$ to represent an all-ones vector of dimension $p$ and the identity matrix of size $p\times p$.

The remainder of this article is structured as follows.  
In Section~\ref{sec:preliminaries}, we present the fully-connected DNN model under study, together with our working assumptions.
Section~\ref{sec:main_results} contains our main technical results on the eigenspectra of the conjugate kernel $\K_{\CK}$ and NTK matrix $\K_{\NTK}$, along with an account of how they apply to the compression of fully-connected DNNs with the proposed ``lossless'' compression scheme.
Empirical evidence is provided in Section~\ref{sec:experiments} to demonstrate the significant computation and storage savings, with minimal performance degradation, that can be obtained using the proposed approach. 
Conclusion and future perspectives are placed in Section~\ref{sec:conclusion}.

% Empirical evidence is provided in Section~\ref{sec:experiments} to show the computational and storage advantage compared to \zhenyu{TBD}.
\section{Preliminaries}
\label{sec:preliminaries}

Let $\x_1, \ldots, \x_n \in \RR^p$ be $n$ data vectors independently drawn from one of the $K$-class Gaussian mixtures $\mathcal C_1, \ldots, \mathcal C_K$, and let $\X = [\x_1, \ldots, \x_n] \in \RR^{p \times n}$, with class $\mathcal C_a$ having cardinality $n_a$; that is,
\begin{equation}\label{eq:def_GMM}
  \x_i \in \mathcal C_a \Leftrightarrow \x_i \sim \mathcal N(\bmu_a/\sqrt p, \C_a/p),
\end{equation}
for mean vector $\bmu_a \in \RR^p$ and covariance matrix $\C_a \in \RR^{p \times p}$ associated with class $\mathcal C_a$.

In the high-dimensional scenario where $n,p$ are both large and comparable, we position ourselves in the following non-trivial classification setting, so that the classification of the $K$-class mixture is neither trivially easy nor impossible; see also \cite{couillet2018classif} and \cite[Section~2]{blum2020foundations}.

\begin{Assumption}[High-dimensional asymptotics]\label{ass:high-dimen}
As $n \to \infty$, we have, for $a \in \{1,\ldots,K\}$ that 
(i) $p/n \to c \in (0,\infty)$ and $n_a/n \to c_a \in [0,1)$; 
(ii) $\| \bmu_a \| = O(1)$; 
(iii) for $\C^\circ \equiv \sum_{a=1}^K \frac{n_a}n \C_a$ 
and $\C_a^\circ \equiv \C_a - \C^\circ$, 
we have $\| \C_a\| = O(1)$, $\tr \C_a^\circ = O(\sqrt p)$ 
and $\tr (\C_a \C_b) = O(p)$ for $a,b \in \{1,\ldots,K\}$; 
and (iv) $\tau_0 \equiv \sqrt{\tr \C^\circ/p}$ converges in $(0,\infty)$.

\end{Assumption}

We consider using a fully-connected DNN model of depth $L$ for the classification of the above $K$-class Gaussian mixture. 
Such a network can be parameterized by a sequence of weight matrices $\W_1 \in \RR^{d_1 \times d_0}, \ldots, \W_L \in \RR^{d_L \times d_{L-1}}$ (with the convention $d_0 = p$), and nonlinear activation functions $\sigma_1, \ldots, \sigma_L$ that apply entry-wise, so that the network output $f(\x) \in \RR$ is given by:
\begin{equation}\label{eq:NN_model}
  f(\x) = \frac1{\sqrt {d_L}} \w^\T  \sigma_L \left( \frac1{\sqrt {d_{L-1}} } \W_L  \sigma_{L-1} \left( \ldots \frac1{ \sqrt{d_2}} \sigma_2 \left( \frac1{\sqrt{d_1} } \W_2 \sigma_1 (\W_1 \x ) \right) \right) \right),
\end{equation}
for an input data vector $\x \in \RR^p$ and output vector $\w \in \RR^{d_L}$. 
We denote $\bSigma_\ell \in \RR^{d_{\ell} \times n}$ the representations of the data matrix $\X \in \RR^{p \times n}$ at layer $\ell \in \{1, \ldots, L \}$ defined as
\begin{equation}\label{eq:Sigma_ell}
  \bSigma_\ell = \frac1{\sqrt {d_{\ell}}} \sigma_{\ell} \left( \frac1{\sqrt {d_{\ell-1}} } \W_{\ell}  \sigma_{\ell-1} \left( \ldots \frac1{ \sqrt{d_2}} \sigma_2 \left( \frac1{\sqrt{d_1} } \W_2 \sigma_1 (\W_1 \X ) \right) \right) \right).
\end{equation}
The normalization by $1/\sqrt{d_\ell}$ follows from the NTK literature and ensures the consistent asymptotic behavior of the network in the high-dimensional setting in Assumption~\ref{ass:high-dimen}~and~\ref{ass:W}; see also \cite{jacot2018neural,bietti2019inductive,fan2020spectra}.

The training and generalization performance of the NN model defined in \eqref{eq:NN_model} are closely related to two types of kernel matrices: the Conjugate Kernel (CK) matrix and Neural Tangent Kernel (NTK) matrix, 
defined respectively for $\ell \in \{1, \ldots,L\}$ as follows:
\begin{equation}\label{eq:def_K_CK_K_NTK}
  \K_{\CK,\ell} = \EE[ \bSigma_\ell^\T \bSigma_\ell] \in \RR^{n \times n}, %\quad \K_{\NTK,\ell} = \EE[\J^\T \J]
\end{equation}
with expectation taken with respect to the random weights $\W_1, \ldots, \W_\ell$ and $\bSigma_\ell \in \RR^{d_{\ell} \times n}$ the data representation at the output of layer $\ell$ defined in \eqref{eq:Sigma_ell}. In particular, CK matrices are known to satisfy the following recursive relation \cite{jacot2018neural,bietti2019inductive}
\begin{equation}\label{eq:K_CK_relation}
  %[\K_{\CK,\ell}]_{ij} = \EE_{u,v \sim \NN(\zo, \B_\ell)} [\sigma_\ell(u) \sigma_\ell(v)],~\text{with}~\B_\ell = \begin{bmatrix} [\K_{\CK,\ell-1}]_{ii} & [\K_{\CK,\ell-1}]_{ij} \\ [\K_{\CK,\ell-1}]_{ij} & [\K_{\CK,\ell-1}]_{jj} \end{bmatrix},
  [\K_{\CK,\ell}]_{ij} = \EE_{u,v} [\sigma_\ell(u) \sigma_\ell(v)],~\text{with}~u,v \sim \NN \left(\zo, \begin{bmatrix} [\K_{\CK,\ell-1}]_{ii} & [\K_{\CK,\ell-1}]_{ij} \\ [\K_{\CK,\ell-1}]_{ij} & [\K_{\CK,\ell-1}]_{jj} \end{bmatrix}\right),
\end{equation}
while for the NTK matrix $\K_{\NTK,\ell} \in \RR^{n \times n}$ of layer $\ell$, we have:
\begin{equation}\label{eq:K_NTK_relation}
  \K_{\NTK,\ell} = \K_{\CK,\ell} + \K_{\NTK,\ell-1} \circ \K'_{\CK,\ell}, \quad \K_{\NTK,0} = \K_{\CK,0} = \X^\T \X,
\end{equation}
where `$\A \circ \B$' denotes the Hadamard product between two matrices $\A,\B$ of the same size, and $\K'_{\CK,\ell}$ denotes the CK matrix with nonlinear function $\sigma_\ell'$ instead of $\sigma_\ell$ as for $\K_{\CK,\ell}$ defined in \eqref{eq:K_CK_relation}; that is, $[\K'_{\CK,\ell}]_{ij} = \EE_{u,v} [\sigma'_\ell(u) \sigma'_\ell(v)]$. %for $\B_\ell$ defined in \eqref{eq:K_CK_relation}. 
Note in particular that for a given DNN model, the corresponding CK and NTK matrices depend  \emph{only} on the network structure (i.e., the number of layers and the activation function in each layer), the \emph{distribution} of the random (initializations of the) weights to be integrated over (e.g., in the expectation in equation \eqref{eq:def_K_CK_K_NTK}), and the input data.

It has been shown in a series of previous efforts \cite{jacot2018neural,fan2020spectra,huang2020dynamics} that for very (and sometimes unrealistically) wide DNNs trained using gradient descent with a small step size, the time evolution of the residual errors and in-sample predictions of a given DNN are \emph{explicit} functionals of the corresponding $\K_{\NTK}$ involving its eigenvalues and eigenvectors. 
In this respect, the NTK theory provides, via the eigenspectral behavior of $\K_{\NTK}$, precise characterizations of the convergence and generalization properties of DNNs \cite{jacot2018neural,bietti2019inductive}, by focusing on the impact of the network structure, the input data, and the weight initialization schemes.

In this paper, we focus on fully-connected nets under the following assumption regarding the weights.

\begin{Assumption}[On weight initializations]\label{ass:W}
The random weights $\W_1 \in \RR^{d_1 \times p}, \ldots, \W_L \in \RR^{d_L \times d_{L-1}}$ are independent and have i.i.d.\@ entries of zero mean, unit variance, and finite fourth-order moment.
%, that is, $\EE[\W_\ell]_{ij} = 0$, $\EE[\W_\ell]_{ij} = 0$, and $\EE[\W_\ell]_{ij} =$
\end{Assumption}
Assumption~\ref{ass:W}, together with the $1/\sqrt{d_\ell}$ normalization, is compatible with fully-connected DNNs in \eqref{eq:NN_model}, which are admittedly less interesting, from a practical perspective, compared to their convolutional counterparts. 
The proposed framework is envisioned to be extendable to a convolutional \cite{arora2019exact,bietti2019inductive} and more involved setting (e.g., graph NNs \cite{du2019graph}) by considering (e.g., Toeplitz-type) structures on $\W$s.

Unlike most existing NTK literature \cite{jacot2018neural,bietti2019inductive,fan2019spectral}, we do not assume the Gaussianity of the entries of $\W_\ell$s, but only that they are i.i.d.\@ and ``normalized'' to have zero mean and unit variance. As it turns out, this assumption together with a (Lyapunov-type) central limit theorem argument, is sufficient to establish most existing results on the convergence and generalization of DNNs; see for example \cite{lee2018deep}.

We also need the following assumption on the activation functions in each layer.

\begin{Assumption}[On activation functions]\label{ass:activation}
The activations $\sigma_1, \ldots, \sigma_L$ are at least four-times differentiable with respect to standard normal measure, in the sense that $\max_{k \in \{0,1,2,3,4\}} \{ |\EE[\sigma_\ell^{(k)}(\xi)]| \} < C$ for some universal constant $C > 0$, $\xi \sim \NN(0,1)$, and $\ell \in \{1,\ldots, L\}$.
\end{Assumption}
Using the Gaussian integration by parts formula, one has $\EE[\sigma'(\xi)] = \EE[\xi \sigma(\xi)]$ for $\xi \sim \NN(0,1)$, as long as the right-hand side expectation exists.
As a result, for non-differentiable functions, it suffices to have $|\sigma_\ell|$ upper-bounded by some (high-degree) polynomial function for Assumption~\ref{ass:activation} to hold.

With these preliminaries, we are now ready to present our main technical results on the eigenspectral behavior of the CK and NTK matrices for a large family of fully-connected DNN models.

\section{Main results}
\label{sec:main_results}

For a fully-connected DNN defined in \eqref{eq:NN_model}, our first result is on the eigenspectral behavior of the corresponding CK matrices $\K_{\CK}$ defined in \eqref{eq:def_K_CK_K_NTK}. More specifically, we show for Gaussian mixture data in \eqref{eq:def_GMM} and in the high-dimensional setting of Assumption~\ref{ass:high-dimen}, that the $\K_{\CK,\ell}$ of layer $\ell \in \{1, \ldots, L \}$ is asymptotically \emph{spectrally equivalent} to another random matrix $\tilde \K_{\CK,\ell}$, in the sense that their spectral norm difference $\| \K_{\CK,\ell} - \tilde \K_{\CK,\ell} \|$ vanishes as $n,p \to \infty$. 
This result is stated as follows, the proof of which is based on an induction on $\ell$ and is given in Section~\ref{subsec:proof_theo_CK} of the appendix.

\begin{Theorem}[Asymptotic spectral equivalents for CK matrices]\label{theo:CK}
Let Assumptions~\ref{ass:high-dimen}--\ref{ass:activation} hold, and let $\tau_0, \tau_1, \ldots, \tau_L \geq 0$ be a sequence of non-negative numbers satisfying the following recursion:
\begin{equation}\label{eq:def_tau}
  \tau_\ell = \sqrt{  \EE[ \sigma_\ell^2 (\tau_{\ell-1} \xi) ] }, \quad \xi \sim \NN(0,1), \quad \ell \in \{ 1, \ldots, L \}.
\end{equation}
Further assume that the activation functions $\sigma_\ell(\cdot)$s are ``centered,'' such that $\EE[\sigma_\ell(\tau_{\ell-1} \xi)] = 0$. Then, for the CK matrix $\K_{\CK,\ell}$ of layer $\ell \in \{0, 1, \ldots, L\}$ defined in \eqref{eq:def_K_CK_K_NTK}, as $n,p \to \infty$, one has that
\begin{equation}\label{eq:def_tilde_K_CK}
\| \K_{\CK,\ell} (\X) - \tilde \K_{\CK,\ell} (\X) \| \to 0, \quad \tilde \K_{\CK,\ell} (\X) \equiv \alpha_{\ell,1} \X^\T \X + \V \A_\ell \V^\T + \alpha_{\ell,0} \I_n,
\end{equation}
almost surely, with $ \alpha_{\ell,0} = \tau_\ell^2 - \tau_0^2 \alpha_{\ell,1} \geq 0$ and
\begin{equation}\label{eq:def_V_A}
\V = [\J/\sqrt p,~\bpsi] \in \RR^{n \times (K+1)}, \quad \A_\ell = \begin{bmatrix} \alpha_{\ell,2} \bt \bt^\T + \alpha_{\ell,3} \bT & \alpha_{\ell,2}\bt \\ \alpha_{\ell,2}\bt^\T & \alpha_{\ell,2} \end{bmatrix} \in \RR^{(K+1)\times(K+1)},
\end{equation}

for class label vectors $\J = [\mathbf{j}_1, \ldots, \mathbf{j}_K] \in \RR^{n \times K}$ with $[\mathbf{j}_a] = \delta_{\x_i \in \mathcal C_a}$, second-order data fluctuation vector $\bpsi = \{ \| \x_i - \EE[\x_i] \|^2 -  \EE[ \|\x_i - \EE[\x_i] \|^2 ] \}_{i=1}^n \in \RR^n$, second-order discriminative statistics $\bt = \{ \tr \C_a^\circ/\sqrt p \}_{a=1}^K \in \RR^K$ and $\bT = \{ \tr \C_a \C_b/p \}_{a,b=1}^K \in \RR^{K \times K}$ of the Gaussian mixture in \eqref{eq:def_GMM}, as well as non-negative scalars $\alpha_{\ell,1}, \alpha_{\ell,2}, \alpha_{\ell,3} \geq 0$ satisfying the following recursions:
\begin{align}%\label{eq:def_ds}
  \alpha_{\ell,1} &= \EE [ \sigma_\ell'( \tau_{\ell-1} \xi)]^2 \alpha_{\ell-1,1}, \quad \alpha_{\ell,2} = \EE [ \sigma_\ell'( \tau_{\ell-1} \xi)]^2 \alpha_{\ell-1,2} + \frac14 \EE [ \sigma_\ell''( \tau_{\ell-1} \xi)]^2 \alpha_{\ell-1,4}^2,\label{eq:def_ds1}
  \\ 
  \alpha_{\ell,3} &= \EE [ \sigma_\ell'( \tau_{\ell-1} \xi)]^2 \alpha_{\ell-1,3} + \frac12 \EE [ \sigma_\ell''( \tau_{\ell-1} \xi)]^2 \alpha_{\ell-1,1}^2,\label{eq:def_ds2}
\end{align}
with $\alpha_{\ell,4} = \alpha_{\ell-1,4} \EE \left[ (\sigma_\ell'(\tau_{\ell - 1} \xi) )^2 + \sigma_\ell(\tau_{\ell - 1} \xi) \sigma_\ell''(\tau_{\ell - 1} \xi) \right]$ for $\xi \sim \NN(0,1)$.
\end{Theorem}

A few remarks on Theorem~\ref{theo:CK} are in order. The first remark is on the assumption $\EE[\sigma_\ell(\tau_{\ell-1} \xi)] = 0$.

\begin{Remark}[On activation centering]\label{rem:ass_alpha_0}\normalfont
The condition {$\EE[\sigma_\ell(\tau_{\ell-1} \xi)] = 0$}, seemingly restrictive at first sight, in fact only subtracts an \emph{identical constant} from all entries of the data representation $\bSigma_\ell$ at layer $\ell$, and should therefore \emph{not} restrict the expressive power of the network, nor its performance on downstream ML tasks.
For a given DNN model of interest, it suffices to ``center'' the output of each layer by subtracting a constant to satisfy Assumption~\ref{ass:activation}, and to further apply our Theorem~\ref{theo:CK}.
\end{Remark}

Theorem~\ref{theo:CK} unveils the (possibly surprising) fact that, for the high-dimensional and non-trivial Gaussian mixture classification in \eqref{eq:def_GMM}, the spectral behavior of $\tilde \K_{\CK,\ell}$, and thus that of the CK matrix $\K_{\CK,\ell}$, is (i) \emph{independent} of the distribution of the (entries of the) weights $\W_\ell$ when they are ``normalized'' to have zero mean and unit variance, as demanded in Assumption~\ref{ass:W}, and (ii) depends on the activation function $\sigma_\ell$ \emph{only} via four\footnote{It is worth noting that the parameter $\tau_\ell$ appears in the CK eigenspectrum \emph{only} by shifting all its eigenvalues  (by $\tau_\ell^2$), thereby acting as an (implicit) ridge-type regularization in DNN models, see also \cite{jacot2020implicit,derezinski2020exact,liu2021kernel}.\label{foot:regularization}} scalar parameters $\alpha_{\ell,1}, \alpha_{\ell,2}, \alpha_{\ell,3}$ and $\tau_\ell$: such universal results have been previously observed in random matrix theory and high-dimensional statistics literature (see for example \cite{couillet_liao_2022,bai2010spectral,vershynin2018high}) and indicate some kind of universality of DNN models.
%the wide applicability of our theoretical result.

On closer inspection of Theorem~\ref{theo:CK}, we further observe that:
\begin{itemize}%[leftmargin=\parindent,align=left,labelwidth=\parindent,labelsep=2pt]
  \item[(i)] for a given DNN, Theorem~\ref{theo:CK} characterizes, via the form of $\tilde \K_{\CK,\ell}$ in \eqref{eq:def_tilde_K_CK} and the recursions in \eqref{eq:def_ds1}~and~\eqref{eq:def_ds2}, how the linear (via $\alpha_{\ell,1}$, which is multiplied by $\X^\T \X$) and nonlinear (via $\alpha_{\ell,2}$ and $\alpha_{\ell,3}$ in $\A_\ell$, which respectively weight the second-order data statistics $\bt$ and $\bT$) data features ``propagate'' in a DNN, in a layer-by-layer fashion, as $\ell$ increases, as \emph{quantitatively} measured by the corresponding $\alpha_\ell$s; and
  \item[(ii)] for two DNNs with the same number of layers, but possibly different weights and activations, given the same input data $\X$ (so that the two nets have the same $\K_{\CK,0}$), if they have asymptotically equivalent CK matrices $\K_{\CK,\ell-1}$ at layer $\ell-1$ with the \emph{same} $\alpha_{\ell-1,1}$, $\alpha_{\ell-1,2}$ and $\alpha_{\ell-1,3}$, then its follows from Equation~\eqref{eq:def_ds1}~and~\eqref{eq:def_ds2} that having the \emph{same} $\EE[\sigma_\ell'(\tau_{\ell-1}\xi)]^2$, $\EE[\sigma_\ell''(\tau_{\ell-1}\xi)]^2$, and $\EE[(\sigma_\ell^2(\tau_{\ell-1}\xi))'']$ (which \emph{only} depends on the activation $\sigma_\ell$ of layer $\ell$ and $\tau_{\ell-1}$) suffices for the two nets to have asymptotically equivalent $\K_{\CK,\ell}$ at layer $\ell$.
\end{itemize}

It follows from the above item~(ii) that for a given DNN of depth $L$, it is possible to design a novel DNN model that ``matches'' the original one -- in the sense that both models will have asymptotically equivalent CK matrices \emph{at each layer}, by using the following layer-by-layer matching strategy: Starting from the same $\K_{\CK,0} = \X^\T \X$, one chooses the first-layer weights $\W_1$ of the novel DNN according to Assumption~\ref{ass:W}, and then select the first-layer activation $\sigma_1$ in such a way that the novel net has the same parameters $\alpha_{1,1}$, $\alpha_{1,2}$ and $\alpha_{1,3}$ as the original one, so that the first-layer CK matrices $\K_{\CK,1}$ of the two nets are \emph{spectrally} matched as per Theorem~\ref{theo:CK}; one then proceeds similarly to match the second, the third, etc., and eventually the $L$th layer of the two nets.
As we shall see below, this layer-by-layer matching strategy facilitates the ``lossless'' compression of a given DNN.

Using the relation in \eqref{eq:K_NTK_relation}, a similar result (as in Theorem~\ref{theo:CK} for CKs) can be established for NTK matrices, as shown in the following theorem. The proof is also based on an induction on $\ell$ and is provided in Section~\ref{subsec:proof_theo_NTK} of the appendix.

\begin{Theorem}[Asymptotic spectral equivalent for NTK matrices]\label{theo:NTK}
Let Assumptions~\ref{ass:high-dimen}--\ref{ass:activation} hold, under the same settings and notations of Theorem~\ref{theo:CK}, and let $\dot{\tau}_0 = 0$, $\dot{\tau}_1, \ldots, \dot{\tau}_L \geq 0$ be a sequence of non-negative numbers such that
\begin{equation}\label{eq:def_tau'}
  \dot{\tau}_\ell = \sqrt{  \EE \left[ \left(\sigma_\ell' ( \tau_{\ell-1} \xi) \right)^2 \right] }, \quad \xi \sim \NN(0,1), \quad \ell \in \{ 1, \ldots, L \},
\end{equation}
which is similar to the $\tau_\ell$s defined in \eqref{eq:def_tau}, but on the derivative $\sigma'_{\ell}$ instead of $\sigma_{\ell}$,
and let $\kappa_\ell^2 = \tau^2_\ell + \dot{\tau}^2_\ell$, with $\kappa_0 = \tau_0$. 
Then, for the NTK matrix $\K_{\NTK,\ell}$ of layer $\ell$ defined in \eqref{eq:K_NTK_relation}, as $n,p \to \infty$, one has:
\begin{equation}\label{eq:def_tilde_K_NTK}
  \| \K_{\NTK,\ell}(\X) - \tilde \K_{\NTK,\ell}(\X) \| \to 0, \quad \tilde \K_{\NTK,\ell} (\X) \equiv \beta_{\ell,1} \X^\T \X + \V \B_\ell \V^\T + \beta_{\ell,0} \I_n,
\end{equation}
almost surely, with $\beta_{\ell,0} \equiv \kappa_\ell^2 - \tau_0^2 \beta_{\ell,1} \geq 0$, $\V \in \RR^{n \times (K+1)}, \bt \in \RR^K, \bT \in \RR^{K \times K}$ as defined in Theorem~\ref{theo:CK}, and
\begin{equation}
  %\V = [\J/\sqrt p,~\bpsi] \in \RR^{n \times (K+1)}, \quad 
  \B_\ell \equiv \begin{bmatrix} \beta_{\ell,2} \bt \bt^\T + \beta_{\ell,3} \bT & \beta_{\ell,2}\bt \\ \beta_{\ell,2}\bt^\T & \beta_{\ell,2} \end{bmatrix} \in \RR^{(K+1) \times (K+1)},
\end{equation}
and non-negative scalars $\beta_{\ell,1}, \beta_{\ell,2}, \beta_{\ell,3} \geq 0$ satisfying\footnote{\label{foot:errata2} Note that the system of equations here is different from that of the \href{https://papers.nips.cc/paper_files/paper/2022/hash/185087ea328b4f03ea8fd0c8aa96f747-Abstract-Conference.html}{online version} of this paper in the NeurIPS 2022 proceeding. This is due to an important (although very basic) algebraic mistake in the proof of Theorem~\ref{theo:NTK} in the definition of $\K'_{\CK,\ell}(\X)$. }% as $[\K'_{\CK,\ell}(\X)]_{ij} = \EE_{u,v} [\sigma'_\ell(u) \sigma'_\ell(v)], \quad u,v \sim \NN \left(\zo, \begin{bmatrix} [\K_{\CK,\ell-1}]_{ii} & [\K_{\CK,\ell-1}]_{ij} \\ [\K_{\CK,\ell-1}]_{ij} & [\K_{\CK,\ell-1}]_{jj} \end{bmatrix}\right)$
    \begin{align}
        \beta_{\ell,1} & =\alpha_{\ell,1} + \EE\left[ \sigma_\ell'( \tau_{\ell-1} \xi)\right]^2 \beta_{\ell - 1,1}, \quad \beta_{\ell,2} = \alpha_{\ell,2} + \EE\left[ \sigma_\ell'( \tau_{\ell-1} \xi)\right]^2 \beta_{\ell - 1,2}, \\
        \beta_{\ell,3} & = \alpha_{\ell,3}  + \EE\left[ \sigma_\ell'( \tau_{\ell-1} \xi)\right]^2 \beta_{\ell - 1,3} + \EE\left[\sigma_\ell''( \tau_{\ell-1} \xi)\right]^2 \alpha_{\ell-1, 1} \beta_{\ell - 1,1}.
    \end{align}
with $\alpha_{\ell-1,1 }$ as defined in Theorem~\ref{theo:CK} and
\begin{equation}\label{eq:beta_0s}
    \beta_{0,1} = 1, \quad \beta_{0,2} = \beta_{0,3} = 0.
  \end{equation}
\end{Theorem}

Roughly speaking, Theorem~\ref{theo:NTK} shows that the eigenspectral behavior established in Theorem~\ref{theo:CK} for CK matrices also holds for NTK matrices, up to a change of the associated coefficients $\alpha_\ell$s to $\beta_{\ell}$s.
The remarks after Theorem~\ref{theo:CK} thus remain valid, at least in spirit, for NTK matrices. 

\begin{Remark}[On spectral norm characterization]\label{rem:spectral}\normalfont
Note that the characterizations in Theorem~\ref{theo:CK}~and~\ref{theo:NTK} for CK and NTK matrices are provided in a spectral norm sense. 
It then follows from Weyl's inequality \cite{horn2012matrix} and the Davis--Kahan theorem \cite{yu2015useful} that the difference between the eigenvalues (e.g., when listed in a decreasing order) and the associated eigenvectors (when the eigenvalues under study are ``isolated'') of $\K_{\NTK}$ and $\tilde \K_{\NTK}$ vanish asymptotically as $n,p \to \infty$.
As such, the spectral norm guarantees in Theorem~\ref{theo:CK}~and~\ref{theo:NTK} provide more tractable access to the convergence and generalization properties of wide DNNs, at least for GMM data, via the spectral study of $\tilde \K_{\NTK}$ \cite{jacot2018neural,fan2020spectra}.
\end{Remark}

Despite being derived here for the Gaussian mixture model in \eqref{eq:def_GMM}, we conjecture that the results in Theorem~\ref{theo:CK}~and~\ref{theo:NTK} hold beyond the Gaussian setting and extend, for example, to the family of concentrated random vectors \cite{seddik2020random,ledoux2005concentration}.
As discussed after Assumption~\ref{ass:W} for the distribution of $\W$, such universality commonly arises in random matrix theory and high-dimensional statistics \cite{couillet_liao_2022,bai2010spectral,vershynin2018high}; we refer interested readers to Remark~\ref{rem:beyond_gaussian_data} in Appendix~\ref{sec:proofs_and_auxiliary} for further discussions.

From a technical perspective, the results in Theorem~\ref{theo:CK}~and~\ref{theo:NTK} extend the single-hidden-layer CK analysis in \cite{ali2022random,liao2018spectrum} to both CK and NTK matrices of fully-connected DNNs with an arbitrary number of layers.
In particular, taking $\ell=1$ in Theorem~\ref{theo:CK}, one obtains \cite[Theorem~1]{ali2022random} as a special case.\footnote{Note that in \cite[Theorem~1]{ali2022random}, the authors do \emph{not} assume $\EE[\sigma(\tau_0 \xi)] = 0$ as in our Theorem~\ref{theo:CK}, but instead ``center'' the CK matrices by pre-~and~post-multiplying $\K_{\CK}$ with $\P = \I_n -  \one_n \one_n^\T/n$. This can be shown equivalent to taking $\EE[\sigma(\tau_0 \xi)] = 0$ in the single-hidden-layer setting; see Appendix~\ref{subsec:equivalent_centering} for more details. }

\begin{Remark}[On CK and NTK matrices]\label{rem:macth}\normalfont
    It follows from Theorem~\ref{theo:CK} that for a given DNN model, it suffices to match the coefficients $\alpha_{\ell,1}, \alpha_{\ell,2}$ and $\alpha_{\ell,3}$ in a layer-by-layer manner to propose a novel DNN with asymptotically equivalent CK matrices.
    Similarly, matching the key coefficients $\beta_{\ell,1},\beta_{\ell,2}$, and $\beta_{\ell,3}$ leads to a DNN model with asymptotically equivalent NTK matrices.
    Moreover, for two nets with the same $\alpha_{\ell,1}, \alpha_{\ell,2}, \alpha_{\ell,3}$, and therefore the same CKs, it suffices to match the additional $\dot{d}_{\ell}$, $\dot{\alpha}_{\ell}$ to render the two nets asymptotically equivalent in both CK and NTK senses.
\end{Remark}
% \GD{
% no mention we actually implement $CK$ matching rather than NTK
% } 

In the following corollary, we present a concrete example of how to apply the results in Theorem~\ref{theo:CK}~and~\ref{theo:NTK} in the design of a novel computationally and storage efficient DNN, that shares the same CK and NTK eigenspectra with any given fully-connected neural net having centered activation.

\begin{Corollary}[Sparse and quantized DNNs]\label{coro:sparse_quantized}
For a given fully-connected DNN (referred to as ${\sf DNN1}$) of depth $L$ with centered activation such that $\EE[\sigma_\ell(\tau_{\ell-1} \xi)] = 0$ for $\xi \sim \NN(0,1)$, one is able to construct, say in a layer-by-layer manner, a sparse and quantized ``equivalent'' DNN model, of depth $L$ and referred to as  ${\sf DNN2}$, such that the two nets have asymptotically the same eigenspectra for their CK (and NTK similarly, as per Remark~\ref{rem:macth}) matrices, by using the following ternary weights:
\begin{equation}\label{eq:def_W}
  [\W]_{ij} = 0~\text{with proba $\varepsilon \in [0,1)$}, \quad [\W]_{ij} = \pm (1-\varepsilon)^{-1/2}~\text{each with proba $1/2-\varepsilon/2$},
\end{equation}
as well as quantized activations (as visually displayed in Figure~\ref{fig:sigma}):
%the ternary and quaternary activations
\begin{equation}\label{eq:def_sigmal}
    \sigma_T(t)=a \cdot (1_{t < s_1}+1_{t > s_2}), \quad \sigma_Q(t)=b_1 \cdot (1_{t < r_1}+1_{t > r_4}) + b_2 \cdot 1_{r_2 \leq t \leq r_3}.
\end{equation}
%\label{weight_sparsity}
\end{Corollary}

We refer readers to Appendix~\ref{subsec:proof_coro_sparse_quantized} for the proof and discussions of Corollary~\ref{coro:sparse_quantized}, as well as the detailed expressions of $\EE[\sigma'(\tau \xi)]$, $\EE[\sigma''(\tau \xi)]$, and $\EE[(\sigma^2(\tau \xi))'']$, of direct algorithmic use for both $\sigma_T$ and $\sigma_Q$ as functions of the parameters $a, s_1, s_2$ and $b_1, b_2, r_1, r_2, r_3, r_4$.
Built upon Corollary~\ref{coro:sparse_quantized}, we propose a DNN ``lossless'' compression scheme with equivalent CKs, as summarized in Algorithm~\ref{alg:sparse_quantized} below.

\begin{figure}[thb]
  \centering
  \begin{equation*}
    \vcenter{\hbox{
        \begin{tikzpicture}[font=\footnotesize, inner sep=1]
          \renewcommand{\axisdefaulttryminticks}{4}
          \pgfplotsset{every major grid/.append style={densely dashed}}
          \tikzstyle{every axis y label}+=[yshift=-10pt]
          \tikzstyle{every axis x label}+=[yshift=5pt]
          \pgfplotsset{every axis legend/.style={cells={anchor=east},fill=none,at={(.95,.95)}, anchor=north east, font=\footnotesize}}
          \begin{axis}[
              %ybar,
              height=.32\linewidth,
              width=.45\linewidth,
              xmin=-2,xmax=2,
              ymin=-2.5,ymax=4,
              xtick={-1.5, -1, -0.5, 0, 0.7, 1, 1.5},
              xticklabels = { $r_1$, $r_2$, $s_1$,$0$, $s_2$, $r_3$, $r_4$ },
              ytick={-1.5,0,1,1.5},
              yticklabels = { $a$, $0$, $b_2$, $b_1$ },
              grid=major,
              %ymajorgrids=false,
              scaled ticks=true,
            ]
            \def\s{1}
            \def\BB{3}
            \addplot[samples=3000,domain=-2.5:5,RED,line width=1.5pt] { (abs(x)<=1*\s)*(abs(sign(x)))+(abs(x)>=1.5*\s)*(1.5*abs(sign(x)))
            };
            \addlegendentry{ $\sigma_Q$ };
            \addplot[samples=3000,domain=-2.5:5,BLUE,line width=1.5pt] { ((x <= -0.5*\s)+(x>=0.7*\s))*(-1.5*abs(sign(x)))
            };
            \addlegendentry{ $\sigma_T$ };
          \end{axis}
        \end{tikzpicture}
      }}
    %\qquad\qquad
    \quad
    %{\scriptsize
        \begin{aligned}
      \EE[\sigma_T(\tau\xi)] & = \frac{a}{2}\left(\erf \left( \frac{s_1}{\sqrt{2}\tau} \right)-\erf\left(  \frac{s_2}{\sqrt{2}\tau} \right) \right)+a \\
      \EE[\sigma_Q(\tau\xi)] & = \frac{b_1}{2} \left(\erf\left(\cfrac{r_1}{\sqrt{2}\tau}\right)
      - \erf\left(\frac{r_4}{\sqrt{2}\tau}\right) \right)
      + b_1                                                                                                                                           \\
                             & + \frac{b_2}{2} \left(\erf\left(\frac{r_3}{\sqrt{2}\tau}\right)
      - \erf\left(\frac{r_2}{\sqrt{2}\tau}\right) \right)
    \end{aligned}
    %}%
  \end{equation*}
  \caption{ Visual representations of activations $\sigma_T$ and $\sigma_Q$ in \eqref{eq:def_sigmal} \textbf{(left)} and the expressions of $\EE[\sigma_T(\tau\xi)]$ and $\EE[\sigma_Q(\tau\xi)] $\textbf{(right)}, with $r_1 - r_2 = r_3 - r_4$ here and $\erf(\cdot)$ the Gaussian error function. }
  \label{fig:sigma}
\end{figure}
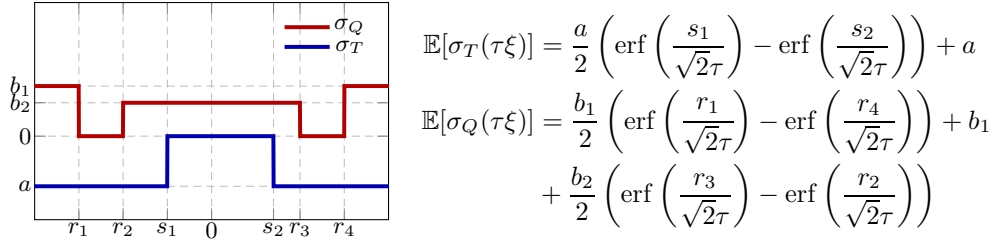

\begin{algorithm}[thb]
    \caption{ ``Lossless'' compression scheme for fully-connected DNNs } %Sparse and quantized 
    \label{alg:sparse_quantized}
\begin{algorithmic}[1]
  \STATE {{\bfseries Input:} Input data $\x_1, \ldots, \x_n$, sparsity level $\varepsilon \in [0,1)$, and ${\sf DNN1}$ with activations $\sigma_1, \ldots, \sigma_L$. }
  \STATE {{\bfseries Output:} Sparse and quantized model ${\sf DNN2}$ with weights $\W_\ell$ and activations $\tilde \sigma_\ell$, $\ell \in \{1, \ldots, L\}$. }
%   \STATE {{\bfseries Output:} Compressed Model}
  \STATE {Estimate $\tau_0$ from data as $\tau_0 = \sqrt{ \frac1n \sum_{i=1}^n \| \x_i \|^2 }$. 
  Set $\tau = \tau_0$ for ${\sf DNN1}$, and $\tilde \tau =  \tau_0$ for ${\sf DNN2}$. }
\FOR{ $\ell = 1 ,\ldots, L-1$ }
\STATE {Compute $\alpha_{\ell,1}$, $\alpha_{\ell,2}$, $\alpha_{\ell,3}$ of ${\sf DNN1}$ using $\tau_{\ell-1}$, and derive the expressions of $\tilde \alpha_{\ell,1}$, $\tilde \alpha_{\ell,2}$, $\tilde \alpha_{\ell,3}$ of ${\sf DNN2}$ using $\tilde \tau_{\ell-1}$ as per \eqref{eq:def_ds1} and \eqref{eq:def_ds2}.}
\STATE {Solve, with Corollary~\ref{coro:sparse_quantized} and the detailed expressions in Appendix~\ref{subsec:proof_coro_sparse_quantized}, the system of equations $(\alpha_{\ell,1}, \alpha_{\ell,2}, \alpha_{\ell,3}) = (\tilde \alpha_{\ell,1}, \tilde \alpha_{\ell,2}, \tilde \alpha_{\ell,3})$ for the parameters $a, s_1, s_2$, to get the activation $\tilde \sigma_\ell$ of ${\sf DNN2}$ at layer $\ell$. }
\STATE {Update $\tau, \tilde \tau$ as $\tau = \sqrt{  \EE[ \sigma_\ell^2 (\tau \xi) ] } $, and $\tilde \tau = \sqrt{  \EE[ \tilde \sigma_\ell^2 (\tilde \tau \xi) ] } $. }
\ENDFOR
\STATE { 
For the layer $\ell = L$, compute $\alpha_{L,1}$, $\alpha_{L,2}$, $\alpha_{L,3}$, $\tau_{L}$ and $\tilde \alpha_{L,1}$, $\tilde \alpha_{L,2}$, $\tilde \alpha_{L,3}$, $\tilde \tau_{L}$. 
Use them to solve for the parameters $b_1, b_2, r_1, r_2, r_3, r_4$ to obtain the activation $\tilde \sigma_L$ of ${\sf DNN2}$ at layer $\ell$. }
\STATE {Draw independently the i.i.d.\@ entries of $\W_1, \ldots, \W_L$ according to \eqref{eq:def_W} with sparsity level $\varepsilon$. }
\RETURN {${\sf DNN2}$ model with weights $\W_\ell$ and activations $\tilde \sigma_\ell$, $\ell = 1, \ldots, L$. }
  %\STATE {Set $\tau = \sqrt{  \EE[ \sigma_\ell^2 (\tau \xi) ] } $, $\tilde \tau = \sqrt{  \EE[ \sigma_\ell^2 (\tilde \tau \xi) ] } $. }
\end{algorithmic}
\end{algorithm}

As a side remark, note that the ``sign'' of activations does not matter in Corollary~\ref{coro:sparse_quantized}~or~Algorithm~\ref{ass:high-dimen}, in the sense that the key parameters $\alpha_{\ell,1}, \alpha_{\ell,2}$, and $ \alpha_{\ell,3}$ for CKs,
as well as $\beta_{\ell,1}, \beta_{\ell,2}$, and $ \beta_{\ell,3}$ for NTKs, remain unchanged when $-\sigma_\ell(t)$ is used instead of $\sigma_\ell(t)$.

Before embarking on the detailed numerical experiments in Section~\ref{sec:experiments}, we would like to bring the readers' attention to the recent line of works \cite{lee2018snip,su2020sanity,Wang2020Picking,tanaka2020pruning,frankle2021pruning} showing that for wide and deep NN models, very efficient sparse sub-networks can be found that almost match the performance of the original dense nets \emph{with little or even no training}, for instance by uniformly pruning the network weights \cite{su2020sanity}.
To develop a theoretical grasp of these (extremely counterintuitive) empirical successes, a few attempts have been made, for example, to carefully prune the network weights to retain the same (limiting) NTK \cite{liu20finding}, or to show that randomly pruned sparse nets have the same (limiting) NTK as the original net up to a scaling factor \cite{yang2022on}. 
Instead, our work, by considering the \emph{statistical structure} of the input data, leverages tools from RMT to ``compress'' both the weights and activations (per Theorem~\ref{theo:NTK}~and~Corollary~\ref{coro:sparse_quantized}) without affecting the NTK eigenstructure.

\section{Numerical experiments}
\label{sec:experiments}

In this section, we provide numerical experiments to (i) validate the asymptotic characterizations in Theorem~\ref{theo:CK}~and~\ref{theo:NTK}, on both synthetic GMM and real-world data (such as MNIST and CIFAR10) of (in fact not so) large sizes and dimensions; 
and to (ii) show how these results can be used to sparsify and quantize fully-connected DNNs, leading to huge savings in computational and storage resources (up to a factor of $10^3$ in memory and a level of sparsity $\varepsilon=90\%$) without significant performance degradation. 
We refer readers to Section~\ref{sec:additional_experiments} in the appendix for further experiments and discussions.
The code to reproduce the numerical results in this section is publicly available at \url{https://github.com/Model-Compression/Lossless_Compression}.

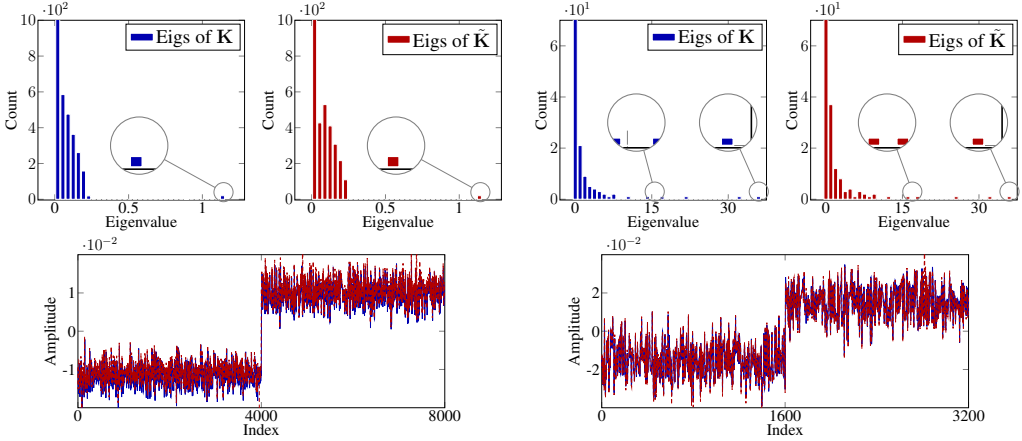
\begin{figure}[htb]
  \centering
  \resizebox{\textwidth}{!}{
  \begin{tabular}{cc}
  \begin{tabular}{cc}
    \begin{tikzpicture}[font=\large,spy using outlines, inner sep=1.2]
    \renewcommand{\axisdefaulttryminticks}{4} 
    \pgfplotsset{every major grid/.append style={densely dashed}}  
    \tikzstyle{every axis y label}+=[yshift=-10pt] 
    \tikzstyle{every axis x label}+=[yshift=5pt]
    \pgfplotsset{every axis legend/.append style={cells={anchor=west},fill=white, at={(0.98,0.98)}, anchor=north east, font=\Large }}
    \begin{axis}[
      %ybar,
      xlabel={Eigenvalue},
      ylabel={Count},
      height=.45\textwidth,
      width=.5\textwidth,
      xmin=-0.08,
      ymin=-4,
      xmax=0.9,
      ymax=1000,
      xtick={0,0.35,0.7},
      xticklabels={0,0.5,1},
      bar width=3pt,
      grid=none,
      ymajorgrids=false,
    %   scaled ticks=true,
      scaled y ticks=base 10:-2,
      ]
      %art
      %% empirical Gram matrix
      \addplot[area legend,ybar,mark=none,color=white,fill=BLUE] coordinates{
      (0.013922050835754755,6136)(0.03834835377640079,585)(0.06277465671704684,476)(0.08720095965769287,363)(0.11162726259833891,261)(0.13605356553898496,158)(0.160479868479631,19)(0.18490617142027704,0)(0.2093324743609231,0)(0.23375877730156913,0)(0.25818508024221515,0)(0.2826113831828612,0)(0.30703768612350724,0)(0.3314639890641533,0)(0.3558902920047993,0)(0.3803165949454454,0)(0.4047428978860914,0)(0.42916920082673743,0)(0.4535955037673835,0)(0.4780218067080295,0)(0.5024481096486756,0)(0.5268744125893216,0)(0.5513007155299676,0)(0.5757270184706137,0)(0.6001533214112597,0)(0.6245796243519057,1)(0.6490059272925518,0)(0.6734322302331978,0)(0.6978585331738438,0)(0.7222848361144899,0)(0.746711139055136,0)(0.7711374419957819,0)(0.795563744936428,20)
      };
    %   (0.8199900478770741,0)(0.84441635081772,0)(0.8688426537583661,0)(0.8932689566990122,0)(0.9176952596396583,0)(0.9421215625803042,0)(0.9665478655209503,0)(0.9909741684615964,0)(1.0154004714022424,0)(1.0398267743428884,0)(1.0642530772835344,0)(1.0886793802241805,0)(1.1131056831648265,0)(1.1375319861054727,0)(1.1619582890461186,0)(1.1863845919867646,0)(1.2108108949274108,20)
      \addlegendentry{ {Eigs of $\K$} };
      \coordinate (spypoint) at (axis cs:0.8,40);
      \coordinate (magnifyglass) at (axis cs:0.4,300);
      \end{axis}
      \spy[black!50!white,size=1.5cm,circle,connect spies,magnification=3] on (spypoint) in node [fill=none] at (magnifyglass);
\end{tikzpicture}
&
     \begin{tikzpicture}[font=\large,spy using outlines, inner sep=1.2]
    \renewcommand{\axisdefaulttryminticks}{4} 
    \pgfplotsset{every major grid/.append style={densely dashed}}       
    \tikzstyle{every axis y label}+=[yshift=-10pt] 
    \tikzstyle{every axis x label}+=[yshift=5pt]
    \pgfplotsset{every axis legend/.append style={cells={anchor=west},fill=white, at={(0.98,0.98)}, anchor=north east, font=\Large }}
    \begin{axis}[
      %ybar,
      xlabel={Eigenvalue},
      ylabel={Count},
      height=.45\textwidth,
      width=.5\textwidth,
      xmin=-0.08,
      ymin=-4,
      xmax=0.9,
      ymax=1000,
      xtick={0,0.35,0.7},
      xticklabels={0,0.5,1},
      bar width=3pt,
      grid=none,
      ymajorgrids=false,
      scaled ticks=true,
      scaled y ticks=base 10:-2,
      ]
\addplot[area legend,ybar,mark=none,color=white,fill=RED] coordinates{
(0.013922050835754755,5999)(0.03834835377640079,427)(0.06277465671704684,529)(0.08720095965769287,410)(0.11162726259833891,308)(0.13605356553898496,216)(0.160479868479631,109)(0.18490617142027704,0)(0.2093324743609231,0)(0.23375877730156913,0)(0.25818508024221515,0)(0.2826113831828612,0)(0.30703768612350724,0)(0.3314639890641533,0)(0.3558902920047993,0)(0.3803165949454454,0)(0.4047428978860914,0)(0.42916920082673743,0)(0.4535955037673835,0)(0.4780218067080295,0)(0.5024481096486756,0)(0.5268744125893216,0)(0.5513007155299676,0)(0.5757270184706137,0)(0.6001533214112597,0)(0.6245796243519057,0)(0.6490059272925518,0)(0.6734322302331978,0)(0.6978585331738438,0)(0.7222848361144899,1)(0.746711139055136,0)(0.7711374419957819,0)(0.795563744936428,20)
};
% (0.8199900478770741,0)(0.84441635081772,0)(0.8688426537583661,0)(0.8932689566990122,0)(0.9176952596396583,0)(0.9421215625803042,0)(0.9665478655209503,0)(0.9909741684615964,0)(1.0154004714022424,0)(1.0398267743428884,0)(1.0642530772835344,0)(1.0886793802241805,0)(1.1131056831648265,0)(1.1375319861054727,0)(1.1619582890461186,0)(1.1863845919867646,0)(1.2108108949274108,20)
\addlegendentry{ {Eigs of $\tilde \K$} };
\coordinate (spypoint) at (axis cs:0.8,40);
\coordinate (magnifyglass) at (axis cs:0.4,300);
\end{axis}
\spy[black!50!white,size=1.5cm,circle,connect spies,magnification=3] on (spypoint) in node [fill=none] at (magnifyglass);
\end{tikzpicture}
  \end{tabular}
&
  \begin{tabular}{cc}
      \begin{tikzpicture}[font=\large,spy using outlines, inner sep=1.2]
    \renewcommand{\axisdefaulttryminticks}{4} 
    \pgfplotsset{every major grid/.append style={densely dashed}}       
    \tikzstyle{every axis y label}+=[yshift=-10pt] 
    \tikzstyle{every axis x label}+=[yshift=5pt]
    \pgfplotsset{every axis legend/.append style={cells={anchor=west},fill=white, at={(0.98,0.98)}, anchor=north east, font=\Large }}
    \begin{axis}[
      %ybar,
      xlabel={Eigenvalue},
      ylabel={Count},
      height=.45\textwidth,
      width=.5\textwidth,
      xmin=-2,
      ymin=-0.3,
      xmax=25,
      ymax=70,
      xtick={0,10,20},
      xticklabels={0,15,30},
      bar width=3pt,
      grid=none,
      ymajorgrids=false,
      scaled ticks=true,
      scaled y ticks=base 10:-1,
      %axis x line = middle,
    	  %axis y line = middle,
      %xticklabels={}
      %scaled ticks=true,
      %scale ticks above={4},
      %xlabel={Eigenvalues},
      %ylabel={Density}
      ]
      %art
      \addplot[area legend,ybar,mark=none,color=white,fill=BLUE] coordinates{
(0.0038880662164545543,3147)(0.6340298102640483,21.0)(1.264171554311642,9.0)(1.8943132983592357,5.0)(2.5244550424068293,4.0)(3.154596786454423,3.0)(3.7847385305020165,2.0)(4.414880274549611,1.0)(5.045022018597204,2.0)(5.6751637626447975,0.0)(6.305305506692392,0.0)(6.935447250739985,1.0)(7.5655889947875785,0.0)(8.195730738835172,0.0)(8.825872482882767,0.0)(9.45601422693036,1.0)(10.086155970977954,0.0)(10.716297715025547,0.0)(11.34643945907314,1.0)(11.976581203120736,0.0)(12.606722947168329,0.0)(13.236864691215922,0.0)(13.867006435263516,0.0)(14.49714817931111,1.0)(15.127289923358703,0.0)(15.757431667406298,0.0)(16.38757341145389,0.0)(17.017715155501485,0.0)(17.64785689954908,0.0)(18.27799864359667,0.0)(18.908140387644266,0.0)(19.538282131691858,0.0)(20.168423875739453,0.0)(20.79856561978705,0.0)(21.42870736383464,1.0)(22.058849107882235,0.0)(22.688990851929827,0.0)(23.319132595977422,0.0)(23.949274340025017,1)
};
% (24.57941608407261,0.0)(25.209557828120204,0.0)(25.839699572167795,0.0)(26.46984131621539,0.0)(27.099983060262986,0.0)(27.730124804310577,0.0)(28.360266548358172,0.0)(28.990408292405764,0.0)(29.62055003645336,1.0)(30.25069178050095,0.0)(30.880833524548546,0.0)
\addlegendentry{ {Eigs of $\K$} };
\coordinate (spypoint1) at (axis cs:24,3);
\coordinate (magnifyglass1) at (axis cs:20,30);
\coordinate (spypoint2) at (axis cs:10.4,3);
\coordinate (magnifyglass2) at (axis cs:8,30);
\end{axis}
\spy[black!50!white,size=1.5cm,circle,connect spies,magnification=3] on (spypoint1) in node [fill=none] at (magnifyglass1);
\spy[black!50!white,size=1.5cm,circle,connect spies,magnification=3] on (spypoint2) in node [fill=none] at (magnifyglass2);
\end{tikzpicture}
&
\begin{tikzpicture}[font=\large,spy using outlines, inner sep=1.2]
    \renewcommand{\axisdefaulttryminticks}{4} 
    \pgfplotsset{every major grid/.append style={densely dashed}}       
    \tikzstyle{every axis y label}+=[yshift=-10pt] 
    \tikzstyle{every axis x label}+=[yshift=5pt]
    \pgfplotsset{every axis legend/.append style={cells={anchor=west},fill=white, at={(0.98,0.98)}, anchor=north east, font=\Large }}
    \begin{axis}[
      %ybar,
      xlabel={Eigenvalue},
      ylabel={Count},
      height=.45\textwidth,
      width=.5\textwidth,
      xmin=-2,
      ymin=-0.3,
      xmax=25,
      ymax=70,
      xtick={0,10,20},
      xticklabels={0,15,30},
      bar width=3pt,
      grid=none,
      ymajorgrids=false,
      scaled ticks=true,
      scaled y ticks=base 10:-1,
      %axis x line = middle,
    	  %axis y line = middle,
      %xticklabels={}
      %scaled ticks=true,
      %scale ticks above={4},
      %xlabel={Eigenvalues},
      %ylabel={Density}
      ]
      \addplot[area legend,ybar,mark=none,color=white,fill=RED] coordinates{
(0.0038880662164545543,3121.0)(0.6340298102640483,37.0)(1.264171554311642,12.0)(1.8943132983592357,8.0)(2.5244550424068293,3.0)(3.154596786454423,4.0)(3.7847385305020165,1.0)(4.414880274549611,3.0)(5.045022018597204,2.0)(5.6751637626447975,1.0)(6.305305506692392,2.0)(6.935447250739985,0.0)(7.5655889947875785,0.0)(8.195730738835172,1.0)(8.825872482882767,0.0)(9.45601422693036,0.0)(10.086155970977954,0.0)(10.716297715025547,1.0)(11.34643945907314,0.0)(11.976581203120736,1.0)(12.606722947168329,0.0)(13.236864691215922,0.0)(13.867006435263516,0.0)(14.49714817931111,0.0)(15.127289923358703,0.0)(15.757431667406298,0.0)(16.38757341145389,0.0)(17.017715155501485,1.0)(17.64785689954908,0.0)(18.27799864359667,0.0)(18.908140387644266,0.0)(19.538282131691858,0.0)(20.168423875739453,0.0)(20.79856561978705,0.0)(21.42870736383464,1.0)(22.058849107882235,0.0)(22.688990851929827,0.0)(23.319132595977422,0.0)(23.949274340025017,1.0)
};
% (24.57941608407261,0.0)(25.209557828120204,0.0)(25.839699572167795,0.0)(26.46984131621539,0.0)(27.099983060262986,0.0)(27.730124804310577,0.0)(28.360266548358172,0.0)(28.990408292405764,0.0)(29.62055003645336,0.0)(30.25069178050095,0.0)(30.880833524548546,1.0)
\addlegendentry{ {Eigs of $\tilde \K$} };
\coordinate (spypoint1) at (axis cs:24,3);
\coordinate (magnifyglass1) at (axis cs:20,30);
\coordinate (spypoint2) at (axis cs:11.3,3);
\coordinate (magnifyglass2) at (axis cs:8,30);
\end{axis}
\spy[black!50!white,size=1.5cm,circle,connect spies,magnification=3] on (spypoint1) in node [fill=none] at (magnifyglass1);
\spy[black!50!white,size=1.5cm,circle,connect spies,magnification=3] on (spypoint2) in node [fill=none] at (magnifyglass2);
\end{tikzpicture}
  \end{tabular}
  \\
\begin{tikzpicture}[font=\large, inner sep=1.2]
    \renewcommand{\axisdefaulttryminticks}{4} 
    \pgfplotsset{every major grid/.append style={densely dashed}}       
    \tikzstyle{every axis y label}+=[yshift=-10pt] 
    \tikzstyle{every axis x label}+=[yshift=5pt]
    \pgfplotsset{every axis legend/.append style={cells={anchor=west},fill=white, at={(0.98,0.98)}, anchor=north east, font=\large }}
    \begin{axis}[
      %ybar,
      xlabel={Index},
      ylabel={Amplitude},
      height=.4\textwidth,
      width=.8\textwidth,
      xmin=0,
      ymin=-0.02,
      xmax=1600,
      ymax=0.02,
      xtick={0,800,1600},
      xticklabels={0,4000,8000},
      ytick={-0.01,0,0.01},
      yticklabels={-1,0,1},
      grid=none,
      ymajorgrids=false,
      scaled ticks=true,
      %axis x line = middle,
    	  %axis y line = middle,
      %xticklabels={}
      %scaled ticks=true,
      %scale ticks above={4},
      %xlabel={Eigenvalues},
      %ylabel={Density}
      ]
      %art
\addplot [smooth,BLUE,line width=0.5pt] file {eigenvector_gaussion_K.txt};
\addplot [densely dashed,RED,line width=0.5pt] file {eigenvector_gaussion_tildeK.txt};
\end{axis}
\end{tikzpicture} 
&
\begin{tikzpicture}[font=\large, inner sep=1.2]
    \renewcommand{\axisdefaulttryminticks}{4} 
    \pgfplotsset{every major grid/.append style={densely dashed}}       
    \tikzstyle{every axis y label}+=[yshift=-10pt] 
    \tikzstyle{every axis x label}+=[yshift=5pt]
    \pgfplotsset{every axis legend/.append style={cells={anchor=west},fill=white, at={(0.98,0.98)}, anchor=north east, font=\large }}
    \begin{axis}[
      %ybar,
      xlabel={Index},
      ylabel={Amplitude},
      height=.4\textwidth,
      width=.8\textwidth,
      xmin=0,
      ymin=-0.04,
      xmax=1066,
      ymax=0.04,
      xtick={0,533,1066},
      xticklabels={0,1600,3200},
      ytick={-0.02,0,0.02},
      yticklabels={-2,0,2},
      grid=none,
      ymajorgrids=false,
      scaled ticks=true,
      %axis x line = middle,
    	  %axis y line = middle,
      %xticklabels={}
      %scaled ticks=true,
      %scale ticks above={4},
      %xlabel={Eigenvalues},
      %ylabel={Density}
      ]
      %art
\addplot [smooth,BLUE,line width=0.5pt] file {eigenvector_bernuli_K.txt};
\addplot [densely dashed,RED,line width=0.5pt] file {eigenvector_bernuli_tildeK.txt};
\end{axis}
\end{tikzpicture}
\end{tabular}}
\caption{Eigenvalue histograms (\textbf{top}) and dominant eigenvectors (\textbf{bottom}) of last-layer CK matrices $\K_{\CK}$ ({\BLUE \textbf{blue}}) defined in \eqref{eq:def_K_CK_K_NTK} (with expectation estimated from $1\,000$ independent realizations of $\W$s) and the asymptotic equivalent $\tilde \K_{\CK}$ ({\RED \textbf{red}}) matrices.
(\textbf{Left}) Gaussian $\W$ on two-class GMM data, with $p=2\,000$, $n=8\,000$, $\bmu_a=[\zo_{8(a-1)};~8;~\zo_{p-8a+7}], \C_a=(1+8(a-1)/\sqrt{p})\I_p$, $a \in \{ 1,2\}$ using $[\ReLU,~\ReLU,~\ReLU]$ activations, here $\| \K_{\CK} - \tilde \K_{\CK} \| = 0.15$; and (\textbf{right}) symmetric Bernoulli $\W$ on MNIST data (number $6$ versus $8$) \cite{lecun1998gradient}, with $p=784$, $n=3\,200$, using $[\poly,~\ReLU,~\ReLU]$ activations, $\| \K_{\CK} - \tilde \K_{\CK} \|  = 6.86$. $\x_1, \ldots, \x_{n/2} \in \mathcal{C}_1$ and $\x_{n/2+1}, \ldots, \x_n \in \mathcal{C}_2$ in both cases.}

% {\RED \bf Since the number of the smallest eigenvalue is much larger than the second smallest one, here we cut the number of eigenvalue. In fact, the \textbf{left} smallest eigenvalues are 6136 and 5999, and \textbf{right} are 3121 and 3147 respectively.} }
\label{fig:recursion relation}
\end{figure}

Figure~\ref{fig:recursion relation} compares the eigenvalues and dominant eigenvectors of the CK matrices $\K_{\CK}$ defined in \eqref{eq:def_K_CK_K_NTK} versus those of their asymptotic approximations $\tilde \K_{\CK}$ given in Theorem~\ref{theo:CK}, in the case of fully-connected DNNs having three hidden layers (of width $d_1 = 2\,000, d_2 = 2\,000, d_3 =1\,000$).
For different types of activations: $\poly(t) $ = $0.2t^2+t$ and $\ReLU(t) = \max(t,0)$, different weight distributions (Gaussian and symmetric Bernoulli), and on synthetic GMM as well as MNIST data,

we consistently observe a close match between the eigenvalues and dominant eigenvectors of $\K_{\CK}$ and $\tilde \K_{\CK}$, as a consequence of the spectral norm convergence in Theorem~\ref{theo:CK} (and Remark~\ref{rem:spectral}), suggesting a possibly wider applicability of the proposed results beyond GMM data.\footnote{ Small eigenvalues of $\K_{\CK},\tilde \K_{\CK}$ close to zero are removed from Figure~\ref{fig:recursion relation} for better visualization. }

Following the idea of CK matching in Figure~\ref{fig:recursion relation}, Figure~\ref{fig:classif-perf} depicts the test accuracies of (i) the original dense and unquantized network with three fully-connected layers, (ii) the proposed ``lossless'' compression scheme described in Corollary~\ref{coro:sparse_quantized}~and~Algorithm~\ref{alg:sparse_quantized} via CK matching, as well as its variant having ternarized weights but \emph{dense and unquantized} activations, (iii) the popular magnitude-based pruning approach as in \cite{gale2019state}, together with (iv) two ``heuristic'' compression approaches: (iv-i) sparsification by uniformly zeroing out $80\%$ of the weights (we \emph{cannot} do more, as the resultant performance is too poor to be visually compared with other curves)
% in Figure~\ref{fig:classif-perf})
, and (iv-ii) binarization using $\sigma(t) =  1_{t < -1}+1_{t > 1}$, for different choices of width per layer, and the ten-class classification problems of MNIST and CIFAR10. 
Specifically, neural networks before and after compression have three fully-connected layers, and the original network uses $\ReLU$ activation in each layer.
Classification is performed on a concatenated and trainable fully-connected layer.
For MNIST datasets, raw data are taken as the network input; for CIFAR10 dataset, flattened output of the 16th convolutional layer of VGG19 \cite{simonyan2014very} are taken as the network input.

We observe from Figure~\ref{fig:classif-perf} that the proposed ``lossless'' compression scheme produces significantly sparser networks (up to $90\%$ of weights set to zero) with minimal performance loss, while occupying (up to) a factor of $10^3$ less memory, when compared to the original or the heuristically compressed nets. 
Also, higher accuracies are obtained with the proposed approach than, e.g., the popular magnitude-based pruning under the same memory budget. 
In addition, the ternary weights variant (with \emph{quantized} weights \emph{only}) of the proposed scheme can achieve even better performance (with virtually no performance loss compared to the original dense net), however at the price of not conducive to inference accelerating, since the activations are \emph{unquantized}. 
We also see that the sparsity level $\varepsilon$ has limited impact on the classification accuracy, as in line with our theory.

These experimental results show that the proposed approach achieves a better \emph{performance-complexity trade-off} than commonly used heuristic DNN compression methods.

\begin{figure}[bht]
  \centering
\begin{tikzpicture}[font=\footnotesize][scale=0.5]
\pgfplotsset{every major grid/.style={style=densely dashed}}
\pgfplotsset{set layers=standard,cell picture=true}

\begin{groupplot}[
    % legend style={
    %     % cells={anchor= south west},
    %     % cells={anchor=east},
    %     legend pos=outer north east，
    %     at={(1.6, 0.4)}
    %     },
    legend style={
    at={(1.5,-0.8)},
    anchor= south east},
      group style={group size=1 by 2,ylabels at=edge left},
      ylabel style={text height=0.02\textwidth,inner ysep=0pt},
      height=0.475\linewidth,width=0.475\linewidth,/tikz/font=\small,
    %   x descriptions at=edge bottom,
    %   y descriptions at=edge left,
    %   vertical sep=2cm
      ]
    \nextgroupplot
        [width=0.7\textwidth,height=0.3\textwidth,
        xmode = log,grid=major,xlabel={Memory (bits)},ylabel={Accuracy}, 
        xmin=900000,xmax=30000000000,ymin=0.48, ymax=1,  
         xtick={1000000, 10000000, 100000000, 1000000000, 10000000000}, ytick={0.4, 0.6, 0.8, 1},
    xticklabels={$10^6$, $10^7$, $10^8$,$10^9$,$10^{10}$},
    yticklabels={$0.4$, $0.6$,$0.8$, $1$}
    ]
    \addlegendentry{$\varepsilon = 0\%$}
     \addplot[mark size=1.5pt,thick, BLUE, mark=triangle*]coordinates{   
     (11391000,0.7544)(77904000,0.8129)(107970000,0.8764)(315830000,0.8812)(1271682000,0.9152)
    };
    % % 0
    \addlegendentry{$\varepsilon = 50\%$}
     \addplot[mark size=1.5pt,thick, BLUE, mark=square*]coordinates{            
    (5700000,0.7636)(38964000,0.8071)(54000000,0.8751)(157940000,0.8822)(635892000,0.9162)
    };
    % % 50
    \addlegendentry{$\varepsilon = 90\%$}
     \addplot[mark size=1.5pt,thick, BLUE , mark=*]coordinates{   
    (1147200,0.7515)(7812000,0.8159)(10824000,0.8788)(31628000,0.8855)(127260000,0.9176)
    };
    % % 90
    \addlegendentry{original dense}
     \addplot[mark size=1.5pt,thick, RED, mark=*]coordinates{            
    (182256000,0.9279)(1246464000,0.9435)(1727520000,0.9529)(5053280000,0.9549)(20346912000,0.9819)
    };
    % %origin
    \addlegendentry{ naive quantized}
     \addplot[mark size=1.5pt,thick, GREEN, mark=*]coordinates{            
    (182163000,0.4837)(1246154000, 0.5740)(1727210000,0.6704)(5052660000,0.6705)(20345672000,0.8178)
    };
    % %origin_quan
    \addlegendentry{naive sparse}
     \addplot[mark size=1.5pt,thick, PURPLE, mark=*]coordinates{            
    (36566400, 0.744)(249600000,0.7581)(345888000,0.7926)(1011296000,0.7869)(4070688000,0.8710)
    };
    % %origin_spar_0.8
    \addlegendentry{mag-based pruning}
    \addplot[mark size=1.5pt,thick, densely dashdotted, ORANGE, mark=*]coordinates{
          (18355200 , 0.9118)(124992000,0.9234 )(173184000, 0.9241)(506048000, 0.9287)(2036160000,0.9185)};
    % %amplitude pruning_0.9
    \addlegendentry{ternary weights}
    \addplot[mark size=1.5pt,thick, densely dashdotted, BROWN, mark=*]coordinates{            
    (1282200, 0.951)(8172000,0.958)(11274000,0.9713)(32378000,0.9719)(128790000,0.9769)
    };
    %\bigskip
    \nextgroupplot[width=0.7\textwidth,height=0.3\textwidth,
          xmode = log, grid=major,xlabel={Memory (bits)},ylabel={Accuracy}, xmin=44000,xmax=10000000000,ymin=0.725,ymax=0.9,   xtick={ 100000, 1000000, 10000000, 100000000, 1000000000, 10000000000}, ytick={0.725,  0.8, 0.85, 0.9},
      xticklabels={$10^5$, $10^6$,$10^7$, $10^8$,$10^9$,$10^{10}$},
      yticklabels={$0.65$, $0.8$, $0.85$, $0.9$},
      ]
  % 400,400,200; 1000,1000,500; 3000,3000,1000; 5000,5000,2500; 10000,10000,5000;
%   \addplot[mark size=1.5pt, blue, mark=triangle, thick]coordinates{           
%   (11432, 0.8531)(46580, 0.8514)(288640, 0.862)(832900, 0.871)(3165800, 0.871)
%   };
  % 99
%   \addplot[mark size=1pt, BLUE, mark=triangle*]coordinates{   
%   (945200,0.7996)(4163000,0.841)(27478000,0.8586)(80815000,0.863)(311630000,0.8692)
%   };
  \addplot[mark size=1.5pt, thick ,BLUE, mark=triangle*]coordinates{
  (945200,0.7998)(4163000,0.841)(27478000,0.8586)(80815000,0.863)(311630000,0.8692)
  };
  % 0
%   \addplot[mark size=1pt, BLUE, mark=square*]coordinates{   
%   (473600, 0.7933)(2084000,0.838)(13746000,0.8615)(40420000,0.8686)(155840000,0.8718)
%   };
  \addplot[mark size=1.5pt,thick, BLUE, mark=square*]coordinates{   
  (473600, 0.79665)(2084000,0.838)(13746000,0.8615)(40420000,0.8686)(155840000,0.8718)
  };
  % 50
%   \addplot[mark size=1pt, BLUE, mark=*]coordinates{            
%   (96320, 0.7898)(420800, 0.836)(2760400, 0.8547)(8104000, 0.8680)(31208000, 0.8685)
%   };
  \addplot[mark size=1.5pt,thick, BLUE, mark=*]coordinates{            
  (96320, 0.7949)(420800, 0.836)(2760400, 0.8547)(8104000, 0.8680)(31208000, 0.8685)
  };
  % 90
  \addplot[mark size=1.5pt,thick, RED, mark=*]coordinates{            
  (15123200,0.8502)(66608000,0.8714)(439648000,0.8719)(1293040000,0.8721)(4986080000,0.8726)
  };
  % origin
%   \addplot[mark size=1pt, GREEN, mark=*]coordinates{            
%   (15098400,0.6920)(66546000, 0.7188)(439462000,0.7554)(1292730000,0.7559)(4985460000,0.7640)
%   };
  \addplot[mark size=1.5pt,thick, GREEN, mark=*]coordinates{            
  (15098400,0.746)(66546000, 0.7594)(439462000,0.7777)(1292730000,0.77795)(4985460000,0.782)
  };
  %origin_quantization
  \addplot[mark size=1.5pt,thick, PURPLE, mark=*]coordinates{            
  (3050240, 0.7755)(13385600,0.8273)(88108800, 0.8602)(258928000,0.8521)(997856000,0.8633)
  };
  %origin_spar_0.8
  \addplot[mark size=1.5pt,thick, densely dashdotted, ORANGE, mark=*]coordinates{
      (746880 , 0.8057)(3307200,0.8529 )(21897600, 0.8635)(64536000, 0.8747)(249072000,0.8603)};
      %  pruning 0.95 
  \addplot[mark size=1.5pt, thick, densely dashdotted,  BROWN, mark=*]coordinates{
      (46680 , 0.8590)(206700,0.8437 )(1368600, 0.8669)(4033500, 0.8720)(15567000,0.8708)};
      % ternary 0.95 + no activation ternary
% %   \legend{$\varepsilon = 0\%$, $\varepsilon = 50\%$, $\varepsilon = 90\%$, original, naive quantized, naive sparse };
%     \label{plots:DEF}% label for the second plot
%     \coordinate (bot) at (rel axis cs:1,0);% coordinate at bottom of the last plot
  \end{groupplot}
\end{tikzpicture}  
  \caption{ Classification accuracies of different compressed fully-connected nets on MNIST~\cite{lecun1998gradient} (\textbf{top}) and CIFAR10~\cite{krizhevsky2009learning} (\textbf{bottom}) datasets.  
  {\BLUE \bf Blue} curves represent the proposed compression approach with different levels of sparsity $\varepsilon \in \{ 0\%, 50\%, 90\% \}$, 
  {\PURPLE \bf purple} curves represent the heuristic sparsification approach by uniformly zeroing out $80\%$ of the weights,
  {\GREEN \bf green} curves represent the heuristic quantization approach using the binary activation
  $\sigma (t) =  1_{t < -1}+ 1_{t > 1}$ 
%   (only applied to the first two layers, otherwise the performance is too poor to be compared to other curves)
%   \footnote{only applied to the first two layers, otherwise the performance is too poor to be compared to other curves}
  , {\RED \bf red} curves represent the original network,
  %(note we just replace the first two layer's activation otherwise network performance can be so bad that we can not plot in this figure).
  {\BROWN \bf brown} curves represent the proposed compression approach \emph{without} activation quantization, with $\varepsilon=90\%$ for MNIST (\textbf{top}) and $\varepsilon=95\%$ for CIFAR10 (\textbf{bottom}),
%   we use ternary weights and level of sparsity for MNIST is $90\%$,and $95\%$ for CIFAR10.
  and {\ORANGE \bf orange} curves represent magnitude-based pruning \cite{gale2019state} with the same sparsity level $\varepsilon$ as {\BROWN \bf brown}.
  Memory varies due to the \textbf{change of layer width} of the network.}
  \label{fig:classif-perf}
\end{figure}
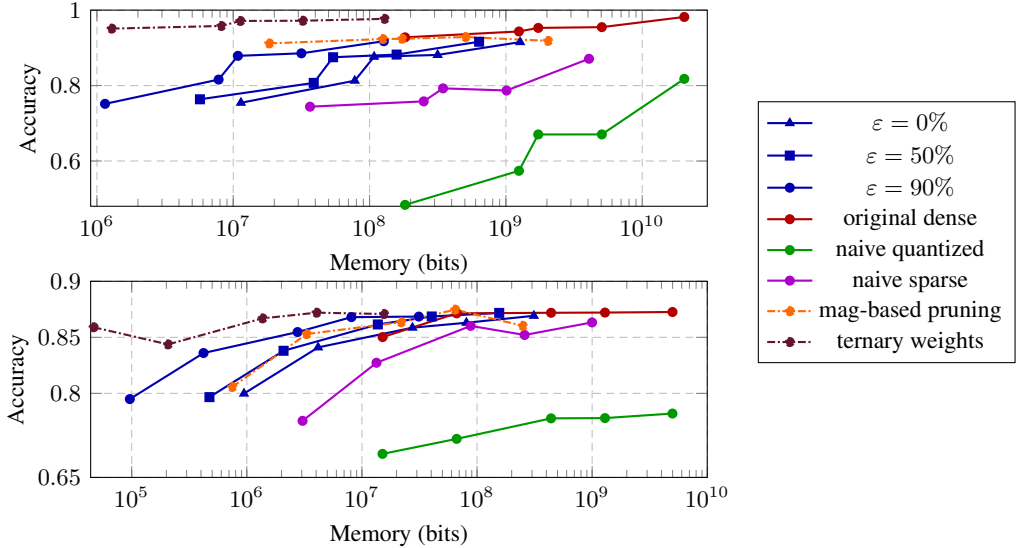

\section{Conclusion and perspectives}
\label{sec:conclusion}

In this paper, built upon recent advances in random matrix theory and high-dimensional statistics, we provide \emph{precise} characterizations of the eigenspectra of both conjugate kernel and neural tangent kernel matrices, for high-dimensional Gaussian mixture data and fully-connected multi-layer neural nets.
These results further allows us to sparsify and quantize fully-connected deep nets, resulting in a factor of $10^3$ less memory consumption with virtually no performance degradation.

Extending the present theoretical framework to more involved settings such as convolutional nets requires refined analysis on the block Toeplitz weights $\W$ and on their connection to the corresponding CK and NTK matrices \cite{arora2019exact,xiao2018dynamical}. 
Also, since the NTK eigenspectra determine the gradient descent dynamics of ultra-wide deep nets, the asymptotic characterizations in Theorem~\ref{theo:NTK} can be applied to assess the learning dynamics \cite{liao2018dynamics,advani2020high,advani2020high} of fully-connected DNN models, in and possibly beyond the infinitely wide NTK regime.

\begin{ack}
ZL would like to acknowledge the CCF-Hikvision Open Fund (20210008), the National Natural Science Foundation of China (via fund NSFC-62206101 and NSFC-12141107), the Fundamental Research Funds for the Central Universities of China (2021XXJS110), the Key Research and Development Program of Hubei (2021BAA037), and the Key Research and Development Program of Guangxi (GuiKe-AB21196034) for providing partial support.
\end{ack}

{
\small

\bibliography{liao}

\begin{thebibliography}{10}

\bibitem{adlam2020neural}
Ben Adlam and Jeffrey Pennington.
\newblock {The Neural Tangent Kernel in High Dimensions: Triple Descent and a
  Multi-Scale Theory of Generalization}.
\newblock In {\em Proceedings of the 37th International Conference on Machine
  Learning}, volume 119 of {\em Proceedings of Machine Learning Research},
  pages 74--84. PMLR, 2020.

\bibitem{advani2020high}
Madhu~S. Advani, Andrew~M. Saxe, and Haim Sompolinsky.
\newblock {High-dimensional dynamics of generalization error in neural
  networks}.
\newblock {\em Neural Networks}, 132:428--446, 2020.

\bibitem{alemohammad2020recurrent}
Sina Alemohammad, Zichao Wang, Randall Balestriero, and Richard Baraniuk.
\newblock {The Recurrent Neural Tangent Kernel}.
\newblock {\em arXiv}, 2020.

\bibitem{ali2022random}
Hafiz~Tiomoko Ali, Zhenyu Liao, and Romain Couillet.
\newblock Random matrices in service of {ML} footprint: ternary random features
  with no performance loss.
\newblock In {\em International Conference on Learning Representations}, 2022.

\bibitem{arora2019exact}
Sanjeev Arora, Simon~S Du, Wei Hu, Zhiyuan Li, Russ~R Salakhutdinov, and
  Ruosong Wang.
\newblock {On Exact Computation with an Infinitely Wide Neural Net}.
\newblock In {\em Advances in Neural Information Processing Systems},
  volume~32, 2019.

\bibitem{bai2010spectral}
Zhidong Bai and Jack~W. Silverstein.
\newblock {\em {Spectral Analysis of Large Dimensional Random Matrices}},
  volume~20 of {\em Springer Series in Statistics}.
\newblock Springer-Verlag New York, 2 edition, 2010.

\bibitem{benigni2019eigenvalue}
Lucas Benigni and Sandrine Péché.
\newblock {Eigenvalue distribution of nonlinear models of random matrices}.
\newblock {\em arXiv}, 2019.

\bibitem{bietti2019inductive}
Alberto Bietti and Julien Mairal.
\newblock {On the Inductive Bias of Neural Tangent Kernels}.
\newblock In {\em Advances in Neural Information Processing Systems}, volume~32
  of {\em NIPS'19}, pages 12893--12904. Curran Associates, Inc., 2019.

\bibitem{blum2020foundations}
Avrim Blum, John Hopcroft, and Ravindran Kannan.
\newblock {\em {Foundations of Data Science}}.
\newblock Cambridge University Press, 2020.

\bibitem{chizat2019lazy}
Lénaïc Chizat, Edouard Oyallon, and Francis Bach.
\newblock {On Lazy Training in Differentiable Programming}.
\newblock In {\em Advances in Neural Information Processing Systems}, volume~32
  of {\em NIPS'19}, pages 2937--2947. Curran Associates, Inc., 2019.

\bibitem{couillet2016kernel}
Romain Couillet and Florent Benaych-Georges.
\newblock {Kernel spectral clustering of large dimensional data}.
\newblock {\em Electronic Journal of Statistics}, 10(1):1393--1454, 2016.

\bibitem{couillet_liao_2022}
Romain Couillet and Zhenyu Liao.
\newblock {\em {Random Matrix Methods for Machine Learning}}.
\newblock Cambridge University Press.

\bibitem{couillet2018classif}
Romain Couillet, Zhenyu Liao, and Xiaoyi Mai.
\newblock {Classification Asymptotics in the Random Matrix Regime}.
\newblock In {\em 2018 26th European Signal Processing Conference (EUSIPCO)},
  volume~00 of {\em EUSIPCO'18}, pages 1875--1879, 2018.

\bibitem{deng2020model}
By~Lei Deng, Guoqi Li, Song Han, Luping Shi, and Yuan Xie.
\newblock {Model Compression and Hardware Acceleration for Neural Networks: A
  Comprehensive Survey}.
\newblock {\em Proceedings of the IEEE}, 108(4):485--532, 2020.

\bibitem{derezinski2020exact}
Michal Derezinski, Feynman~T Liang, and Michael~W Mahoney.
\newblock Exact expressions for double descent and implicit regularization via
  surrogate random design.
\newblock In {\em Advances in Neural Information Processing Systems},
  volume~33, pages 5152--5164. Curran Associates, Inc., 2020.

\bibitem{du2019graph}
Simon~S Du, Kangcheng Hou, Russ~R Salakhutdinov, Barnabas Poczos, Ruosong Wang,
  and Keyulu Xu.
\newblock {Graph Neural Tangent Kernel: Fusing Graph Neural Networks with Graph
  Kernels}.
\newblock 32.

\bibitem{fan2019spectral}
Zhou Fan and Andrea Montanari.
\newblock {The spectral norm of random inner-product kernel matrices}.
\newblock {\em Probability Theory and Related Fields}, 173(1-2):27--85, 2019.

\bibitem{fan2020spectra}
Zhou Fan and Zhichao Wang.
\newblock {Spectra of the Conjugate Kernel and Neural Tangent Kernel for
  linear-width neural networks}.
\newblock In {\em Advances in Neural Information Processing Systems},
  volume~33, pages 7710--7721. Curran Associates, Inc., 2020.

\bibitem{frankle2021pruning}
Jonathan Frankle, Gintare~Karolina Dziugaite, Daniel Roy, and Michael Carbin.
\newblock Pruning neural networks at initialization: Why are we missing the
  mark?
\newblock In {\em International Conference on Learning Representations}, 2021.

\bibitem{gale2019state}
Trevor Gale, Erich Elsen, and Sara Hooker.
\newblock The state of sparsity in deep neural networks.
\newblock {\em arXiv preprint arXiv:1902.09574}, 2019.

\bibitem{1995The}
J.~Han and C.~Moraga.
\newblock The influence of the sigmoid function parameters on the speed of
  backpropagation learning.
\newblock 1995.

\bibitem{han2015deep}
Song Han, Huizi Mao, and William~J Dally.
\newblock {Deep Compression: Compressing Deep Neural Networks with Pruning,
  Trained Quantization and Huffman Coding}.
\newblock {\em arXiv}, 2015.

\bibitem{hoefler2021sparsity}
Torsten Hoefler, Dan Alistarh, Tal Ben-Nun, Nikoli Dryden, and Alexandra Peste.
\newblock Sparsity in deep learning: Pruning and growth for efficient inference
  and training in neural networks.
\newblock {\em Journal of Machine Learning Research}, 22(241):1--124, 2021.

\bibitem{horn2012matrix}
Roger~A. Horn and Charles~R. Johnson.
\newblock {\em {Matrix Analysis}}.
\newblock Cambridge University Press, 2 edition, 2012.

\bibitem{howard2017mobilenets}
Andrew~G Howard, Menglong Zhu, Bo~Chen, Dmitry Kalenichenko, Weijun Wang,
  Tobias Weyand, Marco Andreetto, and Hartwig Adam.
\newblock {MobileNets: Efficient Convolutional Neural Networks for Mobile
  Vision Applications}.
\newblock {\em arXiv}, 2017.

\bibitem{huang2020dynamics}
Jiaoyang Huang and Horng-Tzer Yau.
\newblock {Dynamics of Deep Neural Networks and Neural Tangent Hierarchy}.
\newblock 119:4542--4551, 2020-13.

\bibitem{hubara2016binarized}
Itay Hubara, Matthieu Courbariaux, Daniel Soudry, Ran El-Yaniv, and Yoshua
  Bengio.
\newblock {Binarized Neural Networks}.
\newblock In {\em Advances in Neural Information Processing Systems}, volume~29
  of {\em NIPS‘16}, pages 4107--4115. Curran Associates, Inc., 2016.

\bibitem{jacot2018neural}
Arthur Jacot, Franck Gabriel, and Clément Hongler.
\newblock {Neural Tangent Kernel: Convergence and Generalization in Neural
  Networks}.
\newblock In {\em Advances in Neural Information Processing Systems}, volume~31
  of {\em NIPS'18}, pages 8571--8580. Curran Associates, Inc., 2018.

\bibitem{jacot2020implicit}
Arthur Jacot, Berfin Simsek, Francesco Spadaro, Clement Hongler, and Franck
  Gabriel.
\newblock Implicit regularization of random feature models.
\newblock In {\em Proceedings of the 37th International Conference on Machine
  Learning}, volume 119 of {\em Proceedings of Machine Learning Research},
  pages 4631--4640. PMLR, 13--18 Jul 2020.

\bibitem{kammounCovarianceDiscriminativePower2023}
Abla Kammoun and Romain Couillet.
\newblock Covariance discriminative power of kernel clustering methods.
\newblock {\em Electronic Journal of Statistics}, 17(1):291--390, January 2023.

\bibitem{krizhevsky2009learning}
Alex Krizhevsky, Geoffrey Hinton, et~al.
\newblock Learning multiple layers of features from tiny images.
\newblock 2009.

\bibitem{lecun1998gradient}
Yann LeCun, Leon Bottou, Yoshua Bengio, and Patrick Haffner.
\newblock {Gradient-based learning applied to document recognition}.
\newblock {\em Proceedings of the IEEE}, 86(11):2278--2324, 1998.

\bibitem{lecun1990optimal}
Yann LeCun, John Denker, and Sara Solla.
\newblock {Optimal Brain Damage}.
\newblock In {\em Advances in Neural Information Processing Systems}, volume~2
  of {\em NIPS‘90}, pages 598--605. Morgan-Kaufmann, 1990.

\bibitem{ledoux2005concentration}
Michel Ledoux.
\newblock {\em {The Concentration of Measure Phenomenon}}.
\newblock Mathematical Surveys and Monographs. 2005.

\bibitem{lee2018deep}
Jaehoon Lee, Jascha Sohl-dickstein, Jeffrey Pennington, Roman Novak, Sam
  Schoenholz, and Yasaman Bahri.
\newblock {Deep Neural Networks as Gaussian Processes}.
\newblock In {\em International Conference on Learning Representations}, 2018.

\bibitem{lee2018snip}
Namhoon Lee, Thalaiyasingam Ajanthan, and Philip Torr.
\newblock {SNIP: Single-shot network pruning based on connection sensitivity}.
\newblock In {\em International Conference on Learning Representations}, 2018.

\bibitem{liao2018spectrum}
Zhenyu Liao and Romain Couillet.
\newblock {On the Spectrum of Random Features Maps of High Dimensional Data}.
\newblock In {\em Proceedings of the 35th International Conference on Machine
  Learning}, volume~80 of {\em Proceedings of Machine Learning Research}, pages
  3063--3071, Stockholmsmässan, Stockholm Sweden, 2018. PMLR.

\bibitem{liao2018dynamics}
Zhenyu Liao and Romain Couillet.
\newblock {The Dynamics of Learning: A Random Matrix Approach}.
\newblock In {\em Proceedings of the 35th International Conference on Machine
  Learning}, volume~80 of {\em Proceedings of Machine Learning Research}, pages
  3072--3081, Stockholmsmässan, Stockholm Sweden, 2018. PMLR.

\bibitem{liu2020linearity}
Chaoyue Liu, Libin Zhu, and Misha Belkin.
\newblock {On the linearity of large non-linear models: when and why the
  tangent kernel is constant}.
\newblock 33:15954--15964.

\bibitem{liu2021kernel}
Fanghui Liu, Zhenyu Liao, and Johan Suykens.
\newblock Kernel regression in high dimensions: Refined analysis beyond double
  descent.
\newblock In {\em Proceedings of The 24th International Conference on
  Artificial Intelligence and Statistics}, volume 130 of {\em Proceedings of
  Machine Learning Research}, pages 649--657. PMLR, 13--15 Apr 2021.

\bibitem{liu20finding}
Tianlin Liu and Friedemann Zenke.
\newblock Finding trainable sparse networks through neural tangent transfer.
\newblock In {\em Proceedings of the 37th International Conference on Machine
  Learning}, volume 119 of {\em Proceedings of Machine Learning Research},
  pages 6336--6347. PMLR, 13--18 Jul 2020.

\bibitem{louart2019concentration}
Cosme Louart and Romain Couillet.
\newblock {Concentration of Measure and Large Random Matrices with an
  application to Sample Covariance Matrices}.
\newblock {\em arXiv}, 2018.

\bibitem{maas2013rectifier}
Andrew~L. Maas, Awni~Y. Hannun, and Andrew~Y. Ng.
\newblock {Rectifier Nonlinearities Improve Neural Network Acoustic Models}.
\newblock In {\em ICML Workshop on Deep Learning for Audio, Speech and Language
  Processing}, ICML Workshop, page~3, 2013.

\bibitem{marvcenko1967distribution}
Vladimir~A Marcenko and Leonid~Andreevich Pastur.
\newblock {Distribution of eigenvalues for some sets of random matrices}.
\newblock {\em Mathematics of the USSR-Sbornik}, 1(4):457, 1967.

\bibitem{martin2019traditional}
Charles~H. Martin and Michael~W. Mahoney.
\newblock {Traditional and Heavy Tailed Self Regularization in Neural Network
  Models}.
\newblock In {\em Proceedings of the 36th International Conference on Machine
  Learning}, volume~97 of {\em Proceedings of Machine Learning Research}, pages
  4284--4293. PMLR, 2019.

\bibitem{martin2020heavy}
Charles~H. Martin and Michael~W. Mahoney.
\newblock {Heavy-Tailed Universality Predicts Trends in Test Accuracies for
  Very Large Pre-Trained Deep Neural Networks}.
\newblock In {\em Proceedings of the 2020 SIAM International Conference on Data
  Mining (SDM)}, SDM'20, pages 505--513. SIAM, 2020.

\bibitem{mei2019generalization}
Song Mei and Andrea Montanari.
\newblock {The generalization error of random features regression: Precise
  asymptotics and double descent curve}.
\newblock {\em arXiv}, 2019.

\bibitem{pastur2020gauss}
Leonid Pastur.
\newblock {On Random Matrices Arising in Deep Neural Networks. Gaussian Case}.
\newblock {\em arXiv}, 2020.

\bibitem{pennington2017nonlinear}
Jeffrey Pennington and Pratik Worah.
\newblock {Nonlinear random matrix theory for deep learning}.
\newblock In {\em Advances in Neural Information Processing Systems}, volume~30
  of {\em NIPS'17}, pages 2637--2646. Curran Associates, Inc., 2017.

\bibitem{schmidhuber2015deep}
Jürgen Schmidhuber.
\newblock {Deep learning in neural networks: An overview}.
\newblock {\em Neural Networks}, 61:85--117, 2015.

\bibitem{seddik2020random}
Mohamed El~Amine Seddik, Cosme Louart, Mohamed Tamaazousti, and Romain
  Couillet.
\newblock {Random Matrix Theory Proves that Deep Learning Representations of
  GAN-data Behave as Gaussian Mixtures}.
\newblock In {\em Proceedings of the 37th International Conference on Machine
  Learning}, Proceedings of Machine Learning Research, pages 8573--8582. PMLR,
  2020.

\bibitem{seddik2019kernel}
Mohamed El~Amine Seddik, Mohamed Tamaazousti, and Romain Couillet.
\newblock {Kernel Random Matrices of Large Concentrated Data: the Example of
  GAN-Generated Images}.
\newblock In {\em 2019 IEEE International Conference on Acoustics, Speech and
  Signal Processing (ICASSP 2019)}, 2019 IEEE International Conference on
  Acoustics, Speech and Signal Processing (ICASSP 2019), pages 7480--7484.
  IEEE, 2019.

\bibitem{simonyan2014very}
Karen Simonyan and Andrew Zisserman.
\newblock {Very deep convolutional networks for large-scale image recognition}.
\newblock In {\em International Conference on Learning Representations},
  ICLR'14, 2014.

\bibitem{su2020sanity}
Jingtong Su, Yihang Chen, Tianle Cai, Tianhao Wu, Ruiqi Gao, Liwei Wang, and
  Jason~D Lee.
\newblock {Sanity-Checking Pruning Methods: Random Tickets can Win the
  Jackpot}.
\newblock In {\em Advances in Neural Information Processing Systems},
  volume~33, pages 20390--20401. Curran Associates, Inc.

\bibitem{tanaka2020pruning}
Hidenori Tanaka, Daniel Kunin, Daniel~L Yamins, and Surya Ganguli.
\newblock Pruning neural networks without any data by iteratively conserving
  synaptic flow.
\newblock In {\em Advances in Neural Information Processing Systems},
  volume~33, pages 6377--6389. Curran Associates, Inc., 2020.

\bibitem{tao2008random}
Terence Tao and Van Vu.
\newblock {Random matrices: the circular law}.
\newblock {\em Communications in Contemporary Mathematics}, 10(02):261--307,
  2008.

\bibitem{vershynin2018high}
Roman Vershynin.
\newblock {\em {High-Dimensional Probability: An Introduction with Applications
  in Data Science}}.
\newblock Cambridge Series in Statistical and Probabilistic Mathematics.
  Cambridge University Press, 2018.

\bibitem{Wang2020Picking}
Chaoqi Wang, Guodong Zhang, and Roger Grosse.
\newblock Picking winning tickets before training by preserving gradient flow.
\newblock In {\em International Conference on Learning Representations}, 2020.

\bibitem{wigner1955characteristic}
Eugene~P. Wigner.
\newblock {Characteristic Vectors of Bordered Matrices with Infinite
  Dimensions}.
\newblock {\em The Annals of Mathematics}, 62(3):548, 1955.

\bibitem{xiao2018dynamical}
Lechao Xiao, Yasaman Bahri, Jascha Sohl-Dickstein, Samuel Schoenholz, and
  Jeffrey Pennington.
\newblock {Dynamical Isometry and a Mean Field Theory of CNNs: How to Train
  10,000-Layer Vanilla Convolutional Neural Networks}.
\newblock In {\em Proceedings of the 35th International Conference on Machine
  Learning}, volume~80 of {\em Proceedings of Machine Learning Research}, pages
  5393--5402, Stockholmsmässan, Stockholm Sweden, 2018. PMLR.

\bibitem{yang2022on}
Hongru Yang and Zhangyang Wang.
\newblock {On the Neural Tangent Kernel Analysis of Randomly Pruned Wide Neural
  Networks}.
\newblock {\em arXiv}, 2022.

\bibitem{yu2015useful}
Y.~Yu, T.~Wang, and R.~J. Samworth.
\newblock {A useful variant of the Davis–Kahan theorem for statisticians}.
\newblock {\em Biometrika}, 102(2):315--323, 2015.

\end{thebibliography}
\bibliographystyle{plain}
}

\clearpage

\begin{center}
  {\Large \textbf{Supplementary Material\\}} \vskip 0.1in \textbf{``Lossless'' Compression of Deep Neural Networks: \\A High-dimensional Neural Tangent Kernel Approach}
\end{center}
\vskip 0.3in

\appendix
\section{Proofs of theorems and auxiliary results}
\label{sec:proofs_and_auxiliary}

Here, we provide the mathematical proofs of our technical results of Theorems~\ref{theo:CK}~and~\ref{theo:NTK} in Sections~\ref{subsec:proof_theo_CK}~and~\ref{subsec:proof_theo_NTK}, respectively.
In Section~\ref{subsec:equivalent_centering}, we provide further discussion on Theorems~\ref{theo:CK}~and~\ref{theo:NTK} in the case of single-hidden-layer NN and compare with the results in \cite{ali2022random}.
In Section~\ref{subsec:proof_coro_sparse_quantized} we provide the proof of Corollary~\ref{coro:sparse_quantized}.
In particular, as mentioned above in Footnote~\ref{foot:errata2}, we fix here the algebraic mistake in the proof of Theorem~\ref{theo:NTK} from the \href{https://papers.nips.cc/paper_files/paper/2022/hash/185087ea328b4f03ea8fd0c8aa96f747-Abstract-Conference.html}{online version} of this paper in the NeurIPS 2022 proceeding.

\subsection{Proof of Theorem~\ref{theo:CK}}
\label{subsec:proof_theo_CK}

%We provide detailed discussions on our working assumptions in Section~\ref{sec:preliminaries}

In this section, we provide the detailed proof of Theorem~\ref{theo:CK}.
Before going into details of the proof, we first recall our system model and working assumptions as follows.

We consider $n$ data vectors $\x_1, \ldots, \x_n \in \RR^p$ independently drawn from one of the $K$-class Gaussian mixtures $\mathcal C_1, \ldots, \mathcal C_K$ and denote $\X = [\x_1, \ldots, \x_n] \in \RR^{p \times n}$, with class $\mathcal C_a$ having cardinality $n_a$; that is, for $\x_i \in \mathcal C_a$ we have
\begin{equation}\label{eq:def_GMM_proof}
    \x_i \sim \mathcal N(\bmu_a/\sqrt p, \C_a/p),
\end{equation}
for mean vector $\bmu_a \in \RR^p$ and non-negative definite covariance $\C_a \in \RR^{p \times p}$ associated with class $\mathcal C_a$.

We position ourselves in the high-dimensional and non-trivial classification regime as stated in Assumption~\ref{ass:high-dimen}, that is: As $n \to \infty$, we have, for $a \in \{1,\ldots,K\}$ that,
\begin{enumerate}
    \item[(i)] $p/n \to c \in (0,\infty)$ and $n_a/n \to c_a \in (0, 1)$; and
    \item[(ii)] $\| \bmu_a \| = O(1)$; and
    \item[(iii)] for $\C^\circ \equiv \sum_{a=1}^K \frac{n_a}n \C_a$ and $\C_a^\circ \equiv \C_a - \C^\circ$, we have $\| \C_a\| = O(1)$, $\tr \C_a^\circ = O(\sqrt p)$ and $\tr (\C_a \C_b) = O(p)$ for $a,b \in \{1,\ldots,K\}$; and
    \item[(iv)] $\tau_0 \equiv \sqrt{\tr \C^\circ/p}$ converges in $(0,\infty)$.
\end{enumerate}
% \GD{We haven't mentioned the calculation method of $\tau_0$ here}

\begin{Remark}[Beyond Gaussian mixture data]\label{rem:beyond_gaussian_data}\normalfont
    Despite derived here for Gaussian mixture data, we conjecture that our results hold more generally beyond the Gaussian setting.
    As concrete examples, many results in random matrix theory and high dimensional statistics such as the popular Mar\u{c}enko-Pastur law \cite{marvcenko1967distribution}, the semicircular law \cite{wigner1955characteristic}, as well as the circular laws \cite{tao2008random}, have all been shown \emph{universal} in the sense that they do \emph{not} depend on the distribution of the (independent entries of the) data, as long as they are normalized to have zero mean and unit variance.
    In a machine learning (ML) context, such \emph{universal} behavior are observed to hold beyond the above models, and extends to \emph{nonlinear} models such as kernel matrices \cite{seddik2019kernel} and neural nets \cite{seddik2020random}, in that, say, for data drawn from the family of concentrated random vectors \cite{ledoux2005concentration,louart2019concentration} (so \emph{not} necessarily Gaussian), the performance on those ML models are the \emph{same}, in the larger $n,p$ setting, \emph{as if they were mere Gaussian mixtures} with the same means and covariances.
    We refer the interested readers to \cite[Chapter~8]{couillet_liao_2022} for more discussions on this point.
\end{Remark}

\subsubsection{Setup and notations}
We consider the fully-connected neural network model of depth $L$ and of successive widths $d_1, \ldots, d_L$ as defined in \eqref{eq:Sigma_ell}. 
We further denote $\W_\ell \in \RR^{d_\ell \times d_{\ell-1}}$ and $\sigma_\ell(\cdot)$ the weight matrix and activation at layer $\ell \in \{ 1, \ldots, L \}$, respectively.

% Before going into details of the proof, we first recall our system model and working assumptions as follows.

% We consider $n$ data vectors $\x_1, \ldots, \x_n \in \RR^p$ independently drawn from one of the $K$-class Gaussian mixtures $\mathcal C_1, \ldots, \mathcal C_K$ and denote $\X = [\x_1, \ldots, \x_n] \in \RR^{p \times n}$, with class $\mathcal C_a$ having cardinality $n_a$; that is, for $\x_i \in \mathcal C_a$ we have
% \begin{equation}\label{eq:def_GMM_proof}
%     \x_i \sim \mathcal N(\bmu_a/\sqrt p, \C_a/p),
% \end{equation}
% for mean vector $\bmu_a \in \RR^p$ and non-negative definite covariance $\C_a \in \RR^{p \times p}$ associated with class $\mathcal C_a$.

% We position ourselves in the high-dimensional and non-trivial classification regime as stated in \Cref{ass:high-dimen}, that is: As $n \to \infty$, we have, for $a \in \{1,\ldots,K\}$ that,
% \begin{enumerate}
%     \item[(i)] $p/n \to c \in (0,\infty)$ and $n_a/n \to c_a \in (0, 1)$; and
%     \item[(ii)] $\| \bmu_a \| = O(1)$; and
%     \item[(iii)] for $\C^\circ \equiv \sum_{a=1}^K \frac{n_a}n \C_a$ and $\C_a^\circ \equiv \C_a - \C^\circ$, we have $\| \C_a\| = O(1)$, $\tr \C_a^\circ = O(\sqrt p)$ and $\tr (\C_a \C_b) = O(p)$ for $a,b \in \{1,\ldots,K\}$; and
%     \item[(iv)] $\tau_0 \equiv \sqrt{\tr \C^\circ/p}$ converges in $(0,\infty)$.
% \end{enumerate}

We assume that the following conditions hold for the random weight matrices $\W_\ell$s and the activation function $\sigma_\ell$s for $\ell \in \{1,\ldots, L\}$, as demanded in Assumptions~\ref{ass:W}~and~\ref{ass:activation}:
\begin{itemize}%[leftmargin=\parindent,align=left,labelwidth=\parindent,labelsep=2pt]
    \item[(i)]
        The weight matrices $\W_\ell$s are independent and have i.i.d.\@ entries of zero mean, unit variance, and finite eight-order moment.
    \item[(ii)]
        The activations $\sigma_\ell$s are at least four-times weakly differentiable with respect to standard normal distribution, in the sense that $\max_{k \in \{0,1,2,3,4\}} \{ |\EE[\sigma_\ell^{(k)}(\xi)]| \} < C$ for some universal constant $C > 0$ and $\xi \sim \NN(0,1)$.
\end{itemize}

Our objective of interest in this section is the Conjugate Kernel (CK) matrix defined via the following recursive relation \cite{jacot2018neural,bietti2019inductive}
\begin{equation}\label{eq:K_CK_relation_proof}
    [\K_{\CK,\ell}]_{ij} = \EE_{u,v } [\sigma_\ell(u) \sigma_\ell(v)], \quad  \K_{\CK,0} = \X^\T \X,
\end{equation}
with
%$\B_\ell$ defined as below:
\begin{equation}%\label{eq:B_l}
    u,v \sim \NN \left(\zo, \begin{bmatrix} [\K_{\CK,\ell-1}]_{ii} & [\K_{\CK,\ell-1}]_{ij} \\ [\K_{\CK,\ell-1}]_{ij} & [\K_{\CK,\ell-1}]_{jj} \end{bmatrix}\right).
\end{equation}
The derivation and discussion of the closely related neural tangent kernel (NTK) matrix is given in Section~\ref{subsec:proof_theo_NTK}.

Let Assumptions~\ref{ass:high-dimen}--\ref{ass:activation} hold, and let $\tau_0, \tau_1, \ldots, \tau_L \geq 0$ be a sequence of non-negative numbers recursively defined via
\begin{equation}
    \tau_\ell = \sqrt{  \EE[ \sigma_\ell^2 (\tau_{\ell-1} \xi) ] },
\end{equation}
as in \eqref{eq:def_tau}. 
Further assume that the activation functions $\sigma_\ell(\cdot)$s are ``centered'' such that $\EE[\sigma_\ell(\tau_{\ell-1} \xi)] = 0$.
%This assumption, as we shall see, plays a central role in our proof.
Then, to prove Theorem~\ref{theo:CK} it suffices to show that,
\begin{itemize}%[leftmargin=\parindent,align=left,labelwidth=\parindent,labelsep=2pt]
    \item[(i)]
        the CK matrix $\K_{\CK,\ell}$ of layer $\ell \in \{0, 1, \ldots, L\}$ defined in \eqref{eq:K_CK_relation_proof} satisfies
        \begin{equation}
            \| \K_{\CK,\ell} - \tilde \K_{\CK,\ell} \| \to 0,
        \end{equation}
        almost surely as $n,p \to \infty$, with $\tilde \K_{\CK,\ell}$ taking the generic form of
        \begin{equation}
            \tilde \K_{\CK,\ell} \equiv \alpha_{\ell,1} \X^\T \X + \V \A_\ell \V^\T + \alpha_{\ell,0} \I_n, \quad \alpha_{\ell,0} = \tau_\ell^2 - \tau_0^2 \alpha_{\ell,1},
        \end{equation}
        for all $\ell \in \{1 \, \ldots, L \}$, $\J = [\mathbf{j}_1, \ldots, \mathbf{j}_K] \in \RR^{n \times K}$,  second-order data fluctuation vector $\bpsi = \{ \| \x_i - \EE[\x_i] \|^2 -  \EE[ \|\x_i - \EE[\x_i] \|^2 ] \}_{i=1}^n \in \RR^n$, second-order discriminative statistics $\bt = \{ \tr \C_a^\circ/\sqrt p \}_{a=1}^K \in \RR^K$, $\bT = \{ \tr \C_a \C_b/p \}_{a,b=1}^K \in \RR^{K \times K}$, and
        \begin{equation}%\label{eq:def_V_A}
            \V = [\J/\sqrt p,~\bpsi] \in \RR^{n \times (K+1)}, \quad \A_\ell = \begin{bmatrix} \alpha_{\ell,2} \bt \bt^\T + \alpha_{\ell,3} \bT & \alpha_{\ell,2}\bt \\ \alpha_{\ell,2}\bt^\T & \alpha_{\ell,2} \end{bmatrix} \in \RR^{(K+1)  \times (K+1)};
        \end{equation}
    \item[(ii)]
        the coefficients $\alpha_{\ell,1}, \alpha_{\ell,2}, \alpha_{\ell,3} $ are non-negative and satisfy the following recursive relations
        \begin{align*}%\label{eq:def_ds}
            \alpha_{\ell,1} & = \EE [ \sigma_\ell'( \tau_{\ell-1} \xi)]^2 \alpha_{\ell-1,1}, \quad \alpha_{\ell,2} = \EE [ \sigma_\ell'( \tau_{\ell-1} \xi)]^2 \alpha_{\ell-1,2} + \frac14 \EE [ \sigma_\ell''( \tau_{\ell-1} \xi)]^2 \alpha_{\ell-1,4}^2, %\label{eq:def_ds1}
            \\
            \alpha_{\ell,3} & = \EE [ \sigma_\ell'( \tau_{\ell-1} \xi)]^2 \alpha_{\ell-1,3} + \frac12 \EE [ \sigma_\ell''( \tau_{\ell-1} \xi)]^2 \alpha_{\ell-1,1}^2.                                                                                      %\label{eq:def_ds2}
        \end{align*}
        with $\alpha_{\ell,4} = \alpha_{\ell-1,4} \EE \left[ (\sigma_\ell'(\tau_{\ell - 1} \xi) )^2 + \sigma_\ell(\tau_{\ell - 1} \xi) \sigma_\ell''(\tau_{\ell - 1} \xi) \right]$ for $\xi \sim \NN(0,1)$.
\end{itemize}

We first introduce the following notations that will be consistently used in the proof: for $\x_i, \x_j \in \RR^p$ with $i \neq j$, let
\begin{equation}\label{eq:xi}
    \x_i = \bmu_i/\sqrt{p} + \z_i/\sqrt{p}, \quad \x_j = \bmu_j/\sqrt{p} + \z_j/\sqrt{p},
\end{equation}
so that $\z_i \sim \NN(\zo, \C_i)$, $\z_j \sim \NN(\zo, \C_j)$, and
\begin{align*}
    \x_i^\T \x_j &=  \underbrace{\frac1p \z_i^\T \z_j}_{O(p^{-1/2})} + \underbrace{\frac1p \bmu_i^\T \bmu_j + \frac1p (\bmu_i^\T \z_j + \bmu_j^\T \z_i)}_{O(p^{-1})}, \\
    t_i    & \equiv \frac1p \tr \C_i^\circ = O(p^{-1/2}), \quad \psi_i = \frac1p \| \z_i \|^2 - \frac1p \tr \C_i = O(p^{-1/2}),                                      \\
    \tau_0 & \equiv \sqrt{\frac1p \tr \C^\circ} = O(1),                                                                                                              \\
    \chi_i & \equiv  \underbrace{t_i + \psi_i}_{O(p^{-1/2})} + \underbrace{\| \bmu_i \|^2/p + 2 \bmu_i^\T \z_i/p}_{O(p^{-1})} = \| \x_i \|^2 - {\tau_0}^{2},
\end{align*}
where we note that the notations $\tau_0, \psi_i$ and $t_i$ (with a slight abuse of notation to denote $\C_i = \C_a$ for $\x_i \in \mathcal C_a$) are in line with those defined in Assumptions~\ref{ass:high-dimen}~and Theorem~\ref{theo:CK}, and we denote $S_{ij}$ \emph{any} term of the form
\begin{equation}\label{eq:def_Sij}
    S_{ij} \equiv S_{ij}(\gamma_1, \gamma_2) = \frac1p \z_i^\T \z_j \left(  \gamma_1 (t_i + \psi_i) + \gamma_2 (t_j + \psi_j) \right),
\end{equation}
for random \emph{or} deterministic scalars $ \gamma_1, \gamma_2 = O(1)$ (with high probability when being random). 

\begin{Remark}\label{rem:S}\normalfont
The introduction of $S_{ij}$ will, as we shall see below, greatly simplify our analysis.
Notably, we have $S_{ij} = O(p^{-1})$ and more importantly, it leads to, in matrix form, a matrix of spectral norm order $O(p^{-1/2})$, see~\cite{couillet2016kernel}.
This (somewhat counterintuitive) spectral norm result will be consistently exploited in the remainder of the proof.
\end{Remark}

To prove the above results, it suffices to establish the following entry-wise approximation of both the diagonal and non-diagonal entries of $\K_{\CK,\ell}$, the proof of which is given in Section~\ref{subsubsec:proof_of_lem_entry-wise-approx-CK-center} below.
\begin{Lemma}\label{lem:entry-wise-approx-CK-center}
Under the assumptions and notations of Theorem~\ref{theo:CK}, for centered activation such that $\EE[\sigma_\ell(\tau_{\ell-1} \xi)] = 0$ with $\xi \sim \NN(0,1)$ as in the statement of Theorem~\ref{theo:CK}, we have, for all $\ell \in \{0, \ldots, L \}$ and $\K_{\CK,\ell}$ defined in \eqref{eq:K_CK_relation_proof} that  uniformly over all $i \neq j \in \{1, \ldots, n\}$,
\begin{equation}\label{eq:K_CK_ij_center}
    [\K_{\CK,\ell}]_{ij} = \alpha_{\ell,1} \x_i^\T \x_j + \alpha_{\ell,2} (t_i + \psi_i) (t_j + \psi_j) + \alpha_{\ell,3} \left(\frac1p \z_i^\T \z_j \right)^2 + S_{ij} + O(p^{-3/2}),
\end{equation}
and
\begin{equation}\label{eq:K_CK_ii_center}
    [\K_{\CK,\ell}]_{ii} = \tau_{\ell}^2 + \alpha_{\ell-1,4} \chi_i + \alpha_{\ell,5} (t_i + \psi_i)^2 + O(p^{-3/2}),
\end{equation}
for $\tau_\ell = \sqrt{ \EE[\sigma_\ell^2(\tau_{\ell - 1} \xi)] }$ and non-negative scalars $\alpha_{\ell,1}, \alpha_{\ell,2}, \alpha_{\ell,3}, \alpha_{\ell,4}, \alpha_{\ell,5}$ satisfy the following recursion
\begin{align*}
  \alpha_{\ell,1} & = \EE [ \sigma_\ell'( \tau_{\ell-1} \xi)]^2 \alpha_{\ell-1,1}, \quad \alpha_{\ell,2} = \EE [ \sigma_\ell'( \tau_{\ell-1} \xi)]^2 \alpha_{\ell-1,2} + \frac14 \EE [ \sigma_\ell''( \tau_{\ell-1} \xi)]^2 \alpha_{\ell-1,4}^2, \\
    \alpha_{\ell,3} & = \EE [ \sigma_\ell'( \tau_{\ell-1} \xi)]^2 \alpha_{\ell-1,3} + \frac12 \EE [ \sigma_\ell''( \tau_{\ell-1} \xi)]^2 \alpha_{\ell-1,1}^2, \\
    \alpha_{\ell,4} &= \alpha_{\ell-1,4} \EE \left[ (\sigma_\ell'(\tau_{\ell - 1} \xi) )^2 + \sigma_\ell(\tau_{\ell - 1} \xi) \sigma_\ell''(\tau_{\ell - 1} \xi) \right],                 \\
    \alpha_{\ell,5} & = \alpha_{\ell-1,5} \EE \left[ (\sigma_\ell'(\tau_{\ell - 1} \xi) )^2 + \sigma_\ell(\tau_{\ell - 1} \xi) \sigma_\ell''(\tau_{\ell - 1} \xi) \right]                                                                                            \\
                    &\quad + \frac{ \alpha_{\ell-1,4}^2 }4 \EE\left[  \sigma_\ell(\tau_{\ell - 1} \xi) \sigma_\ell''''(\tau_{\ell - 1} \xi) + 4 \sigma_\ell'(\tau_{\ell - 1} \xi) \sigma_\ell'''(\tau_{\ell - 1} \xi) + 3 (\sigma_\ell''(\tau_{\ell - 1} \xi))^2 \right],
\end{align*}
with initializations
\begin{equation}
  \alpha_{0,1} = 1, \quad \alpha_{0,2}= 0, \quad \alpha_{0,3} = 0, \quad \alpha_{0,4} = 1, \quad \alpha_{0,5} = 0.
\end{equation}
\end{Lemma}

By Lemma~\ref{lem:entry-wise-approx-CK-center}, we have in particular
\begin{equation}
    [\K_{\CK,\ell}]_{ii} = \tau_\ell^2 + O(p^{-1/2}),
\end{equation}
so that in matrix form (by using the fact that $\| \A \| \leq n \| \A \|_{\max} $ for $\A \in \RR^{n \times n}$ with $\| \A \|_{\max} = \max_{ij} |\A|_{ij}$ and $\{  S_{ij} \}_{i,j} = O_{\| \cdot \|}(p^{-\frac12})$, see again Remark~\ref{rem:S}, as well as $\{ (\z_i^\T \z_j/p )^2 \}_{i \neq j=1}^n = \frac1p \J \mathbf{T} \J^\T + O_{\| \cdot \|}(p^{-1/2})$, see \cite{couillet2016kernel,kammounCovarianceDiscriminativePower2023}):
\begin{equation}
    \K_{\CK,\ell} = \alpha_{\ell,1} \X^\T \X + \V \A_\ell \V^\T + (  \tau_\ell^2 - \tau_0^2 \alpha_{\ell,1} ) \I_n + O_{\| \cdot \|}(p^{-1/2}),
\end{equation}
where $O_{\| \cdot \|}(p^{-1/2})$ denotes matrices of spectral norm order $O(p^{-1/2})$, with
\begin{equation}
    \V = [\J/\sqrt p,~\bpsi] \in \RR^{n \times (K+1)}, \quad \A_\ell = \begin{bmatrix} \alpha_{\ell,2} \bt \bt^\T + \alpha_{\ell,3} \bT & \alpha_{\ell,2}\bt \\ \alpha_{\ell,2}\bt^\T & \alpha_{\ell,2} \end{bmatrix},
\end{equation}
and
\begin{equation}
    \bT = \left\{ \frac1p \tr \C_a \C_b \right\}_{a,b=1}^K, \quad \bt = \left\{ \frac1{\sqrt p} \tr \C_a^\circ \right\}_{a=1}^K,
\end{equation}
as in the statement of Theorem~\ref{theo:CK}.
This thus concludes the proof of Theorem~\ref{theo:CK}.

\subsubsection{Proof of Lemma~\ref{lem:entry-wise-approx-CK-center}}
\label{subsubsec:proof_of_lem_entry-wise-approx-CK-center}

In the following, we will prove Lemma~\ref{lem:entry-wise-approx-CK-center} by induction on $\ell \in \{ 0, 1, \ldots, L\}$.

We start our mathematical induction by considering $\ell = 0$.
In this case, we have by definition $\K_{\CK,0} = \X^\T \X$ so that
\begin{equation}
    \K_{\CK,0} = \tilde \K_{\CK,0} = \X^\T \X,
\end{equation}
with $\alpha_{0,1} = 1$, $\alpha_{0,2} = 0$, $\alpha_{0,3} = 0$ for non-diagonal entries and $\alpha_{0,4} = 1$, $\alpha_{0,5} = 0$ for diagonal entries.

We then assume the entry-wise approximation in Lemma~\ref{lem:entry-wise-approx-CK-center} holds at layer $\ell-1$ with\footnote{ Recall that the introduction of the term $S_{ij}$ does \emph{not} alter the form of $\K_{\CK}$ or $\tilde \K_{\CK}$ in a spectral norm sense, see again Remark~\ref{rem:S} above. }
\begin{equation}\label{eq:K_CK_ij}
    [\K_{\CK,\ell-1}]_{ij} = \alpha_{\ell-1,1} \x_i^\T \x_j + \alpha_{\ell-1,2} (t_i + \psi_i) (t_j + \psi_j) + \alpha_{\ell-1,3} \left(\frac1p \z_i^\T \z_j \right)^2 + S_{ij} + O(p^{-3/2}),
\end{equation}
for $i \neq j$, and
\begin{equation}\label{eq:K_CK_ii}
    [\K_{\CK,\ell-1}]_{ii} = \tau_{\ell-1}^2 + \alpha_{\ell-1,4} \chi_i + \alpha_{\ell-1,5} (t_i + \psi_i)^2 + O(p^{-3/2}).
\end{equation}
and work on the CK matrix $\K_{\CK,\ell}$ at layer $\ell$.
% for $\A_{\ell-1} = \begin{bmatrix} \alpha_{\ell-1,2} \bt \bt^\T + \alpha_{\ell-1,3} \bT & \alpha_{\ell-1,2}\bt \\ \alpha_{\ell-1,2}\bt^\T & \alpha_{\ell-1,2} \end{bmatrix}$, 

By definition in \eqref{eq:K_CK_relation_proof}, using the Gram-Schmidt orthogonalization procedure for standard Gaussian random variable as in \cite{fan2020spectra,ali2022random}, we write
\begin{equation}
    u = \sqrt{ [\K_{\CK,\ell-1}]_{ii} } \cdot \xi_i, \quad v = \frac{[\K_{\CK,\ell-1}]_{ij}}{ \sqrt{ [\K_{\CK,\ell-1}]_{ii} } } \cdot \xi_i + \sqrt{ [\K_{\CK,\ell-1}]_{jj} - \frac{[\K_{\CK,\ell-1}]_{ij}^2 }{  [\K_{\CK,\ell-1}]_{ii}  } } \cdot \xi_j,
\end{equation}
for \emph{independent} $\xi_i, \xi_j \sim \NN(0,1)$. 
As such, we have, for layer $\ell$ that
\begin{align}
    [\K_{\CK,\ell}]_{ii} & = \EE \left[ \sigma_\ell^2 \left( \sqrt{ [\K_{\CK,\ell-1}]_{ii} } \cdot \xi_i  \right) \right] \label{eq:def_K_CK_ii_ell}                                                                                                                                              \\
    [\K_{\CK,\ell}]_{ij} & = \EE\left[ \sigma_\ell \left( \sqrt{ [\K_{\CK,\ell-1}]_{ii} } \cdot \xi_i \right) \right. \nonumber                                                                                                                                                                   \\
                         & \left. \times \sigma_\ell \left( \frac{[\K_{\CK,\ell-1}]_{ij}}{ \sqrt{ [\K_{\CK,\ell-1}]_{ii} } } \cdot \xi_i + \sqrt{ [\K_{\CK,\ell-1}]_{jj} - \frac{[\K_{\CK,\ell-1}]_{ij}^2 }{  [\K_{\CK,\ell-1}]_{ii}  } } \cdot \xi_j \right) \right], \label{eq:def_K_CK_ij_ell}
\end{align}%\label{eq:Sigma_ell_expand}
where the expectations are now taken with respect to the \emph{independent} random variables $\xi_i$ and $\xi_j$ (and in fact, conditioned on the random vectors $\x_i$ and $\x_j$).

% Based on the induction hypothesis on the layer $\ell-1$ in \eqref{eq:induction_hypothesis_ell-1}, we have

The objective is then to derive the approximation of $[\K_{\CK,\ell}]_{ij}$ and $[\K_{\CK,\ell}]_{ii}$ at layer $\ell$, both to terms of order $O(p^{-3/2})$, and to subsequently derive the recursive relations between the key coefficients of layer $\ell -1$:
\begin{equation}
    \{ \alpha_{\ell-1,1}, \alpha_{\ell-1,2}, \alpha_{\ell-1,3}, \alpha_{\ell-1,4}, \alpha_{\ell-1,5} \},
\end{equation}
and those of layer $\ell$
\begin{equation}
    \{ \alpha_{\ell,1}, \alpha_{\ell,2}, \alpha_{\ell,3}, \alpha_{\ell,4}, \alpha_{\ell,5} \},
\end{equation}
as given in the statement of Lemma~\ref{lem:entry-wise-approx-CK-center}~and~Theorem~\ref{theo:CK}.

% To this end, we first focus on the diagonal entries and evaluate $[\K_{\CK,\ell}]_{ii}$, then on the off-diagonal terms $[\K_{\CK,\ell}]_{ij}$ for $i\neq j$.
% We finally conclude the proof by putting the approximation in matrix form.

\paragraph{On the diagonal.}%\label{proof:on the diagnal}
We start with the diagonal entries of $\K_{\CK,\ell}$, which, as per its expression in \eqref{eq:def_K_CK_ii_ell}, depends on the diagonal entries $[\K_{\CK,\ell-1}]_{ii}$ at layer $\ell-1$ as defined in \eqref{eq:K_CK_ii}. 
By Taylor-expanding the nonlinear function $f(t) = \sqrt{t}$ around the leading order term $t \simeq \tau_{\ell-1}^2 = O(1)$, one gets
\begin{align*}
    \sqrt{ [\K_{\CK,\ell-1}]_{ii}} & = \sqrt{ \tau_{\ell-1}^2 + \alpha_{\ell-1,4} \chi_i + \alpha_{\ell-1,5} (t_i + \psi_i)^2 + O(p^{-3/2}) }                                                                                                 \\
                                   & = \tau_{\ell-1} + \frac1{2 \tau_{\ell-1} } \left( \alpha_{\ell-1,4} \chi_i + \alpha_{\ell-1,5} (t_i + \psi_i)^2 \right) - \frac{ \alpha_{\ell-1,4}^2 }{8 \tau_{\ell-1}^3} (t_i + \psi_i)^2 + O(p^{-3/2}) \\
                                   & = \tau_{\ell-1} + \frac1{2 \tau_{\ell-1} } \alpha_{\ell-1,4} \chi_i + \frac{ 4 \tau_{\ell-1}^2 \alpha_{\ell-1,5} - \alpha_{\ell-1,4}^2 }{8 \tau_{\ell-1}^3} (t_i + \psi_i)^2 + O(p^{-3/2}).
\end{align*}
Further Taylor-expand $\sigma_\ell^2(\sqrt{ [\K_{\CK,\ell-1}]_{ii}} \cdot \xi_i) = f(\sqrt{ [\K_{\CK,\ell-1}]_{ii}} \cdot \xi_i)$ around $\tau_{\ell-1} \xi_i$, we get
\begin{align*}
    [\K_{\CK,\ell}]_{ii} & = \EE \left[ \sigma_\ell^2 \left( \sqrt{ [\K_{\CK,\ell-1}]_{ii} } \cdot \xi \right) \right] = \EE \left[ f \left( \sqrt{ [\K_{\CK,\ell-1}]_{ii} } \cdot \xi \right) \right]                                                                        \\
     & = \EE\left[ f(\tau_{\ell - 1} \xi) + f'(\tau_{\ell - 1} \xi) \xi \left( \frac1{2 \tau_{\ell-1} } \alpha_{\ell-1,4} \chi_i + \frac{ 4 \tau_{\ell-1}^2 \alpha_{\ell-1,5} - \alpha_{\ell-1,4}^2 }{8 \tau_{\ell-1}^3} (t_i + \psi_i)^2 \right) \right] \\
     & \quad + \EE \left[ \frac12 f''(\tau_{\ell - 1} \xi) \xi^2 \right] \frac{\alpha_{\ell-1,4}^2 }{4 \tau_{\ell-1}^2 } (t_i + \psi_i)^2 + O(p^{-3/2})                                                                                                   \\
     & = \EE[f(\tau_{\ell - 1} \xi)] + \EE[f''(\tau_{\ell - 1} \xi)] \left( \frac12 \alpha_{\ell-1,4} \chi_i + \frac{ 4 \tau_{\ell-1}^2 \alpha_{\ell-1,5} - \alpha_{\ell-1,4}^2 }{8 \tau_{\ell-1}^2 } (t_i + \psi_i)^2 \right)                            \\
     & \quad + \EE \left[ \frac12 f''(\tau_{\ell - 1} \xi) \xi^2 \right] \frac{\alpha_{\ell-1,4}^2 }{4 \tau_{\ell-1}^2 } (t_i + \psi_i)^2 + O(p^{-3/2})                                                                                                   \\
     & = \EE[f(\tau_{\ell - 1} \xi)] +  \frac{\alpha_{\ell-1,4}}2 \EE[f''(\tau_{\ell - 1} \xi)] \chi_i \\
     &  \quad + \frac{ 4 \alpha_{\ell-1,5}\EE[f''(\tau_{\ell - 1} \xi)] +  \alpha_{\ell-1,4}^2 \EE [ f''''(\tau_{\ell - 1} \xi) ]  }{8 } (t_i + \psi_i)^2 + O(p^{-3/2}),
\end{align*}
where we denote the shortcut $f(x) = \sigma_\ell^2(x)$ and use the facts that
\begin{equation}
    \EE [ f'(\tau_{\ell - 1} \xi) \xi] = \tau_{\ell-1} \EE[f''(\tau_{\ell - 1} \xi)], \quad
    \EE [ f''(\tau_{\ell - 1} \xi) (\xi^2 - 1) ] = \tau_{\ell-1}^2 \EE[f''''(\tau_{\ell - 1} \xi)],
\end{equation}
for $\xi \sim \NN(0,1)$, as a consequence of the Gaussian integration by parts formula.

As a consequence, we obtain the following relation
\begin{align*}
    \tau_\ell       & = \sqrt{ \EE[\sigma_\ell^2(\tau_{\ell - 1} \xi)] }, \quad \alpha_{\ell,4} = \alpha_{\ell-1,4} \EE \left[ (\sigma_\ell'(\tau_{\ell - 1} \xi) )^2 + \sigma_\ell(\tau_{\ell - 1} \xi) \sigma_\ell''(\tau_{\ell - 1} \xi) \right],                 \\
    \alpha_{\ell,5} & = \alpha_{\ell-1,5} \EE \left[ (\sigma_\ell'(\tau_{\ell - 1} \xi) )^2 + \sigma_\ell(\tau_{\ell - 1} \xi) \sigma_\ell''(\tau_{\ell - 1} \xi) \right]                                                                                            \\
                    &\quad + \frac{ \alpha_{\ell-1,4}^2 }4 \EE\left[  \sigma_\ell(\tau_{\ell - 1} \xi) \sigma_\ell''''(\tau_{\ell - 1} \xi) + 4 \sigma_\ell'(\tau_{\ell - 1} \xi) \sigma_\ell'''(\tau_{\ell - 1} \xi) + 3 (\sigma_\ell''(\tau_{\ell - 1} \xi))^2 \right].
\end{align*}
This concludes the proof for the diagonal entries of $\K_{\CK,\ell}$ in Lemma~\ref{lem:entry-wise-approx-CK-center}.

\paragraph{Off the diagonal.}
We now move on to the more involved non-diagonal entries of $\K_{\CK,\ell}$. First note, for $i \neq j$, that
\begin{align*}
    \frac{[\K_{\CK,\ell-1}]_{ij}}{ \sqrt{ [\K_{\CK,\ell-1}]_{ii} } } & = \frac{ [\K_{\CK,\ell-1}]_{ij} }{ \tau_{\ell-1} + \frac1{2 \tau_{\ell-1} } \alpha_{\ell-1,4} \chi_i + \frac{ 4 \tau_{\ell-1}^2 \alpha_{\ell-1,5} - \alpha_{\ell-1,4}^2 }{8 \tau_{\ell-1}^3} (t_i + \psi_i)^2 + O(p^{-3/2}) } \\
     & = [\K_{\CK,\ell-1}]_{ij} \left( \frac1{  \tau_{\ell-1}} - \frac1{ \tau_{\ell-1}^2 } \left(  \frac{\alpha_{\ell-1,4}}{2 \tau_{\ell-1} } (t_i + \psi_i) + O(p^{-1}) \right)  \right) + O(p^{-3/2})                              \\
     & = \frac1{\tau_{\ell-1}} [\K_{\CK,\ell-1}]_{ij} - \frac{\alpha_{\ell-1,4} \alpha_{\ell-1,1} }{2 \tau_{\ell-1}^3 } (t_i + \psi_i) \frac1p \z_i^\T \z_j + O(p^{-3/2})                                                            \\
     & =\frac1{\tau_{\ell-1}} [\K_{\CK,\ell-1}]_{ij} + S_{ij} +  O(p^{-3/2}) = O(p^{-1/2}),
\end{align*}
with
\begin{align*}
    [\K_{\CK,\ell-1}]_{ij} & = \alpha_{\ell-1,1} \x_i^\T \x_j + \alpha_{\ell-1,2} (t_i + \psi_i) (t_j + \psi_j) + \alpha_{\ell-1,3} \left(\frac1p \z_i^\T \z_j \right)^2 + S_{ij} + O(p^{-3/2}) = O(p^{-1/2}),
\end{align*}
as in \eqref{eq:K_CK_ij}, where we recall $S_{ij} = \frac1p \z_i^\T \z_j \left(  \gamma_1 (t_i + \psi_i) + \gamma_2 (t_j + \psi_j) \right) = O(p^{-1})$ and is of vanishing spectral norm, see again Remark~\ref{rem:S}.

Then, it follows that
\begin{align*}
     & \sqrt{ [\K_{\CK,\ell-1}]_{jj} - \frac{[\K_{\CK,\ell-1}]_{ij}^2 }{  [\K_{\CK,\ell-1}]_{ii}  } } \\
     & = \sqrt{ \tau_{\ell-1}^2 + \alpha_{\ell-1,4} \chi_j + \alpha_{\ell-1,5} (t_j + \psi_j)^2 - \frac{ \alpha_{\ell-1,1}^2 }{ \tau_{\ell-1}^2 } \left (\frac1p \z_i^\T \z_j \right)^2 + O(p^{-3/2}) }     \\
     & = \tau_{\ell-1} + \frac1{2 \tau_{\ell-1} } \left( \alpha_{\ell-1,4} \chi_j + \alpha_{\ell-1,5} (t_j + \psi_j)^2 - \frac{ \alpha_{\ell-1,1}^2 }{ \tau_{\ell-1}^2 } \left (\frac1p \z_i^\T \z_j \right)^2 \right) - \frac{ \alpha_{\ell-1,4}^2 (t_j + \psi_j)^2 }{8 \tau_{\ell-1}^3 } + O(p^{-3/2})   \\
     & = \tau_{\ell-1} + \frac1{2 \tau_{\ell-1} } \left( \alpha_{\ell-1,4} \chi_j - \frac{ \alpha_{\ell-1,1}^2 }{ \tau_{\ell-1}^2 } \left (\frac1p \z_i^\T \z_j \right)^2 \right) + \frac{ 4 \tau_{\ell-1}^2 \alpha_{\ell-1,5} - \alpha_{\ell-1,4}^2 }{8 \tau_{\ell-1}^3 } (t_j + \psi_j)^2 + O(p^{-3/2}).
    %&\equiv \sqrt{\alpha_1} + \frac1{2 \sqrt{\alpha_1} } \Delta_{ij} + O(p^{-3/2})
\end{align*}

As a consequence, we get, again by Taylor expansion that
\begin{align*}
     & \sigma_\ell \left( \sqrt{ [\K_{\CK,\ell-1}]_{ii}} \cdot \xi_i \right) \sigma_\ell \left( \frac{[\K_{\CK,\ell-1}]_{ij}}{ \sqrt{ [\K_{\CK,\ell-1}]_{ii} } } \cdot \xi_i + \sqrt{ [\K_{\CK,\ell-1}]_{jj} - \frac{[\K_{\CK,\ell-1}]_{ij}^2 }{  [\K_{\CK,\ell-1}]_{ii}  } } \cdot \xi_j \right)                                 \\
     & = \sigma_\ell \left( \tau_{\ell-1} \xi_i + \frac1{2 \tau_{\ell-1} } \alpha_{\ell-1,4} \chi_i \xi_i + \frac{ 4 \tau_{\ell-1}^2 \alpha_{\ell-1,5} - \alpha_{\ell-1,4}^2 }{8 \tau_{\ell-1}^3} (t_i + \psi_i)^2 \xi_i + O(p^{-3/2}) \right)                                                                                    \\
     & \quad \times \sigma_\ell \left( \frac1{\tau_{\ell-1}} [\K_{\CK,\ell-1}]_{ij} \xi_i + \tau_{\ell-1} \xi_j + \frac1{2 \tau_{\ell-1} } \left( \alpha_{\ell-1,4} \chi_j - \frac{ \alpha_{\ell-1,1}^2 }{ \tau_{\ell-1}^2 } \left (\frac1p \z_i^\T \z_j \right)^2 \right) \xi_j \right.                                                         \\
     & \quad \left.+ \frac{ 4 \tau_{\ell-1}^2 \alpha_{\ell-1,5} - \alpha_{\ell-1,4}^2 }{8 \tau_{\ell-1}^3 } (t_j + \psi_j)^2 \xi_j + O(p^{-3/2})\right)    \\        
% \end{align*}
% \begin{align*}
    & = \left( \sigma_\ell( \tau_{\ell-1} \xi_i)  +  \sigma_\ell'( \tau_{\ell-1} \xi_i) \xi_i \left(  \frac1{2 \tau_{\ell-1} } \alpha_{\ell-1,4} \chi_i + \frac{ 4 \tau_{\ell-1}^2 \alpha_{\ell-1,5} - \alpha_{\ell-1,4}^2 }{8 \tau_{\ell-1}^3} (t_i + \psi_i)^2 \right) \right.                                                 \\
     & \quad \quad \left. + \sigma_\ell''( \tau_{\ell-1} \xi_i) \xi_i^2 \frac{\alpha_{\ell-1,4}^2 }{ 8\tau_{\ell-1}^2 } (t_i + \psi_i)^2 \right)                                                                                                                                                                                        \\
     & \quad \times \left( \sigma_\ell( \tau_{\ell-1} \xi_j)  +  \sigma_\ell'( \tau_{\ell-1} \xi_j) \left( \frac{ [\K_{\CK,\ell-1}]_{ij} }{\tau_{\ell-1}} \xi_i + \frac1{2 \tau_{\ell-1} } \left( \alpha_{\ell-1,4} \chi_j - \frac{ \alpha_{\ell-1,1}^2 }{ \tau_{\ell-1}^2 } \left (\frac1p \z_i^\T \z_j \right)^2 \right) \xi_j  \right. \right. \\
     & \quad \quad \left. +  \frac{ 4 \tau_{\ell-1}^2 \alpha_{\ell-1,5} - \alpha_{\ell-1,4}^2 }{8 \tau_{\ell-1}^3 } (t_j + \psi_j)^2 \xi_j \right)                                                                                                                                                                 \\
     & \quad \quad \left.+ \frac12 \sigma_\ell''( \tau_{\ell-1} \xi_j)  \left( \frac{ \alpha_{\ell-1,1} \frac1p \z_i^\T \z_j }{\tau_{\ell-1}} \xi_i  + \frac{ \alpha_{\ell-1,4} (t_j + \psi_j) }{2 \tau_{\ell-1}} \xi_j \right)^2  \right) +  O(p^{-3/2})                                                                                                                                                                                                                                                                                                        \\
     & \equiv \left( \sigma_\ell( \tau_{\ell-1} \xi_i) + T_{1,i} + T_{2,i} \right) \left( \sigma_\ell( \tau_{\ell-1} \xi_j) + T_{3,ij} + T_{3,j} + T_{4,ij} + T_{4,j} + S_{ij} \right) + O(p^{-3/2}),
\end{align*}

where we denote the shortcuts:
\begin{align*}
    T_{1,i} & = \sigma_\ell'( \tau_{\ell-1} \xi_i) \xi_i \cdot \frac{\alpha_{\ell-1,4}}{2 \tau_{\ell-1} }  \chi_i = O(p^{-1/2}),                                                                                                                                                                               \\
    T_{2,i} & = \left( \frac{ \alpha_{\ell-1,5} \sigma_\ell'( \tau_{\ell-1} \xi_i) \xi_i  }{ 2 \tau_{\ell-1} } + \alpha_{\ell-1,4}^2 \frac{ \sigma_\ell''( \tau_{\ell-1} \xi_i) \xi_i^2 \tau_{\ell-1} -  \sigma_\ell'( \tau_{\ell-1} \xi_i) \xi_i }{ 8 \tau_{\ell-1}^3 } \right) (t_i + \psi_i)^2 = O(p^{-1}),
\end{align*}
that \emph{only} depend on $\xi_i$; and
\begin{align*}
    T_{3,ij} & = \sigma_\ell'( \tau_{\ell-1} \xi_j) \xi_i \cdot \frac{ [\K_{\CK,\ell-1}]_{ij} }{\tau_{\ell-1}} = O(p^{-1/2}),                                   \\
    T_{4,ij} & = \frac12 \sigma_\ell''( \tau_{\ell-1} \xi_j) \xi_i^2 \cdot \frac{ \alpha_{\ell-1,1}^2 }{\tau_{\ell-1}^2 } \left (\frac1p \z_i^\T \z_j \right)^2 = O(p^{-1}),
\end{align*}
that depend on both $\xi_i$ and $\xi_j$; as well as
\begin{align*}
    T_{3,j} & = \sigma_\ell'( \tau_{\ell-1} \xi_j) \xi_j \cdot \left( \frac{\alpha_{\ell-1,4}}{2 \tau_{\ell-1} }  \chi_j - \frac{ \alpha_{\ell-1,1}^2 }{ 2\tau_{\ell-1}^3 } \left (\frac1p \z_i^\T \z_j \right)^2 + \frac{ 4 \tau_{\ell-1}^2 \alpha_{\ell-1,5} - \alpha_{\ell-1,4}^2 }{8 \tau_{\ell-1}^3 } (t_j + \psi_j)^2 \right)  \\
    & = O(p^{-1/2}), \\ %+ O(p^{-1}) \\ 
    T_{4,j} & = \frac12 \sigma_\ell''( \tau_{\ell-1} \xi_j) \xi_j^2 \cdot \frac{ \alpha_{\ell-1,4}^2 }{4 \tau_{\ell-1}^2 } (t_j + \psi_j)^2 = O(p^{-1}),
\end{align*}
that \emph{only} depend on $\xi_j$, where we particularly note that the cross terms are of the form $S_{ij}$.

As such, we have
\begin{align*}
     & \sigma_\ell \left( \sqrt{ [\K_{\CK,\ell-1}]_{ii}} \cdot \xi_i \right) \sigma_\ell \left( \frac{[\K_{\CK,\ell-1}]_{ij}}{ \sqrt{ [\K_{\CK,\ell-1}]_{ii} } } \cdot \xi_i + \sqrt{ [\K_{\CK,\ell-1}]_{jj} - \frac{[\K_{\CK,\ell-1}]_{ij}^2 }{  [\K_{\CK,\ell-1}]_{ii}  } } \cdot \xi_j \right) \\
     & = \sigma_\ell( \tau_{\ell-1} \xi_i) \sigma_\ell( \tau_{\ell-1} \xi_j) + \sigma_\ell( \tau_{\ell-1} \xi_i) (T_{3,ij} + T_{4,ij}) + \sigma_\ell( \tau_{\ell-1} \xi_i) (T_{3,j} + T_{4,j})                                                                                                    \\
     & \quad + \sigma_\ell( \tau_{\ell-1} \xi_j) (T_{1,i} + T_{2,i}) + T_{1,i} (T_{3,ij} + T_{3,j} ) + S_{ij} + O(p^{-3/2}),
\end{align*}
with in particular
\begin{align*}
    T_{1,i} (T_{3,ij} + T_{3,j} ) & = \sigma_\ell'( \tau_{\ell-1} \xi_i) \xi_i \cdot \frac{\alpha_{\ell-1,4}}{2 \tau_{\ell-1} } (t_i + \psi_i) \cdot \sigma_\ell'( \tau_{\ell-1} \xi_j) \xi_i \frac{ \alpha_{\ell-1,1} \frac1p \z_i^\T \z_j }{\tau_{\ell-1}}                \\
                                  & \quad + \sigma_\ell'( \tau_{\ell-1} \xi_i) \xi_i \cdot \frac{\alpha_{\ell-1,4}}{2 \tau_{\ell-1} } (t_i + \psi_i) \cdot \sigma_\ell'( \tau_{\ell-1} \xi_j) \xi_j \frac{\alpha_{\ell-1,4}}{2 \tau_{\ell-1} } (t_j + \psi_j) + O(p^{-3/2}) \\
                                  & =\sigma_\ell'( \tau_{\ell-1} \xi_i) \xi_i \sigma_\ell'( \tau_{\ell-1} \xi_j) \xi_j \frac{\alpha_{\ell-1,4}^2 }{4 \tau_{\ell-1}^2 } (t_i + \psi_i) (t_j + \psi_j) + S_{ij} + O(p^{-3/2}).
\end{align*}

We thus conclude, for $\EE[\sigma_\ell(\tau_{\ell-1} \xi)] = 0$ with $\xi \sim \NN(0,1)$, that
\begin{align*}
&[\K_{\CK,\ell}]_{ij} \\
& = \EE \left[ \sigma_\ell \left( \sqrt{ [\K_{\CK,\ell-1}]_{ii}} \cdot \xi_i \right) \sigma_\ell \left( \frac{[\K_{\CK,\ell-1}]_{ij}}{ \sqrt{ [\K_{\CK,\ell-1}]_{ii} } } \cdot \xi_i + \sqrt{ [\K_{\CK,\ell-1}]_{jj} - \frac{[\K_{\CK,\ell-1}]_{ij}^2 }{  [\K_{\CK,\ell-1}]_{ii}  } } \cdot \xi_j \right) \right]                             \\
& = \EE[ \sigma_\ell( \tau_{\ell-1} \xi_i) (T_{3,ij} + T_{4,ij}) ] + \EE \left[ \sigma_\ell'( \tau_{\ell-1} \xi_i) \xi_i \sigma_\ell'( \tau_{\ell-1} \xi_j) \xi_j \right] \frac{\alpha_{\ell-1,4}^2 }{4 \tau_{\ell-1}^2 } (t_i + \psi_i) (t_j + \psi_j)                                                                                       \\
& \quad + S_{ij} + O(p^{-3/2})                                                                                                                                                                                                                                                                                                                \\
& =\EE \left[ \sigma_\ell( \tau_{\ell-1} \xi_i) \sigma_\ell'( \tau_{\ell-1} \xi_j) \xi_i \frac{ [\K_{\CK,\ell-1}]_{ij} }{\tau_{\ell-1}} \right] + \frac12 \EE \left[ \sigma_\ell( \tau_{\ell-1} \xi_i) \sigma_\ell''( \tau_{\ell-1} \xi_j) \xi_i^2 \frac{ \alpha_{\ell-1,1}^2 }{\tau_{\ell-1}^2 } \left(\frac1p \z_i^\T \z_j\right)^2 \right] \\
& \quad +\EE [ \sigma_\ell'( \tau_{\ell-1} \xi_i) \xi_i] \EE[\sigma_\ell'( \tau_{\ell-1} \xi_j) \xi_j] \frac{\alpha_{\ell-1,4}^2 }{4 \tau_{\ell-1}^2 } (t_i + \psi_i) (t_j + \psi_j)  + S_{ij} + O(p^{-3/2})                                                                                                                                  \\
& =\EE [ \sigma_\ell'( \tau_{\ell-1} \xi)]^2 [\K_{\CK,\ell-1}]_{ij} + \frac{\alpha_{\ell-1,1}^2}2 \EE [ \sigma_\ell''( \tau_{\ell-1} \xi)]^2 \left(\frac1p \z_i^\T \z_j\right)^2                                                                                                                                                              \\
& \quad + \frac{\alpha_{\ell-1,4}^2 }4 \EE [ \sigma_\ell''( \tau_{\ell-1} \xi)]^2 (t_i + \psi_i) (t_j + \psi_j)  + S_{ij} + O(p^{-3/2}),
\end{align*}
where we used again the fact that
\begin{equation}
    \EE[\xi f(\tau \xi)] = \tau \EE[f'(\tau \xi)], \quad \EE[\xi^2 f(\tau \xi)] = \EE[(\xi^2-1) f(\tau \xi)] = \tau^2 \EE[f''(\tau \xi)],
\end{equation}
for $\EE[f(\tau \xi)] = 0$.

This allows us to conclude that
\begin{align}
    \alpha_{\ell,1} & = \EE [ \sigma_\ell'( \tau_{\ell-1} \xi)]^2 \alpha_{\ell-1,1}, \quad \alpha_{\ell,2} = \EE [ \sigma_\ell'( \tau_{\ell-1} \xi)]^2 \alpha_{\ell-1,2} + \frac14 \EE [ \sigma_\ell''( \tau_{\ell-1} \xi)]^2 \alpha_{\ell-1,4}^2, \\
    \alpha_{\ell,3} & = \EE [ \sigma_\ell'( \tau_{\ell-1} \xi)]^2 \alpha_{\ell-1,3} + \frac12 \EE [ \sigma_\ell''( \tau_{\ell-1} \xi)]^2 \alpha_{\ell-1,1}^2.
\end{align}
This concludes the proof for the diagonal entries of $\K_{\CK,\ell}$ in Lemma~\ref{lem:entry-wise-approx-CK-center}, and thus the proof of Lemma~\ref{lem:entry-wise-approx-CK-center}.

\begin{Lemma}[Consistent estimation of $\tau_0$]\label{lem:estim_tau0}
Let Assumption \ref{ass:high-dimen} hold and let $\tau_0 \equiv \sqrt{\tr \C^\circ/p}$. Then, as $n,p \to \infty$ with $p/n \to c \in (0, \infty)$, we have,
\begin{equation}
\frac{1}{n} \sum_{i=1}^n\left\|\x_i\right\|^2-\tau_{0}^{2} \to 0,
\end{equation}
almost surely.
\end{Lemma}
\begin{proof}[Proof of Lemma~\ref{lem:estim_tau0}] 
It follows from \eqref{eq:xi} that
\begin{equation}\label{eq:tau_estimation}
  \frac{1}{n} \sum_{i=1}^n\left\|\x_i\right\|^2=\frac{1}{n} \sum_{a=1}^K \sum_{i=1}^{n_a} \left( \frac{1}{p}\left\|\bmu_a\right\|^2-\frac{2}{p} \bmu_a^{\T} \z_i + \frac{1}{p} \left\|\z_i\right\|^2\right) .
\end{equation}
From Assumption \ref{ass:high-dimen}, we have $\| \bmu_a \| = O(1)$ so that $\frac{1}{n} \sum_{a=1}^K \sum_{i=1}^{n_a} \frac{1}{p}\left\|\bmu_a\right\|^2=O\left(p^{-1}\right)$. 
Since $\EE[\z_i] = \zo$, the second term $\frac{2}{p} \bmu_a^{\T} \z_i$ of \eqref{eq:tau_estimation} is a weighted sum of independent zero mean random variables and vanishes with probability one as $n, p \rightarrow \infty$ by a mere application of Chebyshev's inequality and the Borel--Cantelli lemma. 
Finally, using the strong law of large numbers on the third term of equation \eqref{eq:tau_estimation}, we have almost surely,
\begin{equation}
% \frac{1}{n} \frac{1}{p} \sum_{i=1}^n\left\|\z_i\right\|^2 = \frac{1}{p} \sum_{a=1}^K \frac{n_a}{n} \tr \C_a+o(1) =\sum_{a=1}^K \frac{n_a}{n} \frac{1}{p} \tr  \C^{\circ}+o(1) = \RED \tr \C^\circ/p + o(1) ?
\frac{1}{n} \frac{1}{p} \sum_{a=1}^K \sum_{i=1}^{n_a} \left\|\z_i\right\|^2 = \frac{1}{p} \sum_{a=1}^K \frac{n_a}{n} \tr \C_a+o(1) = \tr \C^\circ/p + o(1),
\end{equation}
where in the last line we use $\tr \C_a^{\circ}=O(\sqrt{p})$ from Assumption~\ref{ass:high-dimen}, and thus $\frac{1}{n} \sum_{i=1}^n\left\|\z_i\right\|^2-\tau_{0}^{2} \rightarrow 0$ almost surely. 
This concludes the proof of Lemma~\ref{lem:estim_tau0}.
\end{proof}

\subsection{ Proof of Theorem~\ref{theo:NTK} }
\label{subsec:proof_theo_NTK}

%We provide detailed discussions on our working assumptions in Section \ref{sec:preliminaries}

In this section, we provide detailed proof of Theorem~\ref{theo:NTK}.
We follows the same notations and working assumptions as in the proof of Theorem~\ref{theo:CK} in Appendix~\ref{subsec:proof_theo_CK}.

% so that in matrix form (by using the fact that $\| \A \| \leq n \| \A \|_\infty $ for $\A \in \RR^{n \times n}$ with $\| \A \|_{\infty} = \max_{ij} |\A|_{ij}$ and $\{  S_{ij} \}_{i,j} = O_{\| \cdot \|}(n^{-\frac12})$ \cite{couillet2016kernel}):
% where $O_{\| \cdot \|}(n^{-\frac12})$ denotes matrices of spectral norm order $O(n^{-\frac12})$, with

As already mentioned in \eqref{eq:K_NTK_relation}, the NTK matrices $\K_{\NTK,\ell}$ of layer $\ell$ can be defined, again in an iterative manner, via the CK matrices $\K_{\CK,\ell}$ and $\K'_{\CK,\ell}$ as follows \cite{jacot2018neural}:
\begin{align*}
    \K_{\NTK, 0}    & = \K_{\CK, 0} = \X^\T \X,                                \\
    \K_{\NTK, \ell} & = \K_{\CK,\ell} + \K_{\NTK, \ell-1} \circ \K'_{\CK,\ell}.
\end{align*}
where `$\A \circ \B$' denotes the Hadamard product between two matrices $\A,\B$, and $\K'_{\CK,\ell} \in \RR^{n \times n}$ denotes a CK matrix with nonlinear activation $\sigma_\ell'(\cdot)$ instead of $\sigma_\ell(\cdot)$ as for $\K_{\CK,\ell}$ in \eqref{eq:def_K_CK_K_NTK}, that is
\begin{equation}
  \label{eq:def_K'}
    [\K'_{\CK,\ell}]_{ij} = \EE_{u,v} [\sigma'_\ell(u) \sigma'_\ell(v)], \quad u,v \sim \NN \left(\zo, \begin{bmatrix} [\K'_{\CK,\ell-1}]_{ii} & [\K'_{\CK,\ell-1}]_{ij} \\ [\K'_{\CK,\ell-1}]_{ij} & [\K'_{\CK,\ell-1}]_{jj} \end{bmatrix}\right).
\end{equation}

To prove Theorem~\ref{theo:NTK}, one may want to apply directly Lemma~\ref{lem:entry-wise-approx-CK-center} to obtain an entry-wise approximation of $\K'_{\CK,\ell}$, with activation $\sigma'(\cdot)$ instead of $\sigma(\cdot)$.
This approach works only if one assumes both $\EE[\sigma_\ell(\tau_{\ell-1} \xi)] = 0$ and $\EE[\sigma_\ell'(\tau_{\ell-1} \xi)] = 0$ hold true for $\xi \sim \NN(0,1)$, which is too restrictive.

Instead, we will prove the following entry-wise approximation result on $\K'_{\CK,\ell}$, by assuming $\EE[\sigma_\ell(\tau_{\ell-1} \xi)] = 0$ only.

\begin{Lemma}
\label{lem:entry-wise-approx-CK'}
Under the assumptions and notations of Theorem~\ref{theo:CK}, for centered activation such that $\EE[\sigma_\ell(\tau_{\ell-1} \xi)] = 0$ with $\xi \sim \NN(0,1)$ as in the statement of Theorem~\ref{theo:CK}, we have, for all $\ell \in \{0, \ldots, L \}$ and $\K'_{\CK,\ell}$ defined in \eqref{eq:def_K'}, that uniformly over all $i \neq j \in \{1, \ldots, n\}$,
\begin{equation}\label{eq:K_CK'_ij}
    [\K_{\CK,\ell}']_{ij} = \dot{\alpha}_{\ell, 0} + \dot{\alpha}_{\ell,1} \x_i^\T \x_j + \dot{\alpha}_{\ell, 2} \left(\chi_i + \chi_j\right) + O(p^{-1}),
\end{equation}
and $[\K_{\CK,\ell}']_{ii} = \dot{\tau}_{\ell}^2 + O(p^{-1/2})$, for $\dot{\tau}_\ell = \sqrt{ \EE[\sigma_\ell'(\tau_{\ell - 1} \xi) \sigma_\ell'(\tau_{\ell - 1} \xi)] }$, $\dot{\alpha}_{\ell,0} = \EE\left[ \sigma_\ell'( \tau_{\ell-1} \xi)\right]^2$, and non-negative scalars $\dot{\alpha}_{\ell,1}, \dot{\alpha}_{\ell,2}$ such that
\begin{align*}
  %\dot{\alpha}_{\ell,1} = \EE [ \sigma_\ell'( \tau_{\ell-1} \xi)]^2 \alpha_{\ell-1,1}, \quad \dot{\alpha}_{\ell,2} = \EE [ \sigma_\ell'( \tau_{\ell-1} \xi)]^2 \alpha_{\ell-1,2} + \frac14 \EE [ \sigma_\ell''( \tau_{\ell-1} \xi)]^2 \alpha_{\ell-1,4}^2,
  \dot{\alpha}_{\ell, 1} = \EE\left[\sigma_\ell''( \tau_{\ell-1} \xi)\right]^2 \alpha_{\ell-1, 1}, \quad \dot{\alpha}_{\ell, 2} = \EE\left[\sigma_\ell'( \tau_{\ell-1} \xi)\right]\EE\left[\sigma_\ell'''( \tau_{\ell-1} \xi)\right]
    \frac{\alpha_{\ell-1,4}}{2 },
\end{align*}
with initializations $\dot{\tau}_0 = 0$, $\dot{\alpha}_{0,0} = 0$, $\dot{\alpha}_{0,1} = 1$, and $\dot{\alpha}_{0,2} = 0$.
\end{Lemma}

The proof of Lemma~\ref{lem:entry-wise-approx-CK'} is built upon Lemma~\ref{lem:entry-wise-approx-CK-center} and given as follows.

\subsubsection{ Proof of Lemma~\ref{lem:entry-wise-approx-CK'} }

By its definition in \eqref{eq:def_K'}, we write again using the Gram-Schmidt orthogonalization procedure for standard Gaussian random variable that
\begin{equation}
  u = \sqrt{ [\K_{\CK,\ell-1}]_{ii} } \cdot \xi_i, \quad v = \frac{[\K_{\CK,\ell-1}]_{ij}}{ \sqrt{ [\K_{\CK,\ell-1}]_{ii} } } \cdot \xi_i + \sqrt{ [\K_{\CK,\ell-1}]_{jj} - \frac{[\K_{\CK,\ell-1}]_{ij}^2 }{  [\K_{\CK,\ell-1}]_{ii}  } } \cdot \xi_j,
\end{equation}
for \emph{independent} $\xi_i, \xi_j \sim \NN(0,1)$.
As such, we have, for layer $\ell$ that
\begin{align}
  [\K'_{\CK,\ell}]_{ii} & = \EE \left[ \sigma_\ell' \left( \sqrt{ [\K_{\CK,\ell-1}]_{ii} } \cdot \xi_i  \right) \sigma_\ell' \left( \sqrt{ [\K_{\CK,\ell-1}]_{ii} } \cdot \xi_i  \right) \right] \label{eq:def_K_CK'_ii_ell}                                                                                                                                              \\
  [\K'_{\CK,\ell}]_{ij} & = \EE\left[ \sigma_\ell' \left( \sqrt{ [\K_{\CK,\ell-1}]_{ii} } \cdot \xi_i \right) \right. \nonumber                                                                                                                                                                   \\
                        & \left. \times \sigma_\ell' \left( \frac{[\K_{\CK,\ell-1}]_{ij}}{ \sqrt{ [\K_{\CK,\ell-1}]_{ii} } } \cdot \xi_i + \sqrt{ [\K_{\CK,\ell-1}]_{jj} - \frac{[\K_{\CK,\ell-1}]_{ij}^2 }{  [\K_{\CK,\ell-1}]_{ii}  } } \cdot \xi_j \right) \right], \label{eq:def_K_CK'_ij_ell}
\end{align}%\label{eq:Sigma_ell_expand}
where the expectations taken with respect to the \emph{independent} random variables $\xi_i$ and $\xi_j$, as in the proof of Lemma~\ref{lem:entry-wise-approx-CK-center}.

It follows from Lemma~\ref{lem:entry-wise-approx-CK-center} that for $\EE[\sigma_\ell(\tau_{\ell-1} \xi)] = 0$, we have, for layer $\ell-1$ that:
\begin{equation}%\label{eq:K_CK_ij}
  [\K_{\CK,\ell-1}]_{ij} = \alpha_{\ell-1,1} \x_i^\T \x_j + \alpha_{\ell-1,2} (t_i + \psi_i) (t_j + \psi_j) + \alpha_{\ell-1,3} \left(\frac1p \z_i^\T \z_j \right)^2 + S_{ij} + O(p^{-3/2}),
\end{equation}
for $i \neq j$, and
\begin{equation}%\label{eq:K_CK_ii}
  [\K_{\CK,\ell-1}]_{ii} = \tau_{\ell-1}^2 + \alpha_{\ell-1,4} \chi_i + \alpha_{\ell-1,5} (t_i + \psi_i)^2 + O(p^{-3/2}),
\end{equation}
for non-negative scalars $\alpha_{\ell,1}, \alpha_{\ell,2}, \alpha_{\ell,3}, \alpha_{\ell,4}, \alpha_{\ell,5}$ defined in Lemma~\ref{lem:entry-wise-approx-CK-center}.

Then, the proof of Lemma~\ref{lem:entry-wise-approx-CK'} is akin to that of Lemma~\ref{lem:entry-wise-approx-CK-center},
with activation $\sigma'$ instead of $\sigma$.

\paragraph{On the diagonal.}%\label{proof:on the diagnal}
Since
\begin{align*}
  \sqrt{ [\K_{\CK,\ell-1}]_{ii}} & = \sqrt{ \tau_{\ell-1}^2 + \alpha_{\ell-1,4} \chi_i + \alpha_{\ell-1,5} (t_i + \psi_i)^2 + O(p^{-3/2}) }                                                                                                 \\
                                 & = \tau_{\ell-1} + \frac1{2 \tau_{\ell-1} } \left( \alpha_{\ell-1,4} \chi_i + \alpha_{\ell-1,5} (t_i + \psi_i)^2 \right) - \frac{ \alpha_{\ell-1,4}^2 }{8 \tau_{\ell-1}^3} (t_i + \psi_i)^2 + O(p^{-3/2}) \\
                                 & = \tau_{\ell-1} + \frac1{2 \tau_{\ell-1} } \alpha_{\ell-1,4} \chi_i + \frac{ 4 \tau_{\ell-1}^2 \alpha_{\ell-1,5} - \alpha_{\ell-1,4}^2 }{8 \tau_{\ell-1}^3} (t_i + \psi_i)^2 + O(p^{-3/2}),
\end{align*}
and denote the shortcut $f(t) = \left( \sigma'(t) \right)^2$, we get, again by Taylor-expansion that
\begin{align*}
  &[\K_{\CK,\ell}']_{ii} = \EE \left[ f \left( \sqrt{ [\K_{\CK,\ell-1}]_{ii} } \cdot \xi \right) \right]                                                                       \\
                       & = \EE\left[ f(\tau_{\ell - 1} \xi) + f'(\tau_{\ell - 1} \xi) \xi \left( \frac1{2 \tau_{\ell-1} } \alpha_{\ell-1,4} \chi_i + \frac{ 4 \tau_{\ell-1}^2 \alpha_{\ell-1,5} - \alpha_{\ell-1,4}^2 }{8 \tau_{\ell-1}^3} (t_i + \psi_i)^2 \right) \right] \\
                       & \quad + \EE \left[ \frac12 f''(\tau_{\ell - 1} \xi) \xi^2 \right] \frac{\alpha_{\ell-1,4}^2 }{4 \tau_{\ell-1}^2 } (t_i + \psi_i)^2 + O(p^{-3/2})                                                                                                   \\
                       & = \EE[f(\tau_{\ell - 1} \xi)] + \EE[f''(\tau_{\ell - 1} \xi)] \left( \frac12 \alpha_{\ell-1,4} \chi_i + \frac{ 4 \tau_{\ell-1}^2 \alpha_{\ell-1,5} - \alpha_{\ell-1,4}^2 }{8 \tau_{\ell-1}^2 } (t_i + \psi_i)^2 \right)                            \\
                       & \quad + \EE \left[ \frac12 f''(\tau_{\ell - 1} \xi) \xi^2 \right] \frac{\alpha_{\ell-1,4}^2 }{4 \tau_{\ell-1}^2 } (t_i + \psi_i)^2 + O(p^{-3/2})                                                                                                   \\
                       & = \EE[f(\tau_{\ell - 1} \xi)] +  \frac{\alpha_{\ell-1,4}}2 \EE[f''(\tau_{\ell - 1} \xi)] \chi_i                                                                                                                                                    \\
                       & \quad + \frac{ 4 \alpha_{\ell-1,5}\EE[f''(\tau_{\ell - 1} \xi)] +  \alpha_{\ell-1,4}^2 \EE [ f''''(\tau_{\ell - 1} \xi) ]  }{8 } (t_i + \psi_i)^2 + O(p^{-3/2}),
\end{align*}
where we used the facts that
\begin{equation}
  \EE [ f'(\tau_{\ell - 1} \xi) \xi] = \tau_{\ell-1} \EE[f''(\tau_{\ell - 1} \xi)], \quad
  \EE [ f''(\tau_{\ell - 1} \xi) (\xi^2 - 1) ] = \tau_{\ell-1}^2 \EE[f''''(\tau_{\ell - 1} \xi)],
\end{equation}
for $\xi \sim \NN(0,1)$, again as a consequence of the Gaussian integration by parts formula.

This allows us to conclude that the proof of Lemma~\ref{lem:entry-wise-approx-CK'} for the diagonal entries of $\K_{\CK,\ell}'$, with $\dot{\tau}_\ell = \sqrt{ \EE[\sigma_\ell'(\tau_{\ell - 1} \xi) \sigma_\ell'(\tau_{\ell - 1} \xi)]}$.

% As a consequence, we obtain the following relation
% \begin{align*}
%   \dot{\tau}_\ell      & = \sqrt{ \EE[\sigma_\ell'^2(\tau_{\ell - 1} \xi)] },
%   \quad \dot{\alpha}_{\ell,4} = \alpha_{\ell-1,4} \EE \left[ (\sigma_\ell''(\tau_{\ell - 1} \xi) )^2 + \sigma_\ell'(\tau_{\ell - 1} \xi) \sigma_\ell'''(\tau_{\ell - 1} \xi) \right],                                                                                               \\
%   \dot{\alpha}_{\ell,5} & = \alpha_{\ell-1,5} \EE \left[ (\sigma_\ell''(\tau_{\ell - 1} \xi) )^2 + \sigma_\ell'(\tau_{\ell - 1} \xi) \sigma_\ell'''(\tau_{\ell - 1} \xi) \right]                                                                                                    \\
%                   & \quad + \frac{ \alpha_{\ell-1,4}^2 }4 \EE\left[  \sigma_\ell'(\tau_{\ell - 1} \xi) \sigma_\ell'''''(\tau_{\ell - 1} \xi) + 4 \sigma_\ell''(\tau_{\ell - 1} \xi) \sigma_\ell''''(\tau_{\ell - 1} \xi) + 3 (\sigma_\ell'''(\tau_{\ell - 1} \xi))^2 \right].
% \end{align*}

\paragraph{Off the diagonal.}%\label{proof:on the diagnal}
We now move on to the non-diagonal entries of $\K_{\CK,\ell}'$. 
First recall from Lemma~\ref{lem:entry-wise-approx-CK-center} that for $i \neq j$, 
\begin{align*}
  \frac{[\K_{\CK,\ell-1}]_{ij}}{ \sqrt{ [\K_{\CK,\ell-1}]_{ii} } } & =\frac1{\tau_{\ell-1}} [\K_{\CK,\ell-1}]_{ij} + S_{ij} +  O(p^{-3/2}) = O(p^{-1/2}),
\end{align*}
with
\begin{align*}
  [\K_{\CK,\ell-1}]_{ij} & = \alpha_{\ell-1,1} \x_i^\T \x_j + \alpha_{\ell-1,2} (t_i + \psi_i) (t_j + \psi_j) + \alpha_{\ell-1,3} \left(\frac1p \z_i^\T \z_j \right)^2 + S_{ij} + O(p^{-3/2}) = O(p^{-1/2}),
\end{align*}
% as in \Cref{eq:K_CK_ij}, where we recall that $S_{ij} = O(p^{-1})$ denotes a matrix of the form $\frac1p \z_i^\T \z_j \left(  \gamma_1 (t_i + \psi_i) + \gamma_2 (t_j + \psi_j) \right)$ and of vanishing spectral norm as defined in \eqref{eq:def_Sij}.
and therefore
\begin{align*}
   & \sqrt{ [\K_{\CK,\ell-1}]_{jj} - \frac{[\K_{\CK,\ell-1}]_{ij}^2 }{  [\K_{\CK,\ell-1}]_{ii}  } }                                                                                                                                                                                       \\
   & = \tau_{\ell-1} + \frac1{2 \tau_{\ell-1} } \left( \alpha_{\ell-1,4} \chi_j - \frac{ \alpha_{\ell-1,1}^2 }{ \tau_{\ell-1}^2 } \left (\frac1p \z_i^\T \z_j \right)^2 \right) + \frac{ 4 \tau_{\ell-1}^2 \alpha_{\ell-1,5} - \alpha_{\ell-1,4}^2 }{8 \tau_{\ell-1}^3 } (t_j + \psi_j)^2 \\
   & \quad + O(p^{-3/2}).
  %&\equiv \sqrt{\alpha_1} + \frac1{2 \sqrt{\alpha_1} } \Delta_{ij} + O(p^{-3/2})
\end{align*}

As a consequence, we get, again by Taylor expansion that
\begin{align*}
   & \sigma_\ell' \left( \sqrt{ [\K_{\CK,\ell-1}]_{ii}} \cdot \xi_i \right) \sigma_\ell' \left( \frac{[\K_{\CK,\ell-1}]_{ij}}{ \sqrt{ [\K_{\CK,\ell-1}]_{ii} } } \cdot \xi_i + \sqrt{ [\K_{\CK,\ell-1}]_{jj} - \frac{[\K_{\CK,\ell-1}]_{ij}^2 }{  [\K_{\CK,\ell-1}]_{ii}  } } \cdot \xi_j \right)                                                 \\
   & = \sigma_\ell' \left( \tau_{\ell-1} \xi_i + \frac1{2 \tau_{\ell-1} } \alpha_{\ell-1,4} \chi_i \xi_i + \frac{ 4 \tau_{\ell-1}^2 \alpha_{\ell-1,5} - \alpha_{\ell-1,4}^2 }{8 \tau_{\ell-1}^3} (t_i + \psi_i)^2 \xi_i + O(p^{-3/2}) \right)                                                                                                     \\
   & \quad \times \sigma_\ell' \left( \frac1{\tau_{\ell-1}} [\K_{\CK,\ell-1}]_{ij} \xi_i + \tau_{\ell-1} \xi_j + \frac1{2 \tau_{\ell-1} } \left( \alpha_{\ell-1,4} \chi_j - \frac{ \alpha_{\ell-1,1}^2 }{ \tau_{\ell-1}^2 } \left (\frac1p \z_i^\T \z_j \right)^2 \right) \xi_j \right.                                                           \\
   & \quad \left.+ \frac{ 4 \tau_{\ell-1}^2 \alpha_{\ell-1,5} - \alpha_{\ell-1,4}^2 }{8 \tau_{\ell-1}^3 } (t_j + \psi_j)^2 \xi_j + O(p^{-3/2})\right)                                                                                                                                                                                             \\
   & = \left( \sigma_\ell'( \tau_{\ell-1} \xi_i)  +  \sigma_\ell''( \tau_{\ell-1} \xi_i) \xi_i \left(  \frac1{2 \tau_{\ell-1} } \alpha_{\ell-1,4} \chi_i + \frac{ 4 \tau_{\ell-1}^2 \alpha_{\ell-1,5} - \alpha_{\ell-1,4}^2 }{8 \tau_{\ell-1}^3} (t_i + \psi_i)^2 \right) \right.                                                                 \\
   & \quad \quad \left. + \sigma_\ell'''( \tau_{\ell-1} \xi_i) \xi_i^2 \frac{\alpha_{\ell-1,4}^2 }{ 8\tau_{\ell-1}^2 } (t_i + \psi_i)^2 \right)                                                                                                                                                                                                   \\
   & \quad \times \left( \sigma_\ell'( \tau_{\ell-1} \xi_j)  +  \sigma_\ell''( \tau_{\ell-1} \xi_j) \left( \frac{ [\K_{\CK,\ell-1}]_{ij} }{\tau_{\ell-1}} \xi_i + \frac1{2 \tau_{\ell-1} } \left( \alpha_{\ell-1,4} \chi_j - \frac{ \alpha_{\ell-1,1}^2 }{ \tau_{\ell-1}^2 } \left (\frac1p \z_i^\T \z_j \right)^2 \right) \xi_j  \right. \right. \\
   & \quad \quad \left. +  \frac{ 4 \tau_{\ell-1}^2 \alpha_{\ell-1,5} - \alpha_{\ell-1,4}^2 }{8 \tau_{\ell-1}^3 } (t_j + \psi_j)^2 \xi_j \right)                                                                                                                                                                                                  \\
   & \quad \quad \left.+ \frac12 \sigma_\ell'''( \tau_{\ell-1} \xi_j)  \left( \frac{ \alpha_{\ell-1,1} \frac1p \z_i^\T \z_j }{\tau_{\ell-1}} \xi_i  + \frac{ \alpha_{\ell-1,4} (t_j + \psi_j) }{2 \tau_{\ell-1}} \xi_j \right)^2  \right) +  O(p^{-3/2})                                                                                          \\
   & \equiv \left( \sigma_\ell'( \tau_{\ell-1} \xi_i) + T_{1,i} + T_{2,i} \right) \left( \sigma_\ell'( \tau_{\ell-1} \xi_j) + T_{3,ij} + T_{3,j} + T_{4,ij} + T_{4,j} + S_{ij} \right) + O(p^{-3/2}),
\end{align*}

where we denote the shortcuts:
\begin{align*}
  T_{1,i} & = \sigma_\ell''( \tau_{\ell-1} \xi_i) \xi_i \cdot \frac{\alpha_{\ell-1,4}}{2 \tau_{\ell-1} }  \chi_i = O(p^{-1/2}),                                                                                                                                                                                 \\
  T_{2,i} & = \left( \frac{ \alpha_{\ell-1,5} \sigma_\ell''( \tau_{\ell-1} \xi_i) \xi_i  }{ 2 \tau_{\ell-1} } + \alpha_{\ell-1,4}^2 \frac{ \sigma_\ell'''( \tau_{\ell-1} \xi_i) \xi_i^2 \tau_{\ell-1} -  \sigma_\ell''( \tau_{\ell-1} \xi_i) \xi_i }{ 8 \tau_{\ell-1}^3 } \right) (t_i + \psi_i)^2 = O(p^{-1}),
\end{align*}
that \emph{only} depend on $\xi_i$; and
\begin{align*}
  T_{3,ij} & = \sigma_\ell''( \tau_{\ell-1} \xi_j) \xi_i \cdot \frac{ [\K_{\CK,\ell-1}]_{ij} }{\tau_{\ell-1}} = O(p^{-1/2}),                                                \\
  T_{4,ij} & = \frac12 \sigma_\ell'''( \tau_{\ell-1} \xi_j) \xi_i^2 \cdot \frac{ \alpha_{\ell-1,1}^2 }{\tau_{\ell-1}^2 } \left (\frac1p \z_i^\T \z_j \right)^2 = O(p^{-1}),
\end{align*}
that depend on both $\xi_i$ and $\xi_j$; and
\begin{align*}
  T_{3,j} & = \sigma_\ell''( \tau_{\ell-1} \xi_j) \xi_j \cdot \left( \frac{\alpha_{\ell-1,4}}{2 \tau_{\ell-1} }  \chi_j - \frac{ \alpha_{\ell-1,1}^2 }{ 2\tau_{\ell-1}^3 } \left (\frac1p \z_i^\T \z_j \right)^2 + \frac{ 4 \tau_{\ell-1}^2 \alpha_{\ell-1,5} - \alpha_{\ell-1,4}^2 }{8 \tau_{\ell-1}^3 } (t_j + \psi_j)^2 \right) \\
          & = O(p^{-1/2}),                                                                                                                                                                                                                                                                                                         \\ %+ O(p^{-1}) \\ 
  T_{4,j} & = \frac12 \sigma_\ell'''( \tau_{\ell-1} \xi_j) \xi_j^2 \cdot \frac{ \alpha_{\ell-1,4}^2 }{4 \tau_{\ell-1}^2 } (t_j + \psi_j)^2 = O(p^{-1}),
\end{align*}
that \emph{only} depend on $\xi_j$, where we particularly note that the cross terms are of the form $S_{ij}$.

As such, we have
\begin{align*}
   & \sigma_\ell' \left( \sqrt{ [\K_{\CK,\ell-1}]_{ii}} \cdot \xi_i \right) \sigma_\ell' \left( \frac{[\K_{\CK,\ell-1}]_{ij}}{ \sqrt{ [\K_{\CK,\ell-1}]_{ii} } } \cdot \xi_i + \sqrt{ [\K_{\CK,\ell-1}]_{jj} - \frac{[\K_{\CK,\ell-1}]_{ij}^2 }{  [\K_{\CK,\ell-1}]_{ii}  } } \cdot \xi_j \right) \\
   & = \sigma_\ell'( \tau_{\ell-1} \xi_i) \sigma_\ell'( \tau_{\ell-1} \xi_j) + \sigma_\ell'( \tau_{\ell-1} \xi_i) (T_{3,ij} + T_{4,ij}) + \sigma_\ell'( \tau_{\ell-1} \xi_i) (T_{3,j} + T_{4,j})                                                                                                  \\
   & \quad + \sigma_\ell'( \tau_{\ell-1} \xi_j) (T_{1,i} + T_{2,i}) + T_{1,i} (T_{3,ij} + T_{3,j} ) + S_{ij} + O(p^{-3/2})
   \\
   & = \sigma_\ell'( \tau_{\ell-1} \xi_i) \sigma_\ell'( \tau_{\ell-1} \xi_j) + \sigma_\ell'( \tau_{\ell-1} \xi_i) (T_{3,ij} + O(p^{-1})) + \sigma_\ell'( \tau_{\ell-1} \xi_i) (T_{3,j} + O(p^{-1}))                                                                                                  \\
   & \quad + \sigma_\ell'( \tau_{\ell-1} \xi_j) (T_{1,i} + O(p^{-1})) + O(p^{-1}) + S_{ij} + O(p^{-3/2})
   \\
   & = \sigma_\ell'( \tau_{\ell-1} \xi_i) \sigma_\ell'( \tau_{\ell-1} \xi_j) + \sigma_\ell'( \tau_{\ell-1} \xi_i) \left(\sigma_\ell''( \tau_{\ell-1} \xi_j) \xi_i \cdot \frac{ [\K_{\CK,\ell-1}]_{ij} }{\tau_{\ell-1}} + O(p^{-1})\right) \\
   & \quad + \sigma_\ell'( \tau_{\ell-1} \xi_i) \left(\sigma_\ell''( \tau_{\ell-1} \xi_j) \xi_j \cdot \left( \frac{\alpha_{\ell-1,4}}{2 \tau_{\ell-1} }  \chi_j + O(p^{-1})\right) + O(p^{-1})\right)                                                                                                  \\
   & \quad + \sigma_\ell'( \tau_{\ell-1} \xi_j) \left(\sigma_\ell''( \tau_{\ell-1} \xi_i) \xi_i \cdot \frac{\alpha_{\ell-1,4}}{2 \tau_{\ell-1} }  \chi_i + O(p^{-1})\right) + O(p^{-1}) + S_{ij} + O(p^{-3/2})
   \\
   & = \sigma_\ell'( \tau_{\ell-1} \xi_i) \sigma_\ell'( \tau_{\ell-1} \xi_j) + \sigma_\ell'( \tau_{\ell-1} \xi_i) \xi_i \sigma_\ell''( \tau_{\ell-1} \xi_j) \cdot \frac{ [\K_{\CK,\ell-1}]_{ij} }{\tau_{\ell-1}} \\
   & \quad + \sigma_\ell'( \tau_{\ell-1} \xi_i) \sigma_\ell''( \tau_{\ell-1} \xi_j) \xi_j \cdot \frac{\alpha_{\ell-1,4}}{2 \tau_{\ell-1} }  \chi_j  +  \sigma_\ell''( \tau_{\ell-1} \xi_i) \xi_i \sigma_\ell'( \tau_{\ell-1} \xi_j) \cdot \frac{\alpha_{\ell-1,4}}{2 \tau_{\ell-1} }  \chi_i + S_{ij}+ O(p^{-1}) 
\end{align*}

This further leads to the following approximation of $[\K'_{\CK,\ell}]_{ij}$ as
\begin{align*}
  % [\K'_{\CK,\ell}]_{ij} & = \\
  &\EE\left[ \sigma_\ell' \left( \sqrt{ [\K_{\CK,\ell-1}]_{ii} } \cdot \xi_i \right) \times  \sigma_\ell' \left( \frac{[\K_{\CK,\ell-1}]_{ij}}{ \sqrt{ [\K_{\CK,\ell-1}]_{ii} } } \cdot \xi_i + \sqrt{ [\K_{\CK,\ell-1}]_{jj} - \frac{[\K_{\CK,\ell-1}]_{ij}^2 }{  [\K_{\CK,\ell-1}]_{ii}  } } \cdot \xi_j \right) \right]\\
  & = \EE \bigg[ \sigma_\ell'( \tau_{\ell-1} \xi_i) \sigma_\ell'( \tau_{\ell-1} \xi_j) + \sigma_\ell'( \tau_{\ell-1} \xi_i) \xi_i \sigma_\ell''( \tau_{\ell-1} \xi_j) \cdot \frac{ [\K_{\CK,\ell-1}]_{ij} }{\tau_{\ell-1}} \\
   & \quad + \sigma_\ell'( \tau_{\ell-1} \xi_i) \sigma_\ell''( \tau_{\ell-1} \xi_j) \xi_j \cdot \frac{\alpha_{\ell-1,4}}{2 \tau_{\ell-1} }  \chi_j  +  \sigma_\ell''( \tau_{\ell-1} \xi_i) \xi_i \sigma_\ell'( \tau_{\ell-1} \xi_j) \cdot \frac{\alpha_{\ell-1,4}}{2 \tau_{\ell-1} }  \chi_i + S_{ij}+ O(p^{-1})\bigg]\\
   & = \EE\left[ \sigma_\ell'( \tau_{\ell-1} \xi)\right]^2 + 
   \EE\left[\sigma_\ell''( \tau_{\ell-1} \xi)\right]^2 \cdot [\K_{\CK,\ell-1}]_{ij}\\
   & \quad + \EE\left[\sigma_\ell'( \tau_{\ell-1} \xi)\right]\EE\left[\sigma_\ell'''( \tau_{\ell-1} \xi)\right]
   \cdot \frac{\alpha_{\ell-1,4}}{2 }  \left(\chi_i + \chi_j \right) + S_{ij} + O(p^{-1})\\
   & = \EE\left[ \sigma_\ell'( \tau_{\ell-1} \xi)\right]^2 + 
   \EE\left[\sigma_\ell''( \tau_{\ell-1} \xi)\right]^2 \cdot \alpha_{\ell-1, 1} \x_i^\T \x_j  + \EE\left[\sigma_\ell'( \tau_{\ell-1} \xi)\right]\EE\left[\sigma_\ell'''( \tau_{\ell-1} \xi)\right]
   \cdot \frac{\alpha_{\ell-1,4}}{2 }  \left(\chi_i + \chi_j \right) + O(p^{-1})\\
   & = \dot{\alpha}_{\ell, 0} + \dot{\alpha}_{\ell,1}  \x_i^\T \x_j + \dot{\alpha}_{\ell, 2}  \left(\chi_i + \chi_j\right) + O(p^{-1})
\end{align*}%\label{eq:Sigma_ell_expand}
with $\dot{\alpha}_{\ell, 0} = \EE\left[ \sigma_\ell'( \tau_{\ell-1} \xi)\right]^2$ and
\begin{align*}
    \dot{\alpha}_{\ell, 1} = \EE\left[\sigma_\ell''( \tau_{\ell-1} \xi)\right]^2 \alpha_{\ell-1, 1}, \quad \dot{\alpha}_{\ell, 2} = \EE\left[\sigma_\ell'( \tau_{\ell-1} \xi)\right]\EE\left[\sigma_\ell'''( \tau_{\ell-1} \xi)\right]
    \frac{\alpha_{\ell-1,4}}{2 },
\end{align*}
where we used again the Gaussian integration by parts formula.
This concludes the proof of Lemma~\ref{lem:entry-wise-approx-CK'}.

\subsubsection{End of proof of Theorem~\ref{theo:NTK}}

With Lemma~\ref{lem:entry-wise-approx-CK'} at hand, it remains to show, as in the proof of Theorem~\ref{theo:CK} in Section~\ref{subsec:proof_theo_CK}, that the on-~and~off-diagonal entries of $\K_{\NTK,\ell}$ satisfy the following relation:
\begin{equation}
  [\K_{\NTK, \ell}]_{ij} = \beta_{\ell,1} \x_i^\T \x_j + \beta_{\ell,2} (t_i + \psi_i) (t_j + \psi_j) + \beta_{\ell,3} \left(\frac1p \z_i^\T \z_j \right)^2 + S_{ij} + O(p^{-3/2}),
\end{equation}
for all $i \neq j$ and
\begin{equation}
  [\K_{\NTK, \ell}]_{ii} = \kappa_{\ell}^2 + O(p^{-1/2}),
\end{equation}
with $\tau_\ell^2 + \kappa_{\ell-1}^2 \dot{\tau}_{\ell}^2$.
% with $\kappa_\ell^2 = \tau^2_\ell + \dot{\tau}^2_\ell$.
% \begin{align*}
%     \kappa_\ell^2 = \tau_\ell^2 + \kappa_{\ell-1}^2 \dot{\tau}_{\ell}^2, \quad ???
% \end{align*}

We again proof the above approximation by induction on $\ell \in \{1, \ldots, L \}$.
For the case $\ell = 0$ we have $\K_{\NTK,0} = \X^\T \X$ so that
\begin{equation}
  \beta_{\ell,1} = 1, \quad \beta_{\ell,2} = \beta_{\ell,3} =0, \quad \kappa_0 = \tau_0.
\end{equation}

We then assume the induction hypothesis holds for $\ell - 1$, that is
\begin{align*}
    [\K_{\NTK, \ell-1}]_{ij} &= \beta_{\ell-1,1} \x_i^\T \x_j + \beta_{\ell-1,2} (t_i + \psi_i) (t_j + \psi_j) + \beta_{\ell-1,3} \left(\frac1p \z_i^\T \z_j \right)^2 + S_{ij} + O(p^{-3/2}), \\ 
    [\K_{\NTK, \ell-1}]_{ii} &= \kappa_{\ell-1}^2 + O(p^{-1/2}),
\end{align*}
and work on the off-diagonal and then diagonal entries of $\K_{\NTK,\ell}$ at the layer $\ell$ as follows.

\paragraph{Off the diagonal.}

By Lemma~\ref{lem:entry-wise-approx-CK-center}~and~Lemma~\ref{lem:entry-wise-approx-CK'}, we have, for the non-diagonal entries of $\K_{\NTK,\ell}$ with $i \neq j$ that
\begin{align*}
  [\K_{\NTK, \ell}]_{ij} & = [\K_{\CK, \ell}]_{ij} + [\K_{\NTK, \ell-1}]_{ij} [\K'_{\CK, \ell}]_{ij}                                                                                  \\
                         & = \alpha_{\ell,1} \x_i^\T \x_j + \alpha_{\ell,2} (t_i + \psi_i) (t_j + \psi_j)  + \alpha_{\ell,3} \left(\frac1p \z_i^\T \z_j \right)^2                           \\
                         & \quad + \left(\beta_{\ell - 1,1} \x_i^\T \x_j + \beta_{\ell - 1,2} (t_i + \psi_i) (t_j + \psi_j)
  + \beta_{\ell - 1,3} \left(\frac1p \z_i^\T \z_j \right)^2 \right)                                                                                                                   \\
                         & \quad \times \left(\dot{\alpha}_{\ell, 0} + \dot{\alpha}_{\ell,1} \cdot \x_i^\T \x_j + \dot{\alpha}_{\ell, 2} \cdot \left(\chi_i + \chi_j\right) + O(p^{-1}) \right)  + O(p^{-3/2})                                                    \\
                         & = (\alpha_{\ell,1} +\beta_{\ell - 1,1} \cdot \dot{\alpha}_{\ell, 0})  \x_i^\T \x_j + (\alpha_{\ell,2}+\beta_{\ell - 1,2} \cdot \dot{\alpha}_{\ell, 0}) (t_i + \psi_i) (t_j + \psi_j) \\
                         & \quad + (\alpha_{\ell,3} +\beta_{\ell - 1,3} \cdot \dot{\alpha}_{\ell, 0} + \beta_{\ell - 1,1}\cdot  \dot{\alpha}_{\ell,1}) \left(\frac1p \z_i^\T \z_j \right)^2  + S_{ij} + O(p^{-3/2}),
\end{align*}
so that
\begin{align*}
  \beta_{\ell,1} &=\alpha_{\ell,1} +\beta_{\ell - 1,1} \dot{\alpha}_{\ell, 0},\\
  \beta_{\ell,2} &= \alpha_{\ell,2}+\beta_{\ell - 1,2} \dot{\alpha}_{\ell, 0},\\
  \beta_{\ell,3} &= \alpha_{\ell,3}  +\beta_{\ell - 1,3} \dot{\alpha}_{\ell, 0} + \beta_{\ell - 1,1} \dot{\alpha}_{\ell,1},
\end{align*}
with
\begin{align*}
  \dot{\alpha}_{\ell, 0} = \EE\left[ \sigma_\ell'( \tau_{\ell-1} \xi)\right]^2, \quad \dot{\alpha}_{\ell, 1} = \EE\left[\sigma_\ell''( \tau_{\ell-1} \xi)\right]^2 \alpha_{\ell-1, 1},
\end{align*}
as in Lemma~\ref{lem:entry-wise-approx-CK'}.
This concludes the proof of the non-diagonal entries of $\K_{\NTK,\ell}$.

\paragraph{On the diagonal.} 
We now evaluate the diagonal entries of $\K_{\NTK,\ell}$ as
\begin{align*}
  [\K_{\NTK, \ell}]_{ii} & = [\K_{\CK,\ell}]_{ii} + [\K_{\NTK, \ell-1}]_{ii} \cdot [\K'_{\CK, \ell}]_{ii} = \tau_\ell^2 + \kappa_{\ell-1}^2 \dot{\tau}_{\ell}^2 + O(p^{-1/2}),
\end{align*}
as a consequence of Lemma~\ref{lem:entry-wise-approx-CK-center}, Lemma~\ref{lem:entry-wise-approx-CK'}, and the induction hypothesis.
We thus obtain 
\begin{equation}
  \kappa_\ell^2 = \tau_\ell^2 + \kappa_{\ell-1}^2 \dot{\tau}_{\ell}^2 ,
\end{equation}
with $\kappa_0 = \tau_0 =  \sqrt{\tr \C^\circ/p}$ defined in Assumption~\ref{ass:high-dimen}, where we recall $\dot{\tau}_0 = 0$ from Lemma~\ref{lem:entry-wise-approx-CK'}.

\paragraph{Assembling in matrix form.} 
Putting everything together in matrix form, we obtain, as in the proof of Theorem~\ref{theo:CK} in Appendix~\ref{subsec:proof_theo_CK} that
\begin{equation*}
  [\K_{\NTK, \ell}]_{ij} = \beta_{\ell,1} \x_i^\T \x_j + \beta_{\ell,2} (t_i + \psi_i) (t_j + \psi_j) + \beta_{\ell,3} \left(\frac1p \z_i^\T \z_j \right)^2 + S_{ij} + O(p^{-3/2}),
\end{equation*}
and
\begin{equation*}
  [\K_{\NTK, \ell}]_{ii}  = [\K_{\CK,\ell}]_{ii} + [\K_{\NTK, \ell-1}]_{ii} \cdot [\K'_{\CK, \ell}]_{ii} = \tau_\ell^2 + \kappa_{\ell-1}^2 \dot{\tau}_{\ell}^2 + O(p^{-1/2}) \equiv \kappa_\ell^2 + O(p^{-1/2}),
\end{equation*}
so that in matrix form
\begin{equation}
  \K_{\NTK, \ell}  = \beta_{\ell,1} \X^\T \X + \V \B_\ell \V^\T + (  \kappa_\ell^2 - \tau_0^2 \beta_{\ell,1} ) \I_n + O_{\| \cdot \|}(p^{-\frac12}),
\end{equation}
where $O_{\| \cdot \|}(p^{-1/2})$ denotes matrices of spectral norm order $O(p^{-1/2})$ as $n, p \to \infty$, with
\begin{equation}
  \V = [\J/\sqrt p,~\bpsi] \in \RR^{n \times (K+1)}, \quad \B_\ell = \begin{bmatrix} \beta_{\ell,2} \bt \bt^\T + \beta_{\ell,3} \bT & \beta_{\ell,2}\bt \\ \beta_{\ell,2}\bt^\T & \beta_{\ell,2} \end{bmatrix},
\end{equation}
%which suffice that as $n,p \to \infty$, $\| \K_{\NTK,\ell} - \tilde \K_{\NTK,\ell} \| \to 0$
which concludes the proof of Theorem~\ref{theo:NTK}.

\subsection{Two equivalent centering approaches in the single-hidden-layer case}
\label{subsec:equivalent_centering}

In this section, we aim to show that ``centering'' the CK matrices $\K_{\CK}$ by pre-~and~post-multiplying $\P = \I_n -  \one_n \one_n^\T/n$ performed in \cite[Theorem~1]{ali2022random} is \emph{equivalent} to take $\EE[\sigma(\tau_0 \xi)] = 0$ as in our Theorem~\ref{theo:CK} in the single-hidden-layer $\ell=1$ setting, in the sense that one has
\begin{equation}\label{eq:diff_centering}
    \| \P (\K_{\CK,1} - \tilde \K_{\CK,1}) \P \| \to 0,
\end{equation}
almost surely as $n,p \to \infty$, for the \emph{same} $\tilde \K_{\CK,1}$ as defined in Theorem~\ref{theo:CK} and an \emph{arbitrary} choice of $\EE[\sigma(\tau_0 \xi)]$ (so in particular, one may freely take $\EE[\sigma(\tau_0 \xi)] \neq 0$ which is different from the setting of our Theorem~\ref{theo:CK}).
The proof is as follows.

First note that the assumption $\EE[\sigma(\tau_0 \xi)] = 0$ is \emph{only} used for the off-diagonal entries of the CK matrix $\K_{\CK,1}$, so we focus, in the sequel, only on the off-diagonal terms, while the discussions on the on-diagonal entries are the same as in Appendix~\ref{subsec:proof_theo_CK}.

By its definition in \eqref{eq:K_CK_relation} and the fact that $\K_{\CK,0} = \X^\T \X$, one has
\begin{equation}
    %[\K_{\CK,1}]_{ij} = \EE_\w[\sigma(\w^\T \x_i) \sigma(\w^\T \x_j)]
    [\K_{\CK,1}]_{ij} = \EE_{u,v} [\sigma_1(u) \sigma_1(v)],~\text{with}~u,v \sim \NN \left(\zo, \begin{bmatrix} \| \x_i \|^2 & \x_i^\T \x_j \\ \x_i^\T \x_j & \| \x_j \|^2 \end{bmatrix}\right),
\end{equation}
so by performing a Gram-Schmidt orthogonalization procedure as in the proof of Theorem~\ref{theo:CK} in Appendix~\ref{subsec:proof_theo_CK}, one has
\begin{equation}
    %\w^\T \x_j - \| \x_j \| \left( \angle_{ij} \xi_i + \sqrt{1 - \angle_{ij}^2} \xi_j \right)  \to 0
    u = \| \x_i \| \cdot \xi_i, \quad v = \| \x_j \| \left( \angle_{ij} \cdot \xi_i + \sqrt{1 - \angle_{ij}^2} \cdot \xi_j \right),
\end{equation}
for two \emph{independent} standard Gaussian random variables $\xi_i$ and $\xi_j$, where we denote the shortcut $\angle_{ij} \equiv \frac{\x_i^\T \x_j}{ \| \x_i \| \cdot \| \x_j \| }$ for the ``angle'' between data vectors $\x_i$ and $\x_j$.
%and for $\xi_j \sim \NN(0,1)$ a standard Gaussian random variable  of $\xi_i$.

It can be checked, for $\x_i = \bmu_i/\sqrt p + \z_i/\sqrt p$ with $\EE[\z_i] = \zo$ and $\EE[\z_i \z_i^\T] = \C_i$ that
\begin{align*}
    \| \x_i \|^2 & = \frac1p (\bmu_i + \z_i)^\T (\bmu_i + \z_i) = \frac1p \| \bmu_i \|^2 + \frac2p \bmu_i^\T \z_i + \frac1p \z_i^\T \z_i                                                                                                                             \\
                 & = \frac1p \| \bmu_i \|^2 + \underbrace{\frac2p \bmu_i^\T  \z_i}_{O(p^{-1})} + \underbrace{\frac1p \tr \C^\circ}_{\equiv \tau_0^2 = O(1)} + \underbrace{\frac1p \tr \C^\circ_i}_{ \equiv t_i = O(p^{-1/2}) } + \underbrace{ \psi_i }_{O(p^{-1/2})},
\end{align*}
where we recall the definition $\psi_i \equiv \frac1p \| \z_i \|^2 - \frac1p \tr \C_i = O(p^{-1/2})$.
As such, by Taylor-expanding $\sqrt{\| \x_i \|^2}$ around $\| \x_i \|^2 \simeq \tau_0^2=O(1)$, we get
\begin{align*}
    \| \x_i \| & = \tau_0 + \frac1{2 \tau_0 } (\| \bmu_i \|^2/p + 2 \bmu_i^\T \z_i/p + t_i + \psi_i) - \frac1{8 \tau_0^3} (t_i + \psi_i)^2 + O(p^{-3/2}) \\
               & \equiv \tau_0 + \theta_i + O(p^{-3/2}),
\end{align*}
where we denote the shortcut
\begin{equation}
    \theta_i \equiv \frac1{2 \tau_0 } (\| \bmu_i \|^2/p + 2 \bmu_i^\T \z_i/p + t_i + \psi_i) - \frac1{8 \tau_0^3} (t_i + \psi_i)^2 = O(p^{-1/2}),
\end{equation}
so that
\begin{align*}
    \| \x_j \|\angle_{ij} & =  \frac{ \frac1p (\bmu_i + \z_i)^\T (\bmu_j + \z_j) }{ \| \x_i \|   }                                                                                                                                  \\
                          & = \frac{ \frac1p \bmu_i^\T \bmu_j + \frac1p (\bmu_i^\T \z_j + \bmu_j^\T \z_i) + \frac1p \z_i^\T \z_j }{ \tau_0 + \frac{1}{2\tau_0} (  t_i + \psi_i) + O(p^{-1})  } + O(p^{-3/2})                        \\
                          & = \left( \frac{1}{\tau_0} - \frac{ t_i + \psi_i}{2\tau_0^3} + O(p^{-1}) \right) \left(\frac1p \bmu_i^\T \bmu_j + \frac1p (\bmu_i^\T \z_j + \bmu_j^\T \z_i) + \frac1p \z_i^\T \z_j \right) + O(p^{-3/2}) \\
                          & = \frac{1}{\tau_0} \left(\frac1p \z_i^\T \z_j + \frac1p \bmu_i^\T \bmu_j + \frac1p (\bmu_i^\T \z_j + \bmu_j^\T \z_i) + S_{ij} \right) + O(p^{-3/2})                                                     \\
                          & = \frac1{\tau_0} A_{ij} + S_{ij} + O(p^{-3/2}).
\end{align*}
Therefore, again by Taylor-expansion,
\begin{align*}
    \sqrt{ \| \x_j \|^2- (\| \x_j \|\angle_{ij})^2 } & = \sqrt{(\| \bmu_j\|^2/p + 2\bmu_j^\T \z_j/p + \tau_0^2 + t_j + \psi_j)  - ( A_{ij}/\tau_0 + S_{ij})^2 }                  \\
                                                     & = \tau_0 + \frac1{2 \tau_0 } (\| \bmu_i \|^2/p + 2 \bmu_i^\T \z_i/p + t_i + \psi_i) - \frac1{8 \tau_0^3} (t_i + \psi_i)^2 \\
                                                     & - \frac{1}{2\tau_0^3}\left(\frac1p \z_i^\T \z_j  \right)^2 + S_{ij} + O(p^{-3/2})                                         \\
                                                     & =\tau_0 + \theta_j - \frac{1}{2\tau_0^3} \left(\frac1p \z_i^\T \z_j  \right)^2 + S_{ij} + O(p^{-3/2}).
\end{align*}
Following the same idea, we again Taylor-expand $\sigma_1(\cdot)$ in the definition of $\K_{\CK,1}$ as
\begin{align*}
    \sigma_1(u) & = \sigma_1 ( \tau_0 \xi_i ) + \sigma_1'( \tau_0 \xi_i ) \xi_i \theta_i +  \frac1{8\tau_0^2 } \sigma_1''( \tau_0 \xi_i ) \xi_i^2 (t_i + \psi_i)^2 + O(p^{-3/2}),
\end{align*}
and
\begin{align*}
    \sigma_1(v) & = \sigma_1 \left( \| \x_j \| \angle_{ij} \xi_j + \| \x_j \| \sqrt{1 - \angle_{ij}^2} \xi_i \right)                                                                      \\
                & = \sigma_1 \left(
    \tau_0 \xi_j + \xi_j \theta_j - \xi_j \frac1{2\tau_0^3} \left(\frac1p \z_i^\T \z_j  \right)^2 + \xi_i \frac1{\tau_0} A_{ij} + S_{ij} + O(p^{-3/2})\right)                             \\
                & =\sigma_1 ( \tau_0 \xi_j ) + \sigma_1'( \tau_0 \xi_j ) \xi_j \theta_j +  \frac1{8\tau_0^2 } \sigma_1''( \tau_0 \xi_j ) \xi_j^2 (t_j + \psi_j)^2 + X_{ij} + O(p^{-3/2}),
\end{align*}
with
\begin{align*}
    X_{ij} & = \frac{1}{\tau_0} \xi_i \sigma_1'(\tau_0 \xi_j) A_{ij} + \frac1{2\tau_0^2} \left( \frac1p \z_i^\T \z_j \right)^2 \left( \xi_i^2 \sigma_1''(\tau_0 \xi_j) - \frac{1}{\tau_0}\xi_j \sigma_1'(\tau_0 \xi_j)  \right) + S_{ij} = O(p^{-1/2}),
\end{align*}
where we recall the definition
\begin{equation}
    A_{ij} = \underbrace{\frac1p \z_i^\T \z_j}_{O(p^{-1/2})} + \underbrace{\frac1p \bmu_i^\T \bmu_j + \frac1p (\bmu_i^\T \z_j + \bmu_j^\T \z_i)}_{O(p^{-1})} = \x_i^\T \x_j.
\end{equation}

For independent $\xi_i$ and $\xi_j$, we denote the following coefficients
\begin{equation}
    p_0 = \EE[\sigma_1(\tau_0 \xi)], \quad p_1 = \EE[\sigma_1'(\tau_0 \xi)], \quad p_2 = \EE[\sigma_1''(\tau_0 \xi)], \quad p_3 = \EE[ \sigma_1'''(\tau_0 \xi) ],
\end{equation}
so that
\begin{equation}
    %\EE[\xi \sigma(\tau_0 \xi)] = \tau_0,\quad   \EE[\sigma'(\tau_0 \xi)] = \tau_0 p_1, \quad \EE[\xi \sigma'(\tau_0 \xi)] = \tau_0 p_2
    \EE[\xi \sigma_1(\tau_0 \xi)] = \tau_0 \EE[\sigma_1'(\tau_0 \xi)] = \tau_0 p_1, \quad \EE[\xi \sigma_1'(\tau_0 \xi)] = \tau_0 p_2,
\end{equation}
as well as
\begin{equation}
    \EE[ \xi^2 \sigma_1''(\tau_0 \xi) ] = \EE[ (\xi^2 -1) \sigma_1''(\tau_0 \xi) ] + p_2 = \tau_0^2 \EE[ \sigma_1''''(\tau_0 \xi) ] + p_2 \equiv \tau_0^2 p_4 + p_2,
\end{equation}
for $p_4 = \EE[\sigma_1''''(\tau_0 \xi)]$.

This further allows us to write, for $A_{ij} = O(p^{-1/2})$ and $\theta_j = O (p^{-1/2})$ that
\begin{align*}
      & [\K_{\CK,1}]_{ij} = \EE_{u,v} [\sigma_1(u) \sigma_1(v)]                                                                                                                                                 \\
    = & \EE \left[\sigma_1 ( \tau_0 \xi_i ) + \sigma_1'( \tau_0 \xi_i ) \xi_i \theta_i +  \frac1{8\tau_0^2 } \sigma_1''( \tau_0 \xi_i ) \xi_i^2 (t_i + \psi_i)^2 \right]                                        \\
      & \times \EE\left[ \sigma_1 ( \tau_0 \xi_j ) + \sigma_1'( \tau_0 \xi_j ) \xi_j \theta_j +  \frac1{8\tau_0^2 } \sigma_1''( \tau_0 \xi_j ) \xi_j^2 (t_j + \psi_j)^2 \right]                                 \\
      & + \EE \left[ \left( \sigma_1 ( \tau_0 \xi_i ) + \sigma_1'( \tau_0 \xi_i ) \xi_i \theta_i +  \frac1{8\tau_0^2 } \sigma_1''( \tau_0 \xi_i ) \xi_i^2 (t_i + \psi_i)^2 \right) X_{ij} \right] + O(p^{-3/2}) \\
    %%%
    = & \left(p_0 + \tau_0 p_2 \theta_i + \frac{\tau_0^2 p_4 + p_2}{8 \tau_0^2} (t_i + \psi_i)^2 \right)  \left(p_0 + \tau_0 p_2 \theta_j + \frac{\tau_0^2 p_4 + p_2}{8 \tau_0^2} (t_j + \psi_j)^2 \right)      \\
      & + \EE \left[ \sigma_1 (\tau_0 \xi_i) X_{ij} \right] + S_{ij} + O(p^{-3/2}),
\end{align*}
where the expectation is taken with respect to the \emph{independent} $\xi_i$ and $\xi_j$ (so, in fact, conditioned on $\x_i, \x_j$), so that
%\zhenyu{here!}\GD{$\EE[(\xi^2 - 1) \sigma_1(\tau_0 \xi)] = \tau_{0}^2\EE[ \sigma_1''(\tau_0 \xi)]$ maybe}
\begin{align*}
      [\K_{\CK,1}]_{ij} &= \EE_{u,v} [\sigma_1(u) \sigma_1(v)]                                                                                                                                                                                                                                                \\
    & =  \left(p_0 + \tau_0 p_2 \theta_i + \frac{\tau_0^2 p_4 + p_2}{8 \tau_0^2} (t_i + \psi_i)^2 \right)  \left(p_0 + \tau_0 p_2 \theta_j + \frac{\tau_0^2 p_4 + p_2}{8 \tau_0^2} (t_j + \psi_j)^2 \right)                                                                                                     \\
      & \quad + \EE \left[  \frac1{\tau_0} \xi_i \sigma_1(\tau_0 \xi_i) \sigma_1'(\tau_0 \xi_j) A_{ij} + \frac1{2\tau_0^2} \left(\frac1p \z_i^\T \z_j \right)^2 \right. 
      \\
      &\quad \quad \quad \left. \left( \xi_i^2 \sigma_1(\tau_0 \xi_i) \sigma_1''(\tau_0 \xi_j) - \frac1{\tau_0} \sigma_1(\tau_0 \xi_i) \xi_j \sigma_1'(\tau_0 \xi_j) \right) \right] 
       \\
      & \quad + S_{ij} + O(p^{-3/2})                                                                                                                                                                                                                                                                                 \\
    %%%
    &=  \left( p_0 + \frac{p_2}2 \chi_i + \frac{p_4}{8} (t_i + \psi_i)^2 \right) \left( p_0 + \frac{p_2}2 \chi_j + \frac{p_4}{8} (t_j + \psi_j)^2 \right)                                                                                                                                                      \\
      & \quad +p_1^2 A_{ij} + \frac1{2\tau_0^2} \left(\frac1p \z_i^\T \z_j \right)^2 \cdot p_2 \left( \EE[(\xi^2 - 1) \sigma_1(\tau_0 \xi)] \right) + S_{ij}+ O(p^{-3/2})                                                                                                                                            \\
    %%%
    &= p_0^2 + \frac{p_0 p_2}2 (\chi_i + \chi_j) +  \frac{p_0 p_4}{8} \left( (t_i + \psi_i)^2 + (t_j + \psi_j)^2 \right) + \frac{p_2^2}4 (t_i + \psi_i) (t_j + \psi_j)                                                                                                                                        \\
      & \quad +p_1^2 A_{ij} + \frac{p_2^2}{2}
    \left(\frac1p \z_i^\T \z_j \right)^2
    %\GD{\frac{p_2^2}{2} \left(\frac1p \z_i^\T \z_j \right)^2}
    + S_{ij} + O(p^{-3/2}),
\end{align*}
where we recall the shortcut
\begin{equation}
    \theta_i \equiv \frac1{2 \tau_0 } (\| \bmu_i \|^2/p + 2 \bmu_i^\T \z_i/p + t_i + \psi_i) - \frac1{8 \tau_0^2 \tau_0} (t_i + \psi_i)^2 \equiv \frac{\chi_i}{2 \tau_0 } - \frac{(t_i + \psi_i)^2}{8 \tau_0^2 \tau_0} = O(p^{-1/2}),
\end{equation}
with
\begin{equation}
    \chi_i \equiv  t_i + \psi_i + \| \bmu_i \|^2/p + 2 \bmu_i^\T \z_i/p = \| \x_i \|^2 - \tau_0.
\end{equation}
This gives, in matrix form,
%\GD{TBD here! the equivalence of the $d$ here and $\alpha$ in $\K_{CK}$}
\begin{align*}
    \K_{\CK,1} & = p_0^2 \one_n \one_n^\T + p_1^2 \left(\frac1p \Z^\T \Z + \frac1p \J \M^\T \M \J^\T + \frac1p(\J \M^\T \Z + \Z^\T \M \J^\T) \right)                                                                                           \\
    %   & + \frac{d_0 d_2}2 (\bchi \one_n^\T + \one_n \bchi^\T) + \frac{d_0 (d_2 + d_4)}{8\tau_0^2} \left( (\bt + \bpsi)^2 \one_n^\T + \one_n [(\bt + \bpsi)^2]^\T \right) + \frac{d_2^2}4 (\bt + \bpsi) (\bt + \bpsi)^\T  \\ 
               & + \frac{p_0 p_2}2 ( \boldsymbol{\chi} \one_n^\T + \one_n \boldsymbol{\chi}^\T ) + \frac{p_0 p_4}{8} \left( ( \{ t_a \one_{n_a} \}_{a=1}^K + \bpsi)^2 \one_n^\T + \one_n [(\{ t_a \one_{n_a} \}_{a=1}^K + \bpsi)^2]^\T \right) \\
               & + \frac{p_2^2}4 (\{ t_a \one_{n_a} \}_{a=1}^K + \bpsi) ( \{ t_a \one_{n_a} \}_{a=1}^K + \bpsi)^\T + \frac{p_2^2}{2} \left(\frac1p \Z^\T \Z \right)^{\circ 2}                                                                  \\
               & + (\EE[\sigma_1^2(\tau_0 \xi)] - p_0^2 - \tau_0^2 p_1^2)\I_n + O_{\| \cdot \|} (p^{-1/2}),
\end{align*}
where we denote $\boldsymbol{\chi} \equiv \{ \chi_i \}_{i=1}^n \in \RR^n$, $\A^{\circ 2}$ the \emph{entry-wise} square of the matrix $\A \in \RR^{n \times n}$, i.e., $[\A^{\circ 2}]_{ij} = [\A_{ij}]^2$, and use again the fact that $\| \A \| \leq n \| \A \|_\infty $ for $\A \in \RR^{n \times n}$ with $\| \A \|_{\infty} = \max_{ij} |\A|_{ij}$, $\{  S_{ij} \}_{i,j} = O_{\| \cdot \|}(p^{-\frac12})$ as well as $\left(\frac1p \Z^\T \Z \right)^{\circ 2} = \frac1p \J \bT \J^\T + O_{\| \cdot \|} (p^{-1/2})$ according to~\cite{couillet2016kernel}.
% Note in particular the additional term $(\EE[\sigma_1^2(\tau_0 \xi)] - p_0^2 - \tau_0^2 p_1^2)\I_n$ that accounts for the difference between the expressions of the diagonal and non-diagonal terms.

Finally, using the fact that for $\P = \I_n -  \one_n \one_n^\T/n$, we have $\one_n^\T \P = \zo, \P \one_n = \zo$, we conclude the proof of \eqref{eq:diff_centering} with the same expression of $\tilde \K_{\CK,1}$ as in the statement of our Theorem~\ref{theo:CK}, \emph{without} the assumption $\EE[\sigma_1(\tau_0 \xi)] = 0$.
This, however, no longer holds in the multi-layer setting with a number of layers $L \geq 1$.

%%% We should explain this in the longer version, at least give a first counter-example why this fails to hold!
%%%
%%%

\subsection{Proof and discussions of Corollary~\ref{coro:sparse_quantized}}
\label{subsec:proof_coro_sparse_quantized}

To prove Corollary~\ref{coro:sparse_quantized}, it can be easily checked that the i.i.d.\@ entries of the weights $\W$ defined in \eqref{eq:def_W} have zero mean and unit variance. So we focus on the design of the activations.

To ensure that the activation functions $\sigma_\ell(\cdot)$s are ``centered'' and satisfy $\EE[\sigma_\ell(\tau_{\ell-1} \xi)] = 0$, we define, with a slight abuse of notation, for the non-negative sequence $\tau_1, \ldots, \tau_L$ defined in Theorem~\ref{theo:CK},
%for fixed $\tau_{\ell-1}$ and  
\begin{equation}%\label{eq:def_sigmal}
    \sigma_T(t)=a \cdot (1_{t < s_1}+1_{t > s_2}), \quad \sigma_Q(t)=b_1 \cdot (1_{t < r_1}+1_{t > r_4}) + b_2 \cdot 1_{r_2 \leq t \leq r_3},
\end{equation}
and take $\alpha_{\ell, 0} \equiv \EE[\sigma_T(\tau_{\ell-1} \xi)]$, $\sigma_T(\tau_{\ell-1} \xi) \equiv \tilde \sigma_T(\tau_{\ell-1} \xi)=\sigma_T(\tau_{\ell-1} \xi)-\alpha_{\ell, 0}$, which serves as the activation of the first $\ell = 1, \ldots , L-1$ layers, and $a$, $s_1$ and $s_2$ satisfying the following equations
\begin{equation}
    \EE [ \sigma_T'( \tau_{\ell-1} \xi)]=\frac{a}{\sqrt{2\pi}\tau_{\ell-1}} \cdot \left(e^{-s_2^2 / (2\tau ^2_{\ell-1})}-e^{-s_1^2 / (2\tau ^2_{\ell-1})} \right),
\end{equation}
\begin{equation}
    \EE [ \sigma_T''( \tau_{\ell-1} \xi)]=\frac{a}{\sqrt{2\pi}\tau^3_{\ell-1}} \cdot \left(s_2e^{-s_2^2/(2\tau^2_{\ell-1})}-s_1e^{-s_1^2 / (2\tau ^2_{\ell-1})} \right),
\end{equation}
\begin{equation}
    % \EE [ (\sigma_T^2(\tau_{\ell-1}\xi))'']=\frac{a^2-2a\EE[\sigma_1(\tau_{\ell-1} \xi)]}{\sqrt{2\pi}\tau^3_{\ell-1}}\cdot \left(s_2e^{-s_2^2/(2\tau^2_{\ell-1})}-s_1e^{-s_1^2 / (2\tau ^2_{\ell-1})} \right)
    \EE [ (\sigma_T^2(\tau_{\ell-1}\xi))'']=\frac{a^2-2a\cdot \alpha_{\ell, 0}}{\sqrt{2\pi}\tau^3_{\ell-1}}\cdot \left(s_2e^{-s_2^2/(2\tau^2_{\ell-1})}-s_1e^{-s_1^2 / (2\tau ^2_{\ell-1})} \right),
\end{equation}
\begin{equation}
    \EE[\sigma_T^2(\tau_{\ell-1}\xi)] = \frac{a^2}{2}\left(\erf \left( \frac{s_1}{\sqrt{2}\tau_{\ell-1}} \right)-\erf\left(  \frac{s_2}{\sqrt{2}\tau_{\ell-1}} \right) +2 \right) - \alpha_{\ell, 0}^2,
\end{equation}
and $\alpha_{L, 0} \equiv \EE[\sigma_Q(\tau \xi)]$, $\sigma_T(\tau \xi) \equiv \tilde \sigma_T(\tau \xi)=\sigma_T(\tau \xi)-\alpha_{L, 0}$, which serves as the activation of the last and $L$th layer, and $b_1$, $b_2$, $r_1$, $r_2$, $r_3$ and $r_4$ satisfying the following equations

\begin{align}
    %&\EE[\sigma_Q(\tau \xi)] = \frac{b_1}{2} \left(\erf\left(\cfrac{r_1}{\sqrt{2}\tau}\right) - \erf\left(\frac{r_4}{\sqrt{2}\tau}\right) \right) + b_1 + \frac{b_2}{2} \left(\erf\left(\frac{r_3}{\sqrt{2}\tau}\right) - \erf\left(\frac{r_2}{\sqrt{2}\tau}\right) \right)  \\
    %%%
     & \EE [ \sigma_Q'( \tau \xi)] = \frac{b_1 \left(e^{-r_4^{2}/(2 \tau^2)} - e^{-r_2^{2}/(2 \tau^2)} \right)}{\sqrt{2\pi}\tau}
    + \frac{b_2 \left(e^{-r^2_2/(2\tau^2)}
    - e^{-r^3_2/(2\tau^2)}  \right)}{\sqrt{2\pi}\tau},                                                                               \\
    %%%
     & \EE [ \sigma_Q''( \tau \xi)] =\frac{b_1 \left(r_4e^{-r_4^2/(2\tau^2)}-r_1e^{-r_1^2 / (2\tau ^2)} \right)}{\sqrt{2\pi}\tau^3}
    + \frac{b_2 \left(r_2 e^{-r^2_2/(2\tau^2)}
        - r_3 e^{-r^3_2/(2\tau^2)}  \right)}{\sqrt{2\pi}\tau^{3}},
    %- 2\EE[\sigma_Q(\tau\xi)] \EE [ \sigma_Q''( \tau \xi)] \\
\end{align}
\begin{equation}
    \begin{aligned}
        \EE [ (\sigma_Q^2( \tau\xi))''] & =\frac{b_1^2 \left(r_4e^{-r_4^2/(2\tau^2)}-r_1e^{-r_1^2 / (2\tau ^2)} \right)}{\sqrt{2\pi}\tau^3}
        + \frac{b_2^2 \left(r_2 e^{-r^2_2/(2\tau^2)}
        - r_3 e^{-r^3_2/(2\tau^2)}  \right)}{\sqrt{2\pi}\tau^{3}},                                                                           \\
                                        & \quad -2 \alpha_{L,0} \EE [ (\sigma_Q''( \tau\xi))],
    \end{aligned}
\end{equation}
\begin{equation}
    \begin{aligned}
        \EE [(\sigma_Q^2(\tau \xi))] & =\frac{b_1^{2}}{2}    \left(\erf\left(\frac{r_1}{\sqrt{2}\tau}\right)
        - \erf \left(\frac{r_4}{\sqrt{2}\tau}\right) \right)
        + b_1^2 + \frac{b_2^{2}}{2} \left(\erf\left(\frac{r_2}{\sqrt{2}\tau}\right)
        - \erf\left(\frac{r_3}{\sqrt{2}\tau}\right) \right)                                                  \\
                                     & \quad -\alpha_{L,0}^{2}
    \end{aligned}
\end{equation}
with $\tau = \tau_{L-1}$.

A few remarks on Corollary~\ref{coro:sparse_quantized} and Algorithm~\ref{alg:sparse_quantized} are as follows.

\paragraph{On the numerical determinations of $\sigma_T$ and $\sigma_Q$.}

The above system of nonlinear equations does not admit explicit solutions, but can be solved efficiently using, for example, a (numerical) least squares method.
Precisely, we use the numerical least squares method (the {\sf optimize.minimize} function of {\sf SciPy} library) to solve the above system of equations, and run for $1\,000$ times with random and independent initializations to get $1\,000$ solutions, among which we choose the optimal parameters to determine $\sigma_Q$ and $\sigma_T$.

\paragraph{On the two activations.}
%\GD{anation expalaination, we want last layer keep more information}
Note that in Algorithm~\ref{alg:sparse_quantized} we use the activation $\sigma_T$ and $\sigma_Q$ respectively for the first $\ell = 1 ,\ldots, L-1$ and the final and $L$th layer, since we \emph{only} need to match the key parameters $\alpha_{\ell,1}$, $\alpha_{\ell,2}$ and $\alpha_{\ell,3}$ for the first $\ell = 1, \ldots, L-1$ layer, and the additional parameter $\tau_{\ell}$ for the last $L$th, so as to obtain spectrally equivalent CK and NTK matrices for the whole network of depth $L$.
% (In fact, one can even just match the last layer and just design activation function of the last layer while matching layer by layer can be a more convenient tips.) 
% \GD{equation systems here maybe}
%\zhenyu{We should discuss also the number of unknowns/coefficients in the two activation functions $\sigma_T$ and $\sigma_Q$. }
Also note that the proposed activation functions $\sigma_T$ and $\sigma_Q$ have respectively three and five (in fact six parameters with the symmetric constraint $r_1- r_2 = r_3 - r_4$ as in Figure~\ref{fig:sigma}) parameters that are freely tunable.
And we have respectively three and four (nonlinear) equations to determine these parameters in the system of equations above.

\section{Additional experiments}
\label{sec:additional_experiments}

In this section, we provide additional experiments to demonstrate the advantageous performance of the proposed ``lossless'' compression approach.
Figure~\ref{fig:classif-MNIST-2&5-classes} depicts the classification accuracies using three different neural networks: (i) the original ``dense and unquantized'' nets with three fully-connected layers of $\ReLU$ activations, (ii) the proposed sparse and quantized DNN model as per Algorithm~\ref{alg:sparse_quantized}, and (iii) the ``heuristically'' compressed networks by (iii-i) uniformly and randomly zeroing out $90\%$ of the weights, as well as (iii-ii) natively binarizing using $\sigma(t) = 1_{t<-1} + 1_{t>1}$, on two tasks of MNIST data classification \cite{lecun1998gradient} having five classes (digits $0,1,2,3,4$) and two classes (digits $6$ versus $8$).
This allows us to have a more qualitative assessment of the impact of data and task on the performance of the proposed compression scheme.
We see, as in Figure~\ref{fig:classif-perf} for ten-class MNIST and ten-class CIFAR10, that the proposed compression approach significantly outperform the two ``naive'' compression approaches, and can achieve a memory compression rate of $10^3$ and a level of sparsity up to $90\%$, with virtually no performance loss.
Also note that the experimental settings of Figure~\ref{fig:classif-MNIST-2&5-classes} is almost the \emph{same} as those of Figure~\ref{fig:classif-perf} in Section~\ref{sec:experiments}, except that the former networks have less neurons per layer and slightly higher level of sparsity ($90\%$ here instead of $80\%$ in the setting of Figure~\ref{fig:classif-perf}), to solve the simpler two-class or five-class classification problems.

\begin{figure}[htb]
    \centering
    \begin{tikzpicture}[font=\footnotesize][scale=0.5]
        \pgfplotsset{every major grid/.style={style=densely dashed}}
        \begin{axis}
            [width=0.7\textwidth,height=0.3\textwidth,
                xmode = log, legend pos = south east, grid=major,xlabel={Memory (bits)},ylabel={Accuracy},
                xmin=100000,xmax=520000000,ymin=0.6, ymax=1,
                xtick={100000, 1000000,10000000, 100000000}, ytick={0.6, 0.8, 1},
                xticklabels={$10^5$,$10^6$, $10^7$, $10^8$},
                yticklabels={$0.6$, $0.8$,$1$},legend style={
                        cells={anchor=east},
                        legend pos=outer north east,
                    },
            ]
            \addplot[mark size=2pt, BLUE, mark=triangle*, very thick]coordinates{
            (1789000,0.8327)(5578000, 0.9036)(11367000,0.9284)(19156000,0.9453)(28945000,0.9518)
            };
            % 0
            \addplot[mark size=2pt, BLUE, mark=square*, very thick]coordinates{
            (896000,0.8398)(2792000,0.9043)(5688000,0.9173)(9584000,0.9388)(14480000,0.9505)
            };
            %50
            \addplot[mark size=2pt, BLUE , mark=*, very thick]coordinates{
            (181600,0.8581)(563200,0.8984)(1144800,0.9258)(1926400,0.9486)(2908000,0.9603)
            };
            % 90
            \addplot[mark size=2pt, RED, mark=*, very thick]coordinates{
            (28624000,0.9857)(89248000,0.9909)(181872000,0.9909)(306496000,0.9915)(463120000,0.9922)
            };
            % origin
            \addplot[mark size=2pt, GREEN, mark=*, very thick]coordinates{
            (28579000,0.6406)(89158000,0.6803)(181737000,0.6862)(306316000,0.7155)(462895000,0.7650)
            };
            % %origin_quan
            \addplot[mark size=2pt, PURPLE, mark=*,very thick]coordinates{
            (2910400, 0.7884)(9020800,0.8197)(18331200, 0.8320)(30841600,0.8490)(46552000,0.8958)
            };
            %origin_spar_0.9

            \legend{$\varepsilon = 0\%$,$\varepsilon = 50\%$, $\varepsilon = 90\%$,original,naive quantized,naive sparsity};
        \end{axis}
    \end{tikzpicture}
    \\
    \begin{tikzpicture}[font=\footnotesize][scale=0.5]
        \pgfplotsset{every major grid/.style={style=densely dashed}}
        \begin{axis}
            [width=0.7\textwidth,height=0.3\textwidth,
                xmode = log, legend pos = south east, grid=major,xlabel={Memory (bits)},ylabel={Accuracy},
                xmin=1000000,xmax=30000000000,ymin=0.6, ymax=1,
                xtick={1000000, 10000000, 100000000, 1000000000, 10000000000}, ytick={0.6,0.8, 1},
                xticklabels={ $10^6$, $10^7$, $10^8$, $10^9$, $10^{10}$},
                yticklabels={$0.6$, $0.8$, $1$},legend style={
                        cells={anchor=east},
                        legend pos=outer north east,
                    },
            ]
            \addplot[mark size=2pt, BLUE, mark=triangle*, very thick]coordinates{
            (11376000,0.8994)(77884000,0.9352)(107920000,0.9607)(315780000,0.9632)(1271572000,0.9745)
            };
            % 0
            \addplot[mark size=2pt, BLUE, mark=square*, very thick]coordinates{
            (5692500, 0.8959)(38954000,0.9259)(53975000,0.9546)(157915000, 0.9624)(635837000,0.9720)
            };
            %50
            \addplot[mark size=2pt, BLUE , mark=*, very thick]coordinates{
            (1145700, 0.9042)(7810000,0.9216)(10819000,0.9541)(31623000,0.9554)(127249000,0.9693)
            };
            % 90
            \addplot[mark size=2pt, RED, mark=*, very thick]coordinates{
            (182016000,0.9856)(1246144000,0.9859)(1726720000,0.9904)(5052480000,0.9902)(20345152000,0.9922)
            };
            % origin
            \addplot[mark size=2pt, GREEN, mark=*, very thick]coordinates{
            (181881000,0.6376)(1245784000,0.7127)(1726270000,0.8042)(5051730000,0.8248)(20343622000,0.9012)
            };
            % %origin_quan
            \addplot[mark size=2pt, PURPLE, mark=*,very thick]coordinates{
            (35618400, 0.8775)(249536000,0.8914)(345728000, 0.9027)(1011136000,0.9118)(4070336000,0.9289)
            };
            %origin_spar_0.8

            \legend{$\varepsilon = 0\%$,$\varepsilon = 50\%$,$\varepsilon = 90\%$, original,naive quantized,naive sparsity};
        \end{axis}
    \end{tikzpicture}
    %\end{tabular}
    % \caption{Test accuracy of classification on MNIST dataset.\GD{lr:0.01}  }
    \caption{Test accuracy of classification on 2-class (\textbf{top}) MNIST dataset - digits $6$ versus $8$ and 5-class (\textbf{bottom}) MNIST dataset - digits $(0,1,2,3,4)$.
            {\BLUE \textbf{Blue}} curves represent the proposed ``lossless'' compression scheme with different levels of sparsity $\varepsilon \in \{ 0\%, 50\%, 90\% \}$,
        {\PURPLE \textbf{purple}} curves represent the heuristic sparsification approach by uniformly zeroing out $90\%$ of the weights,
        {\GREEN \textbf{green}} curves represent the heuristic quantization approach using the binary activation
        $\sigma (t) =  1_{t < -1}+ 1_{t > 1}$ (only applied on the first two layers, otherwise the performance is too poor to be compared to other curves),
        and {\RED \bf red} curves represent the original (dense and unquantized) network.
        %(note we just replace the first two layer's activation otherwise network performance can be so bad that we can not plot in this figure).
        All nets have three fully-connected layers, and the original network uses $\ReLU$ activations for all layers.
        Memory varies due to the \textbf{change of layer width} of the network.}
    \label{fig:classif-MNIST-2&5-classes}
\end{figure}
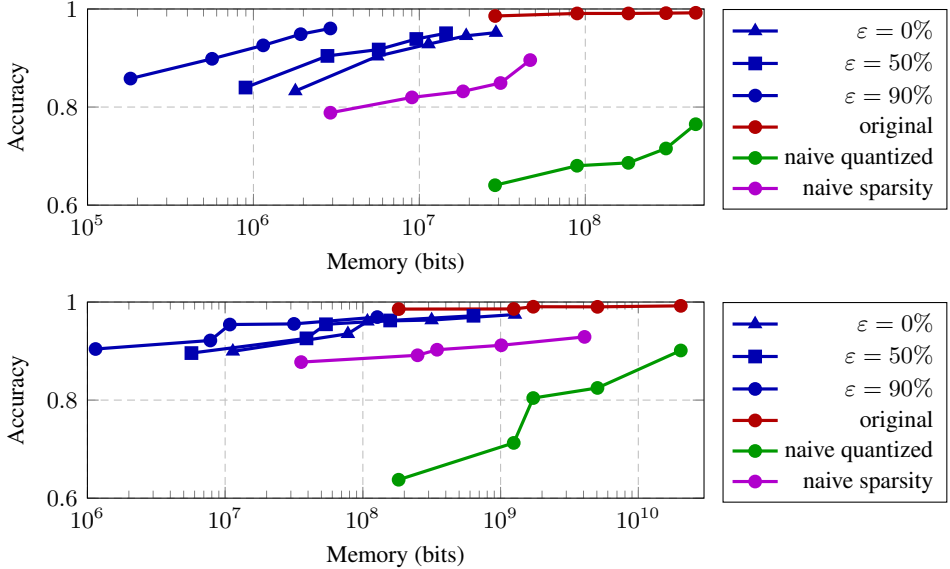

\medskip

In Table~\ref{tab:coe varies on mnist}~and~\ref{tab:coe varies on GMM}, we evaluate the impact of activation functions on the classification performance on data of \emph{different nature}, on a set of fully-connected DNN models having three hidden layers (of width $d_1 = 3\,000, d_2 = 3\,000, d_3 =1\,000$ in each layer) and use the \emph{same} activation $\sigma(\cdot)$ for all layers.

More precisely, Table~\ref{tab:coe varies on mnist} depicts the classification accuracy and the values of the key parameters $\alpha_1$, $\alpha_2$, $\alpha_3$ and $\tau$ for different activations $\sigma(\cdot)$ in the asymptotic equivalent CK matrix $\tilde \K_{\CK}$ defined in Theorem~\ref{theo:CK} of the third and final layer of the network, on a binary classification of MNIST data (class $6$ versus $8$).
Similarly, Table~\ref{tab:coe varies on GMM} compares the classification accuracy and $\alpha_1$, $\alpha_2$, $\alpha_3, \tau$ for different activations, on two-class GMM data with identical mean $\bmu_a=\zo_p$ and different covariance $\C_a=(1+8(a-1)/\sqrt{p})\I_p$, $a \in \{ 1,2\}$,
The numerical experiments are performed on a training set of size $12\,000$, a test set of $1\,800$, with $\x_1, \ldots, \x_{n/2} \in \mathcal{C}_1$ and $\x_{n/2+1}, \ldots, \x_n \in \mathcal{C}_2$, for standard Gaussian $\W$ on both MNIST and GMM data.

We observe from Table~\ref{tab:coe varies on mnist}~and~\ref{tab:coe varies on GMM} that:
\begin{itemize}
    \item[(i)] while in theory, the key parameters $\alpha_1$, $\alpha_2$, $\alpha_3$ and $\tau$ (in Theorem~\ref{theo:CK}) depend on both (the statistics of) the data and the activation, the impact of the activation $\sigma$ appears much more significant;
    \item[(ii)] by using some $\sigma$ (with the corresponding $\alpha_1$, $\alpha_2$ and/or $\alpha_3$ being zero), one asymptotically ``discards'' either the first-order ($\bmu_a$) or the second-order ($\bt, \bT$) statistics of the data (which, per Theorem~\ref{theo:CK}, are respectively weighted by the key parameter $\alpha_1$, $\alpha_2$ and $\alpha_3$), resulting in performance degradation;
    \item[(iii)] precisely, we divide commonly used activations in Table~\ref{tab:coe varies on mnist}~and~\ref{tab:coe varies on GMM} into the following three categories:
        \begin{enumerate}
            \item covariance-oriented activations with $\alpha_1 = 0$: this includes $\cos(t)$ and $|t|$; and
            \item mean-oriented activations with $\alpha_2 = 0$ and $\alpha_3 = 0$: this includes $1_{t \geq 0}$, ${\rm sign}(t)$, $(1+e^{-t})^{-1}$~\cite{1995The}, $\sin(t)$, linear function, and the Gaussian error function ${\rm erf(t)}$; and
            \item balanced activations with nonzero $\alpha_1, \alpha_2, \alpha_3$: this includes ReLU activation $\ReLU(t) = \max(t,0)$ and Leaky ReLU activation \cite{maas2013rectifier}.
        \end{enumerate}
        The above classification of activation functions is reminiscent of that proposed in \cite{liao2018spectrum}, which is, however, only valid in a single-hidden-layer setting.
        In line with the observations made in \cite{liao2018spectrum}, we see in Table~\ref{tab:coe varies on mnist} that covariance-oriented activations behave poorly in the classification of MNIST data (that are known to have very different first-order statistics, see for example \cite[Table~3]{liao2018spectrum}), while mean-oriented activations yield unsatisfactory performance on GMM data having different covariance structure in Table~\ref{tab:coe varies on GMM}.
        %is proportional to $\alpha_1$, and (iii) on GMM dataset(significant variance varies), accuracy is proportional to $\alpha_2$ and $\alpha_3$. 
        In a sense, the parameter $\alpha_{1}$ characterizes the ``ability'' of a given net to extract first-order data statistics and $\alpha_{2}, \alpha_3$ the ``ability'' to extract second-order statistics from the input data, respectively.
\end{itemize}

\begin{table}[htb]
    \caption{Classification accuracy and values of $\alpha_1$, $\alpha_2$, $\alpha_3$ and $\tau$ at the third and final layer, on MNIST data (digits $6$ versus $8$).}
    \label{tab:coe varies on mnist}
    \centering
    \begin{tabular}{cccccc}
        \toprule
        %\multicolumn{2}{c}{MNIST}                   \\
        %\cmidrule(r){1-2}
        $\sigma(t)$                                  & $\alpha_1$ & $\alpha_2$ & $\alpha_3$ & $\tau$ & Accuracy \\
        \midrule
        $\max(0,t)$                                  & 0.0156     & 0.0105     & 0.0112     & 0.1994 & 0.971    \\
        %$\left\{ \begin{matrix} 0.1t, \quad t <0 \\ t, \quad t \geq 0 \end{matrix}\right .$
        $0.1 t \cdot 1_{t<0} + t \cdot 1_{t \geq 0}$ & 0.0083     & 0.0097     & 0.0081     & 0.1750 & 0.9654   \\
        \midrule
        $1_{t \geq 0}$                               & 0.0642     & 0          & 0          & 0.5    & 0.9665   \\
        ${\rm sign}(t)$                              & 0.1779     & 0          & 0          & 0.4689 & 0.9715   \\
        $1/(1+e^{-t})$                               & 0.0002     & 0          & 0          & 0.0129 & 0.9637   \\
        $\sin(t)$                                    & 0.1779     & 0          & 0          & 0.4689 & 0.9749   \\
        $t$                                          & 1          & 0          & 0          & 1.0021 & 0.981    \\
        $\erf(t)$                                    & 0.2166     & 0          & 0          & 0.5053 & 0.9788   \\
        %$t^2+t+1$                                                                             & 1      & 4339   & 6      & 30     & 0.8789 \\ 
        \midrule
        $\cos(t)$                                    & 0          & 0.0003     & 0          & 0.0116 & 0.5257   \\
        $|t|$                                        & 0          & 0.0209     & 0          & 0.2195 & 0.5709   \\
        %$\exp(t)$                                                                             & 0      & 0      & 0      & 0.0017 & 0.5346 \\
        \bottomrule
    \end{tabular}
\end{table}

\begin{table}[htb]
    \caption{Classification accuracy and values of $\alpha_1$, $\alpha_2$, $\alpha_3$ and $\tau$ at the third and final layer, on GMM data.}
    \label{tab:coe varies on GMM}
    \centering
    \begin{tabular}{cccccc}
        \toprule
        % \multicolumn{2}{c}{GMM}                   \\
        % \cmidrule(r){1-2}
        $\sigma(t)$                                  & $\alpha_1$ & $\alpha_2$ & $\alpha_3$ & $\tau$ & Accuracy \\
        \midrule
        $\max(0,t)$                                  & 0.0156     & 0.0092     & 0.0099     & 0.2128 & 0.8945   \\
        %$\left\{ \begin{matrix} 0.1t, \quad t <0 \\ t, \quad t \geq 0 \end{matrix}\right .$  
        $0.1 t \cdot 1_{t<0} + t \cdot 1_{t \geq 0}$ & 0.0083     & 0.0085     & 0.0071     & 0.1867 & 0.9079   \\
        \midrule
        $1_{t \geq 0}$                               & 0.0564     & 0          & 0          & 0.5    & 0.5028   \\
        ${\rm sign}(t)$                              & 0.2256     & 0          & 0          & 1      & 0.4916   \\
        $1/(1+e^{-t})$                               & 0.0002     & 0          & 0          & 0.0135 & 0.5173   \\
        $\sin(t)$                                    & 0.1512     & 0          & 0          & 0.4729 & 0.5025   \\
        $t$                                          & 1          & 0          & 0          & 1.0693 & 0.5045   \\
        $\erf(t)$                                    & 0.1912     & 0          & 0          & 0.51   & 0.4989   \\
        \midrule
        %$t^2+t+1$                                   & 1      & 8018   & 6      & 45.6     & 0.9621 \\ 
        $\cos(t)$                                    & 0          & 0.0003     & 0          & 0.015  & 0.9598   \\
        $|t|$                                        & 0          & 0.0184     & 0          & 0.2342 & 0.9302   \\
        %$\exp(t)$                                 & 0      & 0      & 0      & 0.002 & 0.6161 \\ 
        \bottomrule
    \end{tabular}
\end{table}

\end{document}